\documentclass[11pt,english]{article}
\usepackage[T1]{fontenc}
\usepackage[latin9]{inputenc}
\usepackage{amsmath}
\usepackage{amsthm}
\usepackage{amssymb}
\usepackage{dsfont}
\usepackage{bm}
\usepackage{graphicx}
\usepackage{booktabs} 
\usepackage{caption}
\usepackage{subcaption}
\usepackage{comment}
\usepackage{microtype} 
\usepackage{verbatim}
\usepackage[title]{appendix}
\usepackage{algpseudocode}
\usepackage{algorithm}
\usepackage{enumerate}
\usepackage[dvipsnames]{xcolor}
\definecolor{Gray}{gray}{0.94}
\usepackage[align=center,shadow=true,shadowsize=6pt,nobreak=true,framemethod=tikz,skipabove=9.5pt,skipbelow=9pt,innertopmargin=5pt,innerbottommargin=5pt,innerleftmargin=5pt,innerrightmargin=5pt,leftmargin=1.5pt,rightmargin=1.5pt]{mdframed}
\usetikzlibrary{shadows}

\makeatletter
\numberwithin{equation}{section}
\numberwithin{figure}{section}
\theoremstyle{plain}
\newtheorem{thm}{\protect\theoremname}
\theoremstyle{definition}
\newtheorem{defn}{\protect\definitionname}
\theoremstyle{plain}
\newtheorem{lemma}{\protect\lemmaname}
\newtheorem{corollary}{\protect\corollaryname}

\theoremstyle{remark}
\newtheorem*{rem*}{\protect\remarkname}

\theoremstyle{plain}

\@ifundefined{date}{}{\date{}}
\usepackage{color}
\usepackage{colortbl}
\usepackage{hyperref}
\hypersetup{
    colorlinks,
    allcolors=[rgb]{0.3,0.1,0.9}
}

\usepackage[lined,boxed,ruled, linesnumbered, algo2e]{algorithm2e}

\usepackage[margin=1in]
           {geometry}

\usepackage{thmtools}
\usepackage{thm-restate}
\usepackage{amsfonts}
\usepackage{tikz}
\usepackage{enumitem}
\usepackage{wrapfig}
\usepackage{ragged2e}
\allowdisplaybreaks

\makeatother

\usepackage{babel}
\providecommand{\definitionname}{Definition}
\providecommand{\lemmaname}{Lemma}
\providecommand{\remarkname}{Remark}
\providecommand{\theoremname}{Theorem}
\providecommand{\corollaryname}{Corollary}
\providecommand{\remarkname}{Remark}
\providecommand{\assumptionname}{Assumption}

\begin{document}
\global\long\def\norm#1{\left\Vert #1\right\Vert }%
\global\long\def\R{\mathbb{R}}%
\global\long\def\eps{\epsilon}%
 
\global\long\def\Rn{\mathbb{R}^{n}}%
\global\long\def\tr{\mathrm{Tr}}%
\global\long\def\diag{\mathrm{diag}}%
\global\long\def\Diag{\mathrm{Diag}}%
\global\long\def\C{\mathbb{C}}%
\global\long\def\conv{\mathrm{conv}}%
\global\long\def\var{\mathrm{var}}%
\global\long\def\kurt{\mathrm{Kurt}}%

\global\long\def\mr{\text{mr}}%
\global\long\def\EE{\mathbb{E}}%
\global\long\def\E{\mathop{\mathbb{E}}}%
\global\long\def\vol{\mathrm{vol}}%
\global\long\def\argmax{\mathrm{argmax}}%
\global\long\def\argmin{\mathrm{argmin}}%
\global\long\def\sign{\mathrm{sign}}%
\global\long\def\bd{\mathrm{bd}}%
\global\long\def\R{\mathbb{R}}%
\global\long\def\ham{\mathrm{Ham}}%
\global\long\def\e#1{ \exp\left(#1\right)}%
\global\long\def\Var{\mathrm{Var}}%
\global\long\def\dint{{\displaystyle \int}}%
\global\long\def\step{\delta}%
\global\long\def\Ric{\mathrm{Ric}}%
\global\long\def\P{\mathbb{P}}%
\global\long\def\len{\text{\text{len}}}%
\global\long\def\lspan{\mathrm{span}}%
\newcommand{\santosh}[1]{{{\bf \color{blue}{SANTOSH: #1\\}}}}
\newcommand{\Xinyuan}[1]{{{\bf \color{red}{XINYUAN: #1\\}}}}
\newcommand{\squeezeup}{\vspace{-2.5mm}}
\newcommand{\eg}{\emph{e.g.}}
\newcommand{\ie}{\emph{i.e.}}

\bibliographystyle{alpha}

\title{Contrastive Moments: Unsupervised Halfspace Learning \\in Polynomial Time}

\author{Xinyuan Cao \\ Georgia Tech\\\texttt{xcao78@gatech.edu}  \and Santosh S. Vempala \\ Georgia Tech\\\texttt{vempala@gatech.edu}}
\maketitle
\thispagestyle{empty}
\begin{abstract}
 We give a polynomial-time algorithm for learning high-dimensional halfspaces with margins in $d$-dimensional space to within desired TV distance when the ambient distribution is an unknown affine transformation of the $d$-fold product of an (unknown) symmetric one-dimensional logconcave distribution, and the halfspace is introduced by deleting at least an $\epsilon$ fraction of the data in one of the component distributions. Notably, our algorithm does not need labels and establishes the unique (and efficient) identifiability of the hidden halfspace under this distributional assumption.  The sample and time complexity of the algorithm are polynomial in the dimension and $1/\epsilon$. The algorithm uses only the first two moments of {\em suitable re-weightings} of the empirical distribution, which we call {\em contrastive moments}; its analysis uses classical facts about generalized Dirichlet polynomials and relies crucially on a new monotonicity property of the moment ratio of truncations of logconcave distributions. Such algorithms, based only on first and second moments were suggested in earlier work, but hitherto eluded rigorous guarantees.

Prior work addressed the special case when the underlying distribution is Gaussian via Non-Gaussian Component Analysis. We improve on this by providing polytime guarantees based on Total Variation (TV) distance, in place of existing moment-bound guarantees that can be super-polynomial. Our work is also the first to go beyond Gaussians in this setting. 
\end{abstract}

\newpage
\thispagestyle{empty}

\tableofcontents

\newpage
\setcounter{page}{1}

\section{Introduction}

Suppose points in $\R^d$ are labeled according to a linear threshold function (a halfspace). Learning a threshold function from labeled examples is the archetypal well-solved problem in learning theory, in both the PAC and mistake-bound models; its study has led to efficient algorithms, a range of powerful techniques and many interesting learning paradigms. While the sample complexity in general grows with the dimension, when the halfspace has a {\em margin}, the complexity can instead be bounded in terms of the reciprocal of the squared margin width \cite{platt1999large,smola2000advances,arriaga2006algorithmic,long2011algorithms}. The problem is also very interesting for special classes of distributions, e.g., when the underlying distribution is logconcave, agnostic learning is possible~\cite{kalai2008agnostically}, and active learning needs fewer samples compared to the general case \cite{balcan2006agnostic}. 

The main motivation for our work is learning a halfspace with a margin {\em with no labels}, i.e., unsupervised learning of halfspaces. This is, of course, impossible in general --- there could be multiple halfspaces with margins consistent with the data --- raising the question: 
{\em Can there be natural distributional assumptions that allow the unsupervised learning of halfspaces?}
For example, suppose data is drawn from a Gaussian in $\R^d$ with points in an unknown band removed, i.e., we assume there exists a unit vector $u\in\mathbb{R}^d$ and an interval $[a,b]$ so that the input distribution is the Gaussian restricted to the set $\{x\in\mathbb{R}^d |\langle u,x\rangle \le a$ or $\langle u, x\rangle \ge b \}$. Can the vector $u$ be efficiently learned? Such a distributional assumption ensures that the band normal to $u$ is essentially unique, leaving open the question of whether it can be efficiently learned.

Such models have been considered in the literature, notably for Non-Gaussian Component Analysis (NGCA)~\cite{blanchard2006search,tan2018polynomial},
learning relevant subspaces~\cite{blum1994relevant, vempala2011structure} and low-dimensional convex concepts~\cite{vempala2010learning}  where data comes from a product distribution with all components being Gaussian except for one (or a small number). It is assumed that the non-Gaussian component differs from Gaussian in some low moment and the goal is to identify this component. Another related model is Independent Component Analysis (ICA) where the input consists of samples from an affine transformation of a product distribution and the goal is to identify the transformation itself~\cite{comon1994independent,cardoso1998multidimensional, goyal2014fourier,jia2023beyond}. For this problem to be well-defined, it is important that at most one component of the product distribution is Gaussian. No such assumption is needed for NGCA or the more general problem we consider here.

Formally, we consider the following model and problem, illustrated in Fig.~\ref{fig:intro}.
\begin{defn}[Affine Product Distribution with $\eps$-Margin]\label{def_1}
Let $q$ be a symmetric one-dimensional isotropic logconcave density function. Let $Q$ be the $d$-fold product distribution obtained from $q$. Let $\hat{q}$ be the isotropized density obtained after restricting $q$ to $\R \backslash [a,b]$ where $q((-\infty,a])\geq \eps, q([a,b])\geq \eps $ and $q([b,\infty))\geq \eps$. Let $P$ be the product of one copy of $\hat{q}$ and $d-1$ copies of $q$. Let $\hat{P}$ be obtained by a full-rank affine transformation of $P$; we refer to $\hat{P}$ as an {\em Affine Product Distribution with $\eps$-Margin}. Let $u$ be the unit vector normal to the margin before transformation.
%See Figure~\ref{fig:dist_def} as an example of $P$ and $\hat{P}$ corresponding to uniform $q$.
\end{defn}

%add figures here

\begin{figure}[htp]
\centering
\captionsetup[subfigure]{font=small}
\hspace{-0.1in}
\begin{subfigure}[t]{0.35\textwidth}
    \includegraphics[width=\linewidth]{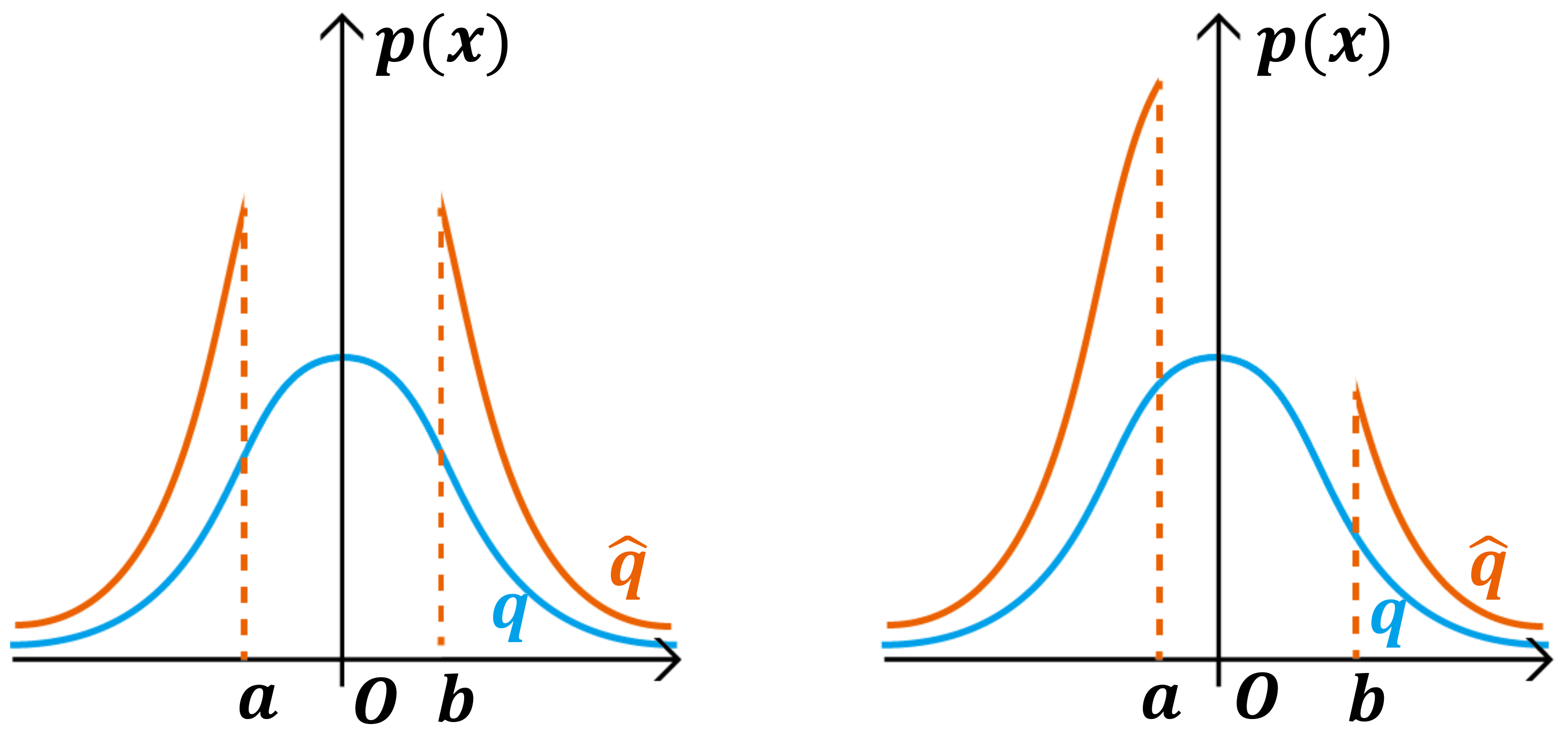}
    \caption{Illustration of the two margin cases - symmetric $[a,b]$ and asymmetric $[a,b]$.}
    % Two cases of the margin, when $[a,b]$ is symmetric and when it is asymmetric. 
    %If $[a,b]$ is symmetric, the mean does not change, so the algorithm computes the contrastive variance. Otherwise, the algorithm computes the contrastive mean.
    \label{fig:two_cases}
\end{subfigure}
\hspace{0.05in}
\begin{subfigure}[t]{0.35\textwidth}
    \includegraphics[width=\linewidth]{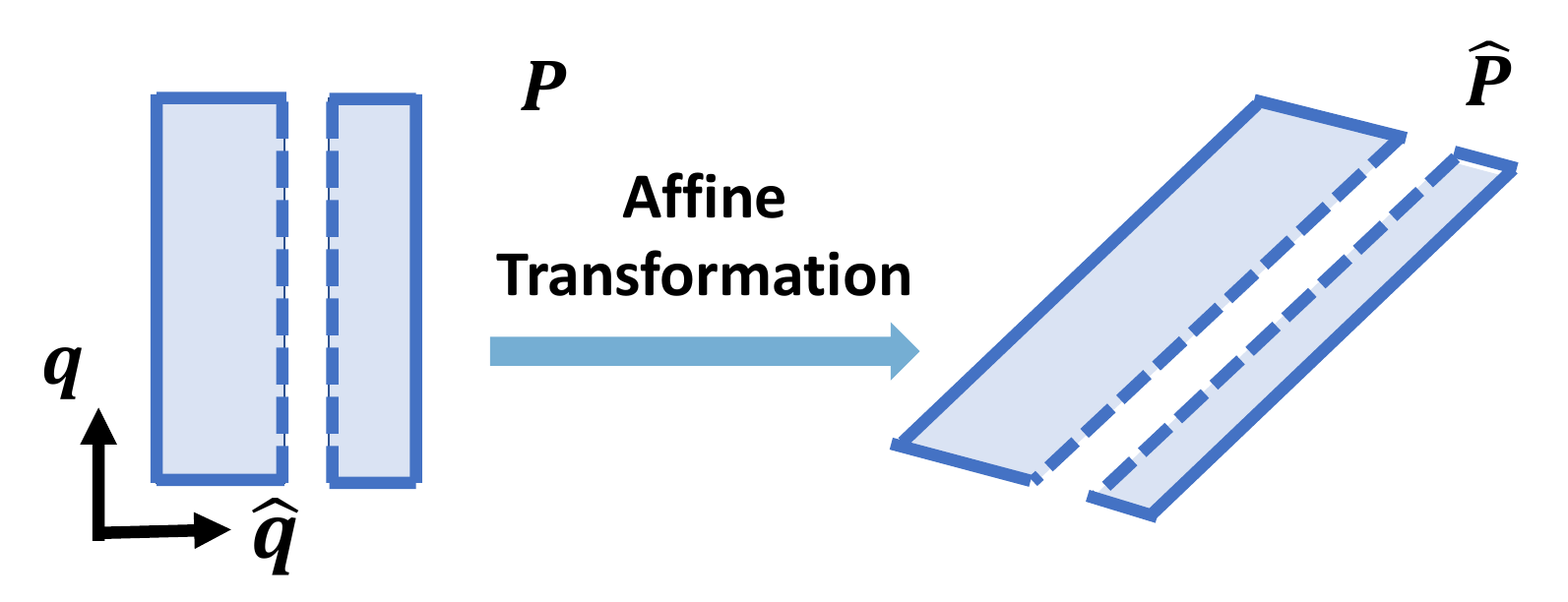}
    \caption{
    Def.~\ref{def_1} with uniform $q$. $\hat{P}$ is a full-rank affine transformation of the product distribution $P$.}
    \label{fig:dist_def}
\end{subfigure}
\hspace{0.05in}
\begin{subfigure}[t]{0.25\textwidth}
    \includegraphics[width=0.85\textwidth]{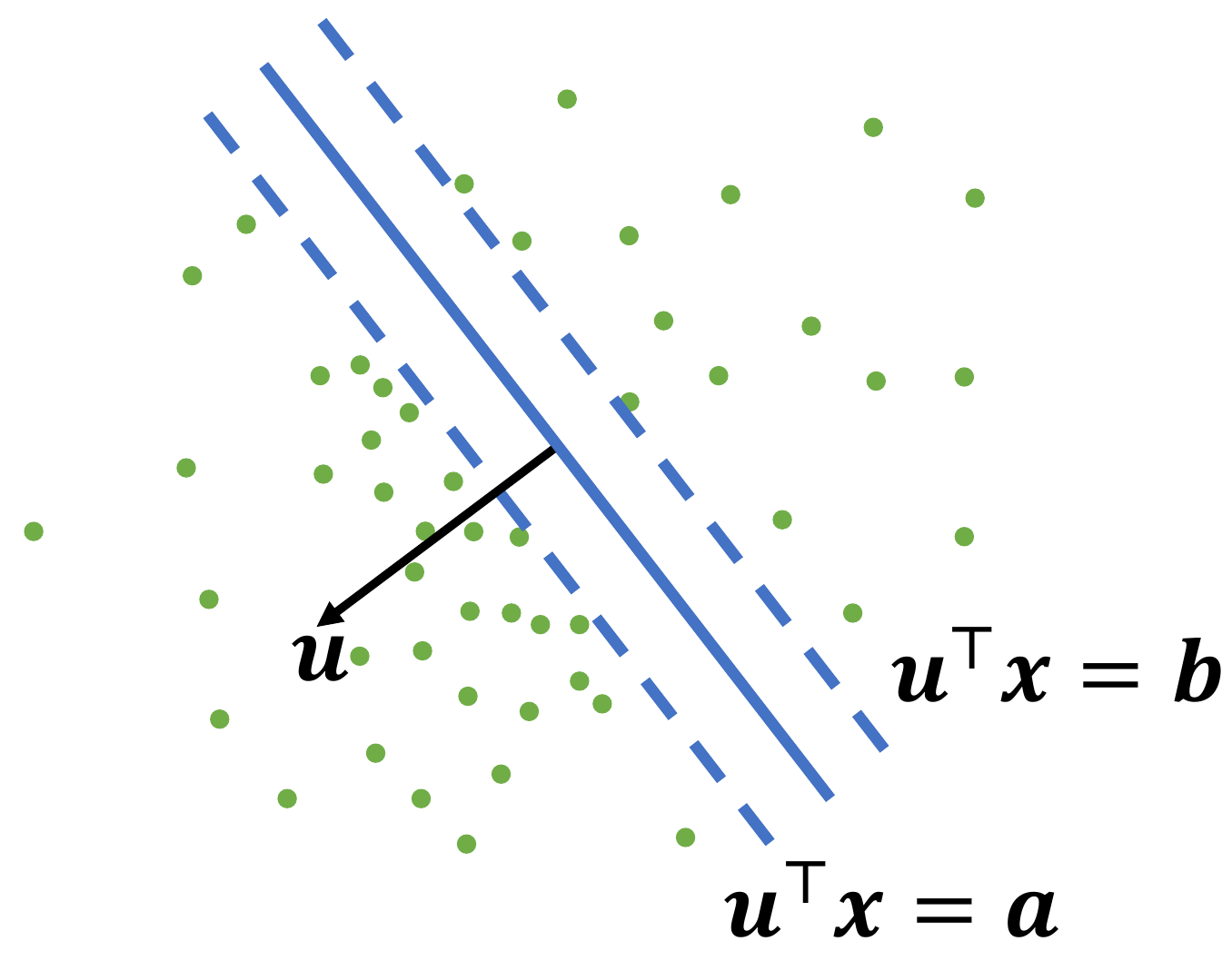}
    \caption{ Unlabeled data is drawn from $\hat{P}$. We aim to learn the normal vector $u$.}
    \label{fig:data_dist_example}
\end{subfigure} % maximize horizontal separation
\caption{Affine Product Distribution with Margin.}
\label{fig:intro}
\end{figure}

With this model in hand, we have the following algorithmic problem.
\vspace{-0.1in}
\paragraph{Problem.} Given input parameters $\eps, \delta > 0$ and access to iid samples from $\hat{P}$, an affine product distribution with $\eps$-margin, the learning problem is to compute a unit vector $\tilde{u}$ that approximates $u$ to within $TV$ distance $\delta$. That is, the TV distance between the corresponding $\tilde{P}$ and $P$ is at most $\delta$, where $\tilde{P}$ is the distribution with margin normal to $\tilde{u}$. 
%(See Figure~\ref{fig:data_dist_example}).

In this formulation of the problem with a TV distance guarantee, if each side of the halfspace receives a different label, then the probability that the output halfspace of the data disagrees with the true label (up to swapping the labels) is at most $\delta$.

A natural approach to identifying the halfspace is maximum margin clustering \cite{xu2004maximum}: find a partition of the data into two subsets s.t. the distance between the two subsets along some direction is maximized. Unfortunately, this optimization problem is NP-hard, even to approximate. 

There are at least two major difficulties we have to address. The first is the unknown affine transformation, which we cannot hope to completely identify in general. The second is that, even if we reversed the transformation, the halfspace normal is in an arbitrary direction in $\R^d$ and would be undetectable in almost all low-dimensional projections, i.e., we have a needle in a haystack problem. 

\subsection{Results and techniques}

We give an efficient algorithm for the unsupervised halfspace learning problem under any symmetric product logconcave distribution. It consists of the following three high-level steps.
\vspace{-0.1cm}
\begin{itemize}
    \item[(1)] Make the data isotropic (mean zero and covariance matrix identity).
    \item[(2)] Re-weight data and compute the re-weighted mean $\tilde{\mu}_i$ and the top eigenvector $v$ of the re-weighted covariance.
    \item[(3)] Project data along the vectors $\tilde{\mu}_i, v$, and output the vector with the largest margin.
\end{itemize}

Although the algorithm is simple and intuitive, its analysis has to overcome substantial challenges. Our main result is the following.
\begin{mdframed}
\begin{thm}[Main]\label{thm:main}
\vspace{-2.5mm}
There is an algorithm that can learn any affine product distribution with $\eps$-margin to within TV distance $\delta$ with time and sample complexity that are polynomial in $d,  1/\eps$ and $1/\delta$ with high probability. 
\end{thm}
\end{mdframed}

To see the idea of the algorithm, we first consider the case when no affine transformation is applied. In this case, we can detect the direction $u$ by calculating the empirical mean and top eigenvector of the empirical uncentered covariance matrix. If the margin $[a,b]$ lies on one side of the origin, the mean along $u$ is nonzero while the mean in any other direction that is orthogonal to $u$ is zero. Thus the mean itself reveals the vector $u$. Otherwise, we can show that the second moment along $u$ is higher than along any other orthogonal direction. Thus, there is a positive gap between the top two eigenvalues of the uncentered covariance matrix and the top eigenvector is $u$.
In fact, the algorithm applies more generally, to the product distribution created from one-dimensional bounded isoperimetric distributions. A one-dimensional distribution $p$ is isoperimetric if there exists a constant $\psi>0$ such that for any $x\in\R$, $p(x)\geq \psi \min\{p([x,\infty)), p((-\infty, x])\}$.

\begin{mdframed}
\begin{restatable}[Isotropic Isoperimetric Distribution]{thm}{thm:isotropic}\label{thm_isotropic_intro}
\vspace{-2.5mm}
    There is an algorithm that can learn any isotropic isoperimetric bounded product distribution with $\eps$-margin to within TV distance $\delta$ with time and sample complexity that are polynomial in $d,1/\eps,1/\delta$ with high probability.
\end{restatable}
\end{mdframed}

In the general case, when an unknown affine transformation is applied, the algorithm
first computes the empirical mean and covariance of the sample and makes the empirical distribution isotropic. Then we will consider two cases as illustrated in Figure~\ref{fig:two_cases}.
If the unknown band is {\em not centered} around the mean along $u$, we can expect the empirical mean to differ from the mean of the underlying product distribution without the margin. Consequently, if we knew the latter, we can use the difference to estimate $u$. However, in general, we do not have this information. Instead, we demonstrate that there exists a re-weighting of the sample so that re-weighted empirical mean compared to the unweighted empirical mean is a good estimate of $u$. In other words, with appropriate re-weighting, the mean shifts along the normal direction to the unknown band. On the other hand, if the band is centered along $u$, the mean shift will be zero. In this scenario, we will show that the maximum eigenvector of a re-weighted uncentered covariance matrix is nearly parallel to $u$!

Our algorithm only uses first and second order moments, can be implemented efficiently, and is in fact practical (see Section~\ref{section:experiments}). The main challenges are (1) proving the existence of band-revealing re-weightings and (b) showing that a polynomial-sized sample (and polynomial time) suffice. 

To prove the main theorem, we will show that either the re-weighted mean
induces a contrastive gap (Lemma~\ref{lemma_contrastive_mean_qual}), or the eigenvalues of the re-weighted uncentered covariance matrix induce a contrastive gap (Lemma~\ref{lemma_contrastive_covariance_qual}). 
In the subsequent two lemmas, we adopt the notation from Definition~\ref{def_1}. Here, $P$ represents a product distribution with $\eps$-margin defined by the interval $[a,b]$ (before transformation). We use $\|\cdot\|$ to denote the $l_2$ norm of a vector.

\begin{restatable}[Contrastive Mean]{lemma}{contrastivemeanqual}
\label{lemma_contrastive_mean_qual}
     If $|a+b| >0$, then for any two distinct nonzero $\alpha_1,\alpha_2 \in \R$, at least one of the corresponding re-weighted means is nonzero, i.e.,
    \[
    \max\left(\left\lvert\E\limits_{x\sim P} e^{\alpha_1 \|x\|^2}u^\top x\right\rvert,  \left\lvert\E\limits_{x\sim P} e^{\alpha_2 \|x\|^2}u^\top x\right\rvert
    \right)> 0.
    \]
\end{restatable}

\begin{restatable}[Contrastive Covariance]{lemma}{contrastivecovariancequal}
\label{lemma_contrastive_covariance_qual}
    If $a+b=0$, then there exists an $\alpha<0$, such that (1) there is a positive gap between the top two eigenvalues of the re-weighted uncentered covariance matrix  $\tilde{\Sigma}=\E_{x\sim P}e^{\alpha \|x\|^2}(xx^\top)$. That is, $\lambda_1(\tilde{\Sigma}) > \lambda_2(\tilde{\Sigma})$. (2) The top eigenvector of $\tilde{\Sigma}$ is $u$.
\end{restatable}

The proof of Lemma~\ref{lemma_contrastive_mean_qual} uses Descartes' Rule of signs applied to a suitable potential function. 
To prove Lemma~\ref{lemma_contrastive_covariance_qual}, we develop a new monotonicity property of the moment ratio (defined as the ratio of the variance of $X^2$ and the squared mean of $X^2$) for truncations of logconcave distributions. 
The moment ratio is essentially the square of the coefficient of variation of $X^2$. 
An insight from the monotonicity of the moment ratio is that for logconcave distributions with positive support, when the distribution is restricted to an interval away from the origin, it needs a smaller sample size to estimate its second moment accurately. We state the lemma as follows.

\begin{restatable}[Monotonicity of Moment Ratio]{lemma}{varianceratio}\label{lemma_variance_ratio}
    Let $q$ be a logconcave distribution in one dimension with nonnegative support. For any $t \ge  0$, let $q_t$ be the distribution obtained by restricting $q$ to $[t,\infty)$.
Then the moment ratio of $q_t$, defined as $\frac{\var_{q_t} (X^2)}{(\E_{q_t} X^2)^2}$, is strictly decreasing with $t$.
\end{restatable}

% \begin{lemma}[Monotonicity of Variance Ratio]\label{lemma_variance_ratio}
% Let $q$ be a logconcave distribution in one dimension with nonnegative support. For any $a \ge  0$, let $q_a$ be the distribution obtained by restricting $q$ to $[a,\infty)$.
% Then the moment ratio of $q_a$, defined as $\frac{\var_{q_a} (X^2)}{(\E_{q_a} X^2)^2}$, is strictly decreasing with $a$.
% \end{lemma}

To obtain polynomial guarantees, we will need quantitative estimates of the inequalities in the above two lemmas. Establishing such quantitative bounds is the bulk of the technical contribution of this paper. While our focus is on proving {\em polynomial} bounds, whose existence a priori is far from clear, we did not optimize the polynomial bounds themselves; our experimental results suggest that in fact the dependence on both $d$ and $1/\eps$ might be linear!

\subsection{Related Work}

Efficient algorithms for supervised halfspace learning~\cite{rosenblatt1958perceptron,minsky69perceptrons}, combined with the kernel trick~\cite{cristianini2000introduction,hearst1998support}, serve as the foundation of much of learning theory. Halfspaces {\em with margin} are also well-studied, due to their motivation from the brain, attribute-efficient learning~\cite{valiant1998projection, blum1990learning}, random projection based learning~\cite{arriaga2006algorithmic}, and turn out to have sample complexity that grows inverse polynomially with the margin, independent of the ambient dimension. When examples are drawn from a unit Euclidean ball in $\R^d$, and the halfspace has margin $\gamma$, then the sample complexity grows as $O(1/\gamma^2)$ regardless of the dimension. This leads to the question of whether labels are even necessary, or the halfspace can be identified from unlabeled samples efficiently --- the focus of the present paper.  

The model of unsupervised learning we study is similar to other classical models in the literature, notably Independent Component Analysis where input data consists of iid samples from an unknown affine transformation of a product distribution. There, the goal is to recover the affine transformation under minimal assumptions. Known polynomial-time algorithms rely on directional moments, and the assumption that component distributions differ from a Gaussian in some small moment. A related relevant problem, Non-Gaussian Component Analysis (NGCA), aims to extract a hidden non-Gaussian direction in a high-dimensional distribution. Here too, the main idea is the fact that non-Gaussian component must have some finite moment different from that of a Gaussian. While finite moment difference implies a TV distance lower bound, to get $\eps$-TV distance, one might need to use $k$'th moments for $k=\Omega(\log(1/\eps))$ even for logconcave densities. As the dependence on the moment number is exponential (even for the sample complexity), this approach does not yield polytime algorithms in terms of TV distance, the natural notion for classification.

The idea of applying Principal component analysis (PCA) to re-weighted samples was used in ~\cite{brubaker2008isotropic} to unravel a mixture of well-separated Gaussians. For a mixture of two general Gaussians that are mean separated, after making the mixture isotropic, it was shown that either the mean or top eigenvector of the covariance of a re-weighted sample reveals the vector of the mean differences. This high-level approach was used for solving general ICA by estimating re-weighted higher moments (tensors)~\cite{goyal2014fourier}. Higher moment re-weightings were also used by~\cite{vempala2011structure} to give an algorithm for factoring a distribution and learning ``subspace juntas", functions of an unknown low-dimensional subspace, and by~\cite{tan2018polynomial} to give a more efficient algorithm for the special case of NGCA. The question of whether expensive higher moment algorithms could be replaced by re-weighted second moment is natural and one variant was specifically suggested by~\cite{tan2018polynomial} for NGCA. Our work validates this intuition with rigorous polynomial-time algorithms.

\section{Warm-up: Isotropic Isoperimetric Distribution with $\eps$-Margin}
\label{section:isotropic_case}

As a warm-up, we consider the isotropic product distributions 
with $\eps$-margin. Notably, without applying an unknown transformation on data, we can extend the logconcave distributions to isoperimetric distributions. In this section,
we will demonstrate how to retrieve the normal vector $u$ by calculating the empirical mean and top eigenvector of the empirical uncentered covariance matrix. This technique is similar to Principal Component Analysis (PCA), but instead of computing covariance matrix, we use the uncentered covariance matrix.

\begin{defn}
    A distribution $p$ with support $\R$ is $\psi$-isoperimetric if there exists $\psi>0$ such that for any $x\in \R$, we have $p(x) \geq \psi \min  \{p([x,\infty)), p((-\infty, x])\}$.
\end{defn}

\begin{defn}[Isotropic Isoperimetric Distribution with $\epsilon$-Margin]
    Let $q_1, \ldots, q_d$ be symmetric one-dimensional isotropic $\psi$-isoperimetric density functions bounded by $\tau$. Let $Q=q_1\otimes\cdots \otimes q_d$. Let $\hat{q}$ be the 
    density obtained after restricting $q_1$ to $\R\backslash[a,b]$ where $q_1((-\infty,a])\geq \eps,q_1([a,b])\geq \eps$ and $q_1([b,\infty))\geq \eps$. Let $P$ be an arbitrary rotation of $\hat{q}\otimes q_2\otimes \cdots \otimes q_d$. We refer $P$ as an \textit{Isotropic Isoperimetric Distribution with $\epsilon$-Margin}. Let $u$ be the unit vector normal to the margin.
\end{defn}

\paragraph{Problem.} Given input parameters $\eps,\delta>0$ and access to iid samples from $P$, an isotropic isoperimetric distribution with $\epsilon$-margin, the learning problem is to compute a unit vector $\tilde{u}$ that approximates $u$ to within TV distance $\delta$. That is, the TV distance between the corresponding $\tilde{P}$ and $P$ is at most $\delta$, where $\tilde{P}$ is the distribution with margin normal to $\tilde{u}$.

\subsection{Algorithm}
Given data drawn from $P$, we compute the sample mean and the top eigenvector of the uncentered covariance matrix. Then we
compare the max margin along these two candidate normal vectors. This gives an efficient algorithm for the problem with no re-weighting. 
We state the algorithm formally in Algorithm~\ref{algo_isotropic}.
\begin{algorithm}[htp]
\caption{Unsupervised Halfspace Learning from Isotropic Isoperimetric Data}
\label{algo_isotropic}
\KwIn{Unlabeled data  $x^{(1)},\cdots,x^{(N)} \in \R^d$. $\epsilon,\delta>0$.}

\begin{itemize}
    \item Compute the sample mean and uncentered covariance matrix:
    \[
    \hat{\mu}=\frac{1}{N}\sum_{j=1}^N x^{(j)},\quad \hat{\Sigma} = \frac{1}{N}\sum_{j=1}^N x^{(j)}{x^{(j)}}^\top
    \]
    \item Compute $\hat{\Sigma}$'s top eigenvector $v$.
    \item Calculate the max margin (i.e., maximum gap) of the one-dimensional projections of the data along the vectors $\hat{\mu}, v$. Let $\hat{u}$ be the vector among these two with a larger margin.
\end{itemize}
\Return{the vector $\hat{u}$.}
\end{algorithm}

\subsection{Analysis}

We demonstrate that Algorithm~\ref{algo_isotropic} operates within polynomial time and sample complexity. The details regarding sample complexity are presented in Theorem~\ref{thm_isotropic} (formal statement of Theorem~\ref{thm_isotropic_intro}). The time complexity is justified by the algorithm's process: it calculates the sample mean and the top eigenvector of the sample covariance matrix, both of which require polynomial time.
\begin{thm}[Sample Complexity for Isotropic Isoperimetric Distribution]
\label{thm_isotropic}
    Algorithm~\ref{algo_isotropic} with $N=\tilde{O}(d^2\eps^{-6}\delta^{-2}\xi^{-1})$ samples learns the target isotropic isoperimetric distribution with $\eps$-margin to within TV distance $\delta$ with probability $1-\xi$.
\end{thm}
The analytical approach is straightforward. Given that the component distributions are isotropic, the empirical mean will reveal the band if the removed band $[a,b]$ stays on one side of the origin. Otherwise, when $[a,b]$ spans across the origin,
the variance along the component with the deleted band will increase. Consequently, this component emerges as the top principal component. 
Intriguing, this property is ``opposite" to
the method used to identify low-dimensional convex concepts in \cite{vempala2010learning}. The latter relies on the Brascamp-Lieb inequality, where the variance of a restricted Gaussian is less than that of the original Gaussian.

To prove Theorem~\ref{thm_isotropic}, we aim to quantify either the mean gap or the spectral gap (gap between the top two eigenvalues) of the uncentered covariance matrix.
Specifically, Lemma~\ref{lemma_iso_mean} indicates that when $0\leq a<b$, the mean along the direction $u$ significantly deviates from zero. Meanwhile, Lemma~\ref{lemma_iso_spectral_gap} demonstrates that when $a\leq 0<b$, there's a gap between the first and second eigenvalues of the uncentered covariance matrix. Subsequently, we employ Lemma \ref{lemma_cov_estimate_2eps} \cite{srivastava2013covariance} to determine the sample complexity, and utilize the Davis-Kahan Theorem  \cite{davis1970rotation} (Lemma~\ref{lemma_sin_theta}) to leverage the eigenvalue gap in identifying the pertinent vector $u$. We leave the proof of the lemmas in Section \ref{section:proof_iso}.

For any $x\in\R^d$, we denote $x_i$ as its $i$-th coordinate. We use $\|x\|$ to denote its $l_2$ norm. For a matrix $A\in\R^{m\times n}$, we denote its operator norm as $\|A\|_{\text{op}}$. We denote the standard basis of $\R^d$ by $\{e_1,\cdots,e_d\}$, and assume wlog that $e_1=u$ is the (unknown) normal vector to the band. Denote $\Sigma=\E_{x\sim P}xx^\top$ as the uncentered covariance matrix of $P$, with eigenvalues $\lambda_1\geq \lambda_2 \geq \cdots\geq \lambda_d$.

\begin{restatable}[Mean Gap]{lemma}{isomean}
\label{lemma_iso_mean}
    For $0\leq a<b$ and $b-a \geq c\eps$ for constant $c>0$, we have
    \[
    \E\limits_{x\sim \hat{q}} x <-\frac{\psi c^2\eps^3}{2}, \mathop{\var}\limits_{x\sim \hat{q}}x \leq \frac{1}{2\eps}.
    \]
\end{restatable}

\begin{restatable}[Spectral Gap of Covariance]{lemma}{isospectralgap}
\label{lemma_iso_spectral_gap}
 If $a\leq 0<b$ and $b>c\eps$ for constant $c>0$, then the first and second eigenvalues of the uncentered covariance matrix $\Sigma$ have the following gap
    \[
    \lambda_1-\lambda_2 >  C\eps^3 \lambda_1 \text{  for constant }C>0.
    \]
    Furthermore, the top eigenvector corresponds to $u$.
\end{restatable}

The following theorem enables us to bound the sample complexity to estimate the covariance matrix.
 \begin{lemma}[Covariance Estimation \cite{srivastava2013covariance}]\label{lemma_cov_estimate_2eps}
    Consider independent isotropic random vectors $X_i$ in $\R^d$ s.t. for some $C,\eta>0$, for every orthogonal projection $P$ in $\R^d$,
    \[
    \P (\|PX_i\| > t)\leq Ct^{-1-\eta}\text{ for }t > C \text{rank}(P).
    \]
    Let $\eps \in(0,1)$. Then with the sample size $N =O ( d\eps^{-2-2/\eta} )$, we have
    \[
    \E\|\Sigma - \hat{\Sigma}\|_{\text{op}} \leq \eps \|\Sigma\|_{\text{op}}.
    \]
\end{lemma}

The following classical theorem allows us to use the eigenvalue gap to identify the relevant vector.
\begin{lemma}[Davis-Kahan \cite{davis1970rotation}]
\label{lemma_sin_theta}
Let $S$ and $T$ be symmetric matrices with the same dimensions. For a fixed $i$, assume that the largest eigenvalue of $S$ is well separated from the second largest eigenvalue of $S$, \ie, $\exists \delta > 0$  s.t. $\lambda_1(S) - \lambda_2(S) > \delta$. Then for the top eigenvectors of $S$ and $T$, denoted as $v_1(S)$ and $v_1(T)$,  we have
\[
\sin\theta(v_1(S),v_1(T))\leq \frac{2\|S-T\|_{\text{op}}}{\delta}.
\]

\end{lemma}

\vspace{0.1in}
Now we are ready to prove Theorem~\ref{thm_isotropic}. 
\vspace{-0.1in}
\begin{proof}[Proof of Theorem~\ref{thm_isotropic}]
    % Let $\{e_1,\cdots,e_d\}$ be a basis in $\R^d$. We assume wlog that $e_1=u$. For any $x\in\R$, we write $x_i$ as its $i$-th coordinate.
    We can proceed with the assumption that $|b|>|a|$. If this condition is not met, we can redefine our interval by setting $a'=-b$ and $b'=-a$. The proof can then be applied considering the distribution is restricted to $\{x\in \R^d: u^\top x \leq a' \text{ or }u^\top x \geq b'\} $. 
    We will prove the theorem by considering two cases: $0\leq a<b$ and $a\leq 0<b$ . 
    % Given that $q_1$ is bounded by $\tau$, it follows that $b-a > \eps/\tau$. Additionally, taking into account that  $|b|>|a|$, it results in $b > \eps/(2\tau)$.
    
    We first consider the case when $0\leq a<b$. Given that $q_1$ is bounded by $\tau$, it follows that $b-a > \eps/\tau$. By Lemma~\ref{lemma_iso_mean}, we know
    \[
    \E\limits_{x\sim P}x_1 <- \frac{\psi\eps^3}{2\tau^2}, \mathop{\var}\limits_{x\sim P}x_1 \leq  \frac{1}{2\eps},
    \]
    while for $i\geq 2$, we have
    \[
    \E\limits_{x\sim P}x_i=0, \mathop{\var}\limits_{x\sim P}x_i =1.
    \]
    Given data $x^{(1)},\cdots,x^{(N)}$, let $\hat{\mu} = \frac{1}{N}\sum_{j=1}^N x^{(j)}$ be the sample mean. Then by Chebyshev's Inequality,
    \[
    \mathbb{P}(\hat{\mu}_1 >- \frac{\psi\eps^3}{4\tau^2}) \leq \frac{8\tau^4}{N\psi^2\eps^7},\quad
    \mathbb{P} (\hat{\mu}_i <-\frac{\psi\eps^3\delta}{4\tau^2\sqrt{d}}  )\leq \frac{16\tau^4 d}{N\psi^2\eps^6\delta^2},2\leq i\leq d
    \]
    Let $0<\xi<1$.
    So we know with sample size $N_1 =\frac{16\tau^4 d^2}{\eps^7\delta^2\psi^2\xi} $,
    \[
     \mathbb{P}(\hat{\mu}_1 >- \frac{\psi\eps^3}{4\tau^2}) \leq \frac{\psi^2\delta^2\xi}{2d^2}< \frac{\xi}{d},
     \quad
      \mathbb{P} (\hat{\mu}_i <-\frac{\psi\eps^3\delta}{4\tau^2\sqrt{d}}  )\leq \frac{\eps \xi}{d}< \frac{\xi}{d}
    \]
    Then we have
    \begin{align*}
        \mathbb{P}(\sin\theta(\hat{\mu},e_1)\leq \delta)
        =& \mathbb{P}(\frac{\hat{\mu}_1^2}{\sum_{i=1}^d \hat{\mu}_i^2} \geq 1-\delta^2)\\
        \geq &\mathbb{P}(\hat{\mu}_1 < - \frac{\psi\eps^3}{4\tau^2}, 
        \hat{\mu}_i >- \frac{\psi\eps^3\delta}{4\tau^2\sqrt{d}} ,2\leq i\leq d)\\
        \geq &1-\xi
    \end{align*}
    Secondly, we consider the case where $a\leq 0<b$. Given that $b-a > \eps/\tau$ and $|b|>|a|$, it results in $b > \eps/(2\tau)$. By Lemma~\ref{lemma_iso_spectral_gap}, the top two eigenvalues of $\Sigma$, denoted as $\lambda_1$ and $\lambda_2$ satisfies
    \[
    \lambda_1-\lambda_2 \geq C\eps^3\lambda_1 \quad \text{for some constant }C>0
    \]
    % \[
    % 1+C\eps^3 < \lambda_1 \leq \frac{1}{2\eps}, \quad \lambda_2=1.
    % \]
    By Lemma~\ref{lemma_cov_estimate_2eps}, with sample size $N_2=\tilde{O}(d\eps_1^{-2})$, with probability at least $1-\xi$,
    \[
    \|\Sigma - \hat{\Sigma}\|_{\text{op}} \leq \eps_1 \|\Sigma\|_{\text{op}} 
    % \leq \frac{\eps_1}{2\eps}
    \]
    By Lemma~\ref{lemma_sin_theta}, we know for the top eigenvector $v$ of $\hat{\Sigma}$ satisfies
    \[
    \sin\theta(e_1,v) \leq \frac{2 \|\Sigma - \hat{\Sigma}\|_{\text{op}}}{C\eps^3 \lambda_1} \leq \frac{2\eps_1 \lambda_1}{C\eps^3\lambda_1}
    = \frac{2\eps_1}{C\eps^3}
    \]
    Choose $\eps_1 = C\eps^3\delta/2$, and we will get $\sin\theta(e_1,v) \leq \delta$. The sample size we need is $N_2 = \tilde{O}(d\eps^{-6}\delta^{-2})$. So with sample size $N = \max(N_1,N_2) = \tilde{O}(d^2\eps^{-6}\delta^{-2}\xi^{-1})$, Algorithm~\ref{algo_isotropic} can recover $e_1$ within TV distance $\delta$ with probability $1-\xi$.
    
\end{proof}

\section{General Case: Affine Product Distribution with $\eps$-Margin}
In this section, we examine the general setting where data is drawn from $\hat{P}$, an affine product distribution with $\eps$-margin, as described in Definition~\ref{def_1}. We employ a strategy analogous to the one used in the warm-up scenario: utilizing the first moment to address cases where the band is asymmetric to the origin and the second moment for cases where the band is symmetric. However, given that $\hat{P}$ results from the application of an unknown affine transformation to $P$, the first and second moments of $\hat{P}$ remain unknown, even in the direction orthogonal to $u$. Our approach, therefore, is first to make the data isotropic. Following that, we deploy re-weighted first and second moments to detect $u$. Theorem~\ref{thm:main} provides a formal demonstration of the efficiency of our proposed algorithm.

\subsection{Algorithm}
Our algorithm first makes the data to be isotropic using the sample mean and sample covariance. Then we apply the weight $w(y,\alpha)=e^{\alpha\|y\|^2}$ to each isotropized sample point $y$, and compute the re-weighted mean and the top eigenvector of the re-weighted covariance matrix. Then for each candidate normal vector, we project the data to it, and scan to find the maximum gap. The algorithm outputs the vector with the maximal gap among all candidate vectors.
We give the formal description in Algorithm~\ref{algo_general}.

\begin{algorithm}[htp]
\caption{Unsupervised Halfspace Learning with Contrastive Moments}
\label{algo_general}
\KwIn{Unlabeled data  $S = \{x^{(1)},\cdots,x^{(N)}\} \subset \R^d$. $\epsilon,\delta > 0$.}
\begin{itemize}
    \item (Isotropize) Compute the sample mean and covariance: 
    \[
    \hat{\mu}=\frac{1}{N}\sum_{j=1}^N x^{(j)}, \qquad \hat{\Sigma} = \frac{1}{N}\sum_{j=1}^N (x^{(j)}-\hat{\mu})(x^{(j)}-\hat{\mu})^\top.
    \]
    Make the data isotropic: $y^{(j)} = \hat{\Sigma}^{-1/2} (x^{(j)}-\hat{\mu})$.
    \item (Re-weighted Moments) Set $\alpha_1=-c_1\eps^{82}/d,\alpha_2=-c_2\eps^{42}/d$ and $\alpha_3=-c_3\eps^{2}$.
    Let $w(y,\alpha) = e^{\alpha \|y\|^2}$.   Compute the re-weighted sample means $\tilde{\mu}_{\alpha_1}, \tilde{\mu}_{\alpha_2}$ using $\alpha_1, \alpha_2$ and the re-weighted sample covariance using $\alpha_3$ as follows:
    \[
    \tilde{\mu}_{\alpha_i}=  \frac{1}{N}\sum_{j=1}^N w(y^{(j)},\alpha_i)y^{(j)}, \mbox{ for } i\in\{1,2\} \mbox{ and }
    \tilde{\Sigma} = \frac{1}{N}\sum_{j=1}^N w(y^{(j)},\alpha_3)y^{(j)}{y^{(j)}}^\top
    \]
    % \[
    % \tilde{\mu}_i= \mathop{\E}\limits_{y \in S} (w(y,\alpha_i)y), \mbox{ for } i\in\{1,2\} \mbox{ and }
    % \tilde{\Sigma} = \mathop{\E}\limits_{y \in S} (w(y, \alpha_3)yy^\top)
    % \]
    Compute the top eigenvector $v$ of $\tilde{\Sigma}$.
    \item (Max Margin) Calculate the max margin (i.e., maximum gap) of the one-dimensional projections of the data along the vectors $\tilde{\mu}_{\alpha_1},\tilde{\mu}_{\alpha_2},v$, and let $\hat{u}$ be the vector among these with the largest margin.
\end{itemize}
\Return{$\hat{u}$.}
%\Return{$\hat{u} = \hat{\Sigma}^{1/2}\hat{v}$.}
\end{algorithm}

\subsection{Analysis}
In our algorithm, we consider two cases depending on whether the removed band $[a,b]$ is origin-symmetric. If it is asymmetric, we will show that one of the re-weighted means with two $\alpha$s gives us the correct direction by showing that the re-weighted mean along $u$ has a gap from zero while the re-weighted mean along all other orthogonal directions is zero. We state the positive gap quantitatively in Lemma~\ref{lemma_reweighted_mean}. Otherwise, if the band is symmetric, we will show a positive gap between the top two eigenvalues of the re-weighted covariance matrix, and the top eigenvector corresponds to $u$. We quantify the gap between the top two eigenvalues in Lemma~\ref{lemma_reweighted_cov}.
In the algorithm, since we know neither the underlying distribution mean nor the location of the removed band, we have to compute both re-weighted means and re-weighted covariance, and then get the correct direction among all three candidate vectors by calculating the margin and finding the one with the largest margin. In the end, we utilize Lemma \ref{lemma_cov_estimate_2eps} \cite{srivastava2013covariance} to determine the sample complexity, and apply Lemma~\ref{lemma_sin_theta} (Davis-Kahan \cite{davis1970rotation}) to leverage the eigenvalue gap in identifying the pertinent vector. We state the two quantitative lemmas below and provide their proofs in Section \ref{section:proof}.

\begin{restatable}[Quantitative Gap of Contrastive Mean]{lemma}{reweightedmean}
\label{lemma_reweighted_mean}
    Suppose that $|a+b|\geq \eps^{5}$. Then, for  $\alpha_1=-c_1\eps^{82}/d,\alpha_2=-c_2\eps^{42}/d$, the re-weighted mean of $P$, denoted as $\mu_{\alpha_1}$ and $\mu_{\alpha_2}$,  satisfies
    \[
    \max \left(\left\lvert u^\top \mu_{\alpha_1}\right\rvert,\left\lvert u^\top\mu_{\alpha_2}\right\rvert \right) >\frac{C\eps^{159}}{d^2} \text{ for some constant }C>0,
    \]
    \[
    \forall v\bot u,\quad v^\top\mu_{\alpha_1} = v^\top \mu_{\alpha_2}=0.
    \]
\end{restatable}

\begin{restatable}[Quantitative Spectral Gap of Contrastive Covariance]
{lemma}{reweightedcov}\label{lemma_reweighted_cov}
    Suppose that $|a+b|<\eps^5$. Choose $\alpha_3 = -c_3\eps^2$ for some constant $c_3>0$. 
    Then, for an absolute constant $C$, the top two eigenvalues $\lambda_1 \ge \lambda_2$ of the corresponding re-weighted covariance of $P$ satisfy
    \[
    \lambda_1 - \lambda_2 \geq C\eps^3 \lambda_1.
    \]
    Moreover, its top eigenvector corresponds to $u$.
\end{restatable}

Armed with two quantitative lemmas and the Davis-Kahan Theorem, we are now prepared to prove the main theorem.
\begin{proof}[Proof of Theorem~\ref{thm:main}]
    Given data drawn from $\hat{P}$, we first compute the sample mean and covariance and make the data to be isotropic, where we denote the isotropic data as $y^{(1)},\cdots,y^{(N)}$. Each $y^{(j)}$ is drawn iid from distribution $P$ up to rotation. We assume wlog that $y^{(j)}\sim P$, and $u=e_1$ is the target direction. 

    Firstly we consider the case when $|a+b|\geq \eps^{5}$. Denote $\alpha^*=\argmax_{\alpha} \{|{(\mu_{\alpha_1})}_1|, |{(\mu_{\alpha_2})}_1|\}$, and  $\mu_{\alpha} = \mu_{\alpha^*}$. By Lemma~\ref{lemma_reweighted_mean}, $|(\mu_\alpha)_1|\geq C_1\eps^{159}/d^2$. Since for any negative $\alpha$, for any $1\leq i\leq d$,
    \[
        \var (\mu_\alpha)_i \leq \E\limits_{y\sim P} e^{2\alpha\|y\|^2}y^2 \leq \E\limits_{y\sim P} y^2=1
    \]
    By Chebyshev's Inequality, the re-weighted sample mean $\tilde{\mu}=\frac{1}{N}\sum_{j=1}^N e^{\alpha^* \|y^{(j)}\|^2}y^{(j)}$ satisfies
    \[
    \mathbb{P}(|\tilde{\mu}_1| \leq \frac{C_1\eps^{159}}{2d^2}) \leq \frac{4d^4}{NC_1^2\eps^{318}}
    ,\quad 
    \mathbb{P}(|\tilde{\mu}_i| \geq \frac{C_1\eps^{159}\delta}{2d^2\sqrt{d}})\leq \frac{4d^5}{NC_1^2\eps^{318}\delta^2},
    2\leq i\leq d.
    \]
    Let the sample size $N_1 = \frac{4d^6}{C_1^2\eps^{318}\delta^2 \xi}=O(Cd^6\eps^{-318}\delta^{-2}\xi^{-1})$, and we have
    \[
    \mathbb{P}(|\tilde{\mu}_1| >  \frac{C_1\eps^{159}}{2d^2}) > 1 - \frac{\xi}{d},\quad
    \mathbb{P}(|\tilde{\mu}_i| < \frac{C_1\eps^{159}\delta}{2d^2\sqrt{d}}) > 1-\frac{\xi}{d},2\leq i\leq d.
    \]
    So we have 
    \begin{align*}
        \mathbb{P}(\sin\theta (\tilde{\mu},e_1)\leq \delta)
        = &\mathbb{P} (\frac{\tilde{\mu}_1^2}{\sum_{i=1}^d \tilde{\mu}_i^2} \geq 1-\delta^2)\\
        \geq & \mathbb{P} (|\tilde{\mu}_1| \geq\frac{C_1\eps^{159}}{2d^2}, |\tilde{\mu}_i| \leq \frac{C_1\eps^{159}\delta}{2d^2\sqrt{d}},2\leq i\leq d )\\
        \geq & 1-\xi
    \end{align*}
    This indicates that with probability $1-\xi$, the re-weighted mean can output the vector $\tilde{\mu}$ that is within angle $\delta$ to the vector $e_1$.

    Secondly, for the case when $a$ and $b$ are near-symmetric. Denote $\Sigma$ as the re-weighted covariance matrix with eigenvalues $\lambda_i$ and $\tilde{\Sigma}$ as the empirical re-weighted covariance matrix with eigenvector $v$.
    By Lemma~\ref{lemma_reweighted_cov}, $\lambda_1-\lambda_2 >C_2\eps^3 \lambda_1$. By Lemma~\ref{lemma_cov_estimate_2eps}, with sample size $N_2 = \tilde{O}(d\eps^{-6}\delta^{-2})$, with probability $1-\xi$,
    \[
    \|\Sigma-\tilde{\Sigma}\|_{\text{op}}\leq C_2\eps^{3}\delta\|\Sigma\|_{\text{op}} /2
    \]
    By Lemma~\ref{lemma_sin_theta},
    \[
    \sin\theta(e_1,v)\leq \frac{2\|\Sigma-\tilde{\Sigma}\|_{\text{op}} }{C_2\eps^3 \lambda_1}\leq 
    \frac{ C_2\eps^{3}\delta \lambda_1}{C_2\eps^3 \lambda_1}=\delta
    \]
    So given $N = \max(N_1,N_2)=\text{poly}(d,\eps^{-1},\delta^{-1})$, the algorithm learns the distribution $P$ w.h.p.

\end{proof}

\section{Proofs}
\label{section:proof}

\subsection{Preliminaries}

\subsubsection{Logconcave Distributions}

\begin{lemma}[Lemma 5.4, \cite{lovasz2007geometry}]\label{lemma_logconcave_mean}
Let $X$ be a random point drawn from a one-dimensional logconcave distribution. Then
    \[
    \mathbb{P}(X \geq \mathbb{E} X) \geq \frac{1}{e}.
    \]
\end{lemma}

\begin{lemma}[Lemma 5.5,\cite{lovasz2007geometry}]\label{lemma_logconcave_density_upper_bound}
    Let $p:\R\rightarrow\R_+$ be an isotropic logconcave density function. Then we have 
    \begin{enumerate}
        \item[(a)] For all $x$, $g(x)\leq 1$.
        \item[(b)] $g(0) \geq 1/8$.
    \end{enumerate}
\end{lemma}

\begin{lemma}[Lemma 5.6,\cite{lovasz2007geometry}]\label{lemma_logconcave_tail_px}
    Let $X$ be a random point drawn from a logconcave density function $p:\R\rightarrow \R_+$. Then for every $c\geq 0$,
    \[
    \mathbb{P}(p(X)\leq c)\leq \frac{c}{\max_x p(x)}
    \]
\end{lemma}

\begin{lemma}[Lemma 5.7, \cite{lovasz2007geometry}]\label{lemma_logconcave_tail}
    Let $X$ be a random variable drawn from a logconcave distribution in $\R$. Assume that $\EE X^2\leq 1 $. Then for any $t\geq 1$, we have
    \[
    \mathbb{P}(X\geq t)\leq e^{1-t}
    \]
\end{lemma}

\begin{corollary}\label{cor_logconcave_tail}
    Let $X$ be a random point drawn from an isotropic symmetric logconcave density function $p:\R\rightarrow \R_+$. Then we have for $t\geq 0$, we have
    \[
    \mathbb{P}(X\geq t)\leq 8p(t)
    \]
\end{corollary}
\begin{proof}
    Since $p(x)$ is symmetric, we know $p(x)$ is monotonically decreasing for $x\geq 0$. Then we apply  Lemma~\ref{lemma_logconcave_tail_px} with $c=p(x)$, and get
    \[
    \mathbb{P}(x\geq t)\leq
    \mathbb{P}(p(X) \leq p(t))\leq \frac{p(t)}{\max_x p(x)}
    \]
    On other hand, by Lemma~\ref{lemma_logconcave_density_upper_bound}, we have $\max_x p(x) \geq p(0) \geq 1/8$. So we have
    \[
    \mathbb{P}(X\geq t) \leq 8p(t).
    \]

\end{proof}

\begin{lemma}[Theorem 5.22, \cite{lovasz2007geometry}]\label{lemma_logconcave_moments}
    For a random point $X$ drawn from a logconcave distribution in $\R^d$, then
    \[
    \EE |X|^k \leq (2k)^k (\mathbb{E}(|X|))^k
    \]
\end{lemma}

\begin{lemma}\label{lemma_logconcave_decayfast}
Let $X$ be a random point drawn from an isotropic symmetric logconcave density function $p:\R\rightarrow \R_+$. Then for any $t\geq 3$, we have
    \[
    p(t)\leq p(0) \cdot 2^{-t/3}
    \]
\end{lemma}
\begin{proof}
    First we claim that $p(3) < p(0)/2$. Otherwise,
    \[
    \EE X^2 \geq \int_0^3 x^2p(x)\,dx \geq  \frac{p(0)}{2}\frac{3^3}{3} \geq  \frac{9}{16}> \frac{1}{2}
    \]
    This leads to the contradiction. Then for any $t \geq 3$, from the logconcavity of $p$,
    \[
    p(3) \geq p(0)^{1-\frac{3}{t}} p(t)^{\frac{3}{t}}
    \]
    This implies that
    \[
    p(t) \leq p(0) \cdot 2^{-t/3}
    \]
\end{proof}

\subsubsection{Descartes' Rule of Signs}
\label{sec:appendix_descartes}

Descartes' Rule of Signs is a well-known principle in algebra that offers a way to estimate the maximum number of real roots for any polynomial. This classical theorem can be stated as follows:

\begin{thm}[Descartes' Rule of Signs]\label{thm_descartes}
    For the generalized Dirichlet polynomial 
    \[
    F(x) = \sum_{j=1}^n a_j e^{p_j x},p_1\geq p_2\geq \cdots \ge p_n,
    \]
    the number of roots of $F(x)=0$ is at most the number of sign changes in the series $\{a_1,a_2,\cdots,a_n\}$.
\end{thm}

In this section, we state and prove a variant of
Descartes Rule of Signs in the integral form, which we apply directly to prove Lemma~\ref{lemma_contrastive_mean_qual}. To begin with, we say a function $f$ has a root of order $k$ at point $x$ if
\[
f(x)=f'(x)=\cdots=f^{(k-1)}(x)=0 \text{ and }f^{(k)}(x)\neq 0.
\]
We denote $Z(f)$ as the number of roots of $f$, counted with their orders. Then we can show that the number of roots of $f$ is upper bounded by one plus the number of roots of $f'$ in Lemma~\ref{lemma:zeros_derivative}. We use Rolle's Theorem in the proof of the lemma.

\begin{thm}[Rolle's Theorem]
    Suppose that a function $f$ is differentiable at all points of interval $[a,b]$ and $f(a)=f(b)$. Then there is at least one point $x_0\in (a,b)$ such that $f'(x_0)=0$.
\end{thm}

\begin{lemma}
\label{lemma:zeros_derivative}
    $Z(f)\leq Z(f')+1$.
\end{lemma}
\begin{proof}
    Let $f$ as a root of order $k_r$ as $x_r,1\leq r\leq n$. Then $f'$ has a root of order $k-1$ at $x_r$. These add up to 
    \[
    \sum_{r=1}^n(k_r-1) = Z(f)-n
    \]
    By Rolle's Theorem, $f'(x)$ also has at least $n-1$ roots in the gaps between the points $x_r$. Together, these two facts give
    \[
    Z(f') \geq Z(f)-n+n-1=  Z(f)-1. 
    \]
\end{proof}

\begin{thm}[Descartes' Rule of Signs in the Integral Form] \label{thm_descartes_integral}
    Let $F(\alpha)=\int_0^\infty a(x)e^{\alpha x^2}\,dx$. Then the number of roots of $F(\alpha)=0$ is at most the number of sign changes in $a(x)$ for $x \geq 0$.
\end{thm}
\begin{proof}
    We prove the theorem with induction on the number of sign changes of $a(x)$. For the base case when $a(x)=0$, we assume wlog that $a(x) \geq 0$. Then $F(\alpha)>0$ and thus $F(\alpha)$ has no root. Now we assume that the theorem holds for $a(x)=t$ and we will show the $a(x)=t+1$ case. 

    Let one of the sign changes of $a(x)$ occurs at $x_0$. Define
    \[
    F_0(\alpha):=\int_0^\infty a(x) e^{\alpha (x^2-x_0^2)}\,dx,
    \]
    which has the same roots as $F(\alpha)$. By taking derivative, we get
    \[
    F_0'(\alpha) = \int_0^\infty a(x)(x^2-x_0^2) e^{\alpha (x^2-x_0^2)}\,dx.
    \]
    Let $b(x) = a(x)(x^2-x_0^2)$ be the new sequence. Then $b(x)$ has one less sign changes than $a(x)$. By induction hypothesis, the number of roots of $F_0'(\alpha)$ is upper bounded by the number of sign changes of $b(x)$. By Lemma~\ref{lemma:zeros_derivative}, the number of roots of $F_0$ is upper bounded by the number of sign changes of $a(x)$, thus leading to the induction step.
    
\end{proof}

\subsection{Isotropic Isoperimetric Distribution with $\eps$-Margin}
\label{section:proof_iso}

\isomean*
\begin{proof}
    Since $q_1$ is $\psi-$isoperimetric, $\forall t\in [a,b]$,
    $q_1(t) \geq \psi \int_t^\infty q_1(x)\,dx \geq \psi \eps$. Then we have
    \[
    \int_a^b xq_1(x)\,dx > \psi \eps \int_a^b x\,dx > \frac{c^2\psi\eps^3}{2}
    \]
    By the definition of expectation, we have
    \[
    |\E\limits_{x\sim \hat{q}} x |
    = \frac{|\int_{\R\backslash [a,b]} x q_1(x)\,dx|}{\int_{\R\backslash [a,b]}  q_1(x)\,dx}
    = \frac{|\int_{a}^b x q_1(x)\,dx|}{\int_{\R\backslash [a,b]}  q_1(x)\,dx}
    > \frac{c^2\eps^3}{2}
    \]
    On the other hand, we can calculate the variance as follows.
    \begin{align*}
        \mathop{\var}\limits_{x\sim \hat{q}}x 
        \leq \mathop{\mathbb{E}}\limits_{x\sim \hat{q}} x^2
        = \frac{\int_{\R\backslash [a,b]} x^2 q_1(x)\,dx}
        {\int_{-\infty}^a q_1(x)\,dx + \int_b^\infty q_1(x)\,dx}
        \leq \frac{\int_\R x^2q_1(x)\,dx}{2\eps} = \frac{1}{2\eps}
    \end{align*}
\end{proof}

\begin{lemma}[Second Moment]\label{lemma_iso_tool_cov}
    For $a,b$ satisfying $a\leq 0<b,b>c\eps$ for constant $c>0$, we have
    \[
    \E\limits_{x\sim \hat{q}}x^2 > 1+C\eps^2 \text{  for constant }C>0.
    \]
\end{lemma}
\begin{proof}
    By definition of $\hat{q}$, we know its density on the support $x\in\R\backslash [a,b]$ is
    \[
    \hat{q}(x) = \frac{q_1(x)}{\int_{-\infty}^a q_1(x)\,dx + \int_b^\infty q_1(x)\,dx}
    \]
    Then we calculate its second moment as follows.
    \begin{align*}
        \E\limits_{x\sim \hat{q}}x^2 = \frac{\int_{-a}^\infty x^2q_1(x)\,dx + \int_{b}^\infty x^2q_1(x)\,dx}
        {\int_{-a}^{\infty} q_1(x)\,dx + \int_b^\infty q_1(x)\,dx}
    \end{align*}
    Define $g(x):=\int_x^\infty (t^2-1)q_1(t)\,dt,x\geq 0$. Its derivative is 
    $
    g'(x) = (1-x^2)q_1(x)$. So we know $g(x)$ is monotonically increasing when $x\in[0,1]$, and decreasing when $x\geq 1$. Since $P_1$ is symmetric and isotropic, we know $\int_0^\infty q_1(x)\,dx = \int_0^\infty x^2q_1(x)\,dx = 1/2$. So we have $g(0) =0$. This derives that $g(x)\geq 0,\forall x\geq 0$. In other words, $\int_{-a}^\infty x^2 q_1(x)\,dx \geq \int_{-a}^\infty q_1(x)\,dx$.
    
    For any $x\in[c\eps,M]$, we have $g(x) \geq \min (g(c\eps),g(M))$. Here we let $M>0$ such that $\int_M^\infty q_1(x)\,dx = \eps$. 
    Then we can lower bound $g(c\eps)$ as follows.
    \[
    g(c\eps)=g(0)+\int_0^{c\eps} g'(x)\,dx = \int_0^{c\eps} (1-x^2)q_1(x)\,dx > c\eps(1-c^2\eps^2)q_1(c\eps) > \psi c\eps^2 (1-c^2\eps^2)
    \]
    
    If $M\leq 1$, we know $|b|<M \leq 1$. Then we have $g(b) \geq g(c\eps).$ 
    
    If $M\geq1+\eps$, we can lower bound $g(M)$ as
    \begin{align*}
        g(M) > (M^2-1)\int_M^\infty q_1(x)\,dx > ((1+\eps)^2-1)\eps = \eps^3 + 2\eps^2
    \end{align*}
    Similarly we will get $g(b) > \min (\psi c\eps^2 (1-c^2\eps^2), \eps^3+2\eps^2)$.
    Finally if $1<M<1+\eps$, there exists $M'>0$ such that $\int_M^{M'}q_1(x)\,dx = \eps/2$. Here $M'-M >\eps/2$. Then we have
    \[
    g(M') > (M'^2-1)\int_{M'}^\infty q_1(x)\,dx > ((1+\eps /2)^2-1)\eps/2 = \eps^3/8 + \eps^2/2
    \]
    In this case, we have $g(b) > \min (g(\eps),g(M')) > \min(\psi c\eps^2 (1-c^2\eps^2),\eps^3/8 + \eps^2/2)$.
    Therefore, we can lower bound the second moment of $\hat{q}$ as follows.
    \begin{align*}
        \E\limits_{x\sim \hat{q}}x^2
        >&  \frac{\int_{-a}^\infty q_1(x)\,dx + \int_b^{\infty}q_1(x)\,dx + g(b)}{\int_{-a}^\infty q_1(x)\,dx + \int_b^{\infty}q_1(x)\,dx}\\
        >& 1 + \frac{\min (\psi c\eps^2 (1-c^2\eps^2), \eps^3/8 + \eps^2/2)}{\int_{-a}^\infty q_1(x)\,dx + \int_b^{\infty}q_1(x)\,dx}\\
    >& 1 +  \min (\psi c\eps^2 (1-c^2\eps^2), \eps^3/8 + \eps^2/2)\\
    > & 1 + C\eps^3 \text{   where }C=\min(\psi c/2,1/8)
    \end{align*}

\end{proof}

\isospectralgap*
\begin{proof}
    We assume wlog that $e_1=u$. That is the marginal distribution of $P$ in $e_1$ is $\hat{q}$ while for $2\leq i\leq d$, the marginal distribution in $e_i$ is $q_i$. Since $q_i$ is isotropic, for any $2\leq i\leq d$, 
    $\E\limits_{x\sim P}x_i^2 = 1$. By Lemma~\ref{lemma_iso_tool_cov}, we have $\E\limits_{x\sim P}x_1^2 >1+C'\eps^3$ for constant $C'>0$. Let $g(v):= \E\limits_{x\sim P} \frac{v^\top xx^\top v}{v^\top v}$. Then we have $g(e_1)>1+C'\eps^3$, while $g(e_i)=1,\forall 2\leq i\leq d$. Then for any unit vector $v = \sum_{i=1}^d \beta_i e_i$ 
    satisfying $\sum_{i=1}^d \beta_i^2=1$, we have
    \begin{align*}
        g(v) = \sum_{i=1}^d \beta_i^2 g(e_i) \leq g(e_1)
    \end{align*}
    Then we know the top eigenvalue of $\Sigma$ is $\lambda_1 > 1+C\eps^3$.
    Furthermore, the top eigenvector corresponds to $e_1$.
    Similarly, the second eigenvalue of $\Sigma$ is
    $\lambda_2 = \max_{v:v\bot e_1}g(v) = g(e_i)=1,2\leq i\leq d$. This implies that $\lambda_1 - \lambda_2 > C\eps^3\lambda_1$ for constant $C>0$.

\end{proof}

\subsection{Affine Product Distribution with $\eps$-Margin}

In this section, we prove the lemmas in the general setting. We prove the two qualitative lemmas (Lemma~\ref{lemma_contrastive_mean_qual} and Lemma~\ref{lemma_contrastive_covariance_qual}) in Section \ref{section:proof_qualitative}, and then prove the quantitative lemmas (Lemma~\ref{lemma_reweighted_mean} and Lemma~\ref{lemma_reweighted_cov}) in the remaining section. For the quantitative part, we first consider the asymmetric case where $|a+b|\geq \eps^{5}$. In this case, contrastive mean leads to recovering $u$, as elaborated in Section \ref{sec:proof_mean_asym}. Secondly, we consider the symmetric case characterized by $a+b=0$, addressed in Section \ref{sec:proof_convariance_sym}. We show that we can recover $u$ by calculating the top eigenvector of the re-weighted covariance matrix. Finally we extend this result to near-symmetric case where $|a+b|<\eps^{5}$ in Section \ref{sec:proof_covariance_almost_sym}.

Recall that we are given data $x^{(1)},\cdots,x^{(N)}$ drawn from the affine product distribution with $\eps$-margin $\hat{P}$. Algorithm~\ref{algo_general} first makes the data isotropic. Denote $y^{(1)},\cdots,y^{(N)}$ as the corresponding isotropicized data. Then each $y^{(j)}$ is an independent and identically distributed variable drawn from $P=\hat{q}\otimes q\otimes\cdots q$. Since we compute the re-weighted moments on $y^{(j)}$ in the algorithm, we analyze the moments of $P$ directly.

Recall in Definition~\ref{def_1} that $q$ is the symmetric one-dimensional isotropic logconcave density function, and $\tilde{q}$ is the density obtained by restricting $q$ to $\R\backslash [a,b]$ for some unknown $a<b$. Denote $\mu_1,\sigma_1^2$ as the mean and variance of $\tilde{q}$. $\hat{q}$ is the density obtained after making $\tilde{q}$ isotropic, with support $\R\backslash[a',b']$, where  $a' = \frac{a-\mu_1}{\sigma_1}, b' = \frac{b-\mu_1}{\sigma
_1}$. The density $\hat{q}$ on its support is
\[
\hat{q}(x) = \frac{\sigma_1 q(x\sigma_1+\mu_1)}{\int_{-\infty}^a q(x)\,dx + \int_b^\infty q(x)\,dx} 
\]
We denote the standard basis of $\R^d$ by $\{e_1,\cdots,e_d\}$, and assume wlog that $e_1=u$ is the (unknown) normal vector to the band.
We write $x_i:=\langle x,e_i\rangle$ as $x$'s $i$-th coordinate.
We assume in our proof that $|b|>|a|$. If this condition is not met, we can redefine our interval by setting $a'=-b$ and $b'=-a$. The proof can then be applied considering the distribution is restricted to $\{x\in \R^d: u^\top x \leq a' \text{ or }u^\top x \geq b'\} $. For a vector $x\in\R^d$, we use $\|x\|$ to denote its $l_2$ norm. For a matrix $A\in\R^{m\times n}$, we denote its operator norm as $\|A\|_{\text{op}}$.

% To make sure that there is at least $\eps$ fractional of mass on each side, we have $\int_b^\infty q(x)\,dx \geq \eps, \int_{-\infty}^a q(x)\,dx \geq \eps$. Since $q$ is logconcave, by Lemma~\ref{lemma_logconcave_tail}, we have $|a|,|b| < 1 + \ln(1/\eps)$. By the definition of the distribution $\hat{P}$, we have $\eps \leq \int_a^\infty q(x)\,dx\leq 1-\eps$.

\subsubsection{Proofs of Qualitative Bounds}
\label{section:proof_qualitative}
We present proofs of two qualitative lemmas: the contrastive mean (Lemma~\ref{lemma_contrastive_mean_qual}) and the contrastive covariance (Lemma~\ref{lemma_contrastive_covariance_qual}). Their quantitative counterparts can be found in Section~\ref{sec:proof_mean_asym} and Section~\ref{sec:proof_covariance_almost_sym}. 
To establish the contrastive mean, we invoke Descartes' Rule of Signs. For the proof concerning contrastive covariance, we introduce a novel monotonicity property on the moment ratio,
as described in Lemma~\ref{lemma_variance_ratio}. We include the proof within this section.

% For a vector $x\in\R^d$, we use $\|x\|$ to denote its $l_2$ norm. For a matrix $A\in\R^{m\times n}$, we denote its operator norm as $\|A\|_{\text{op}}$.
% We denote the standard basis of $\R^d$ by $\{e_1,\cdots,e_d\}$, and assume wlog that $e_1=u$ is the (unknown) normal to the band.
%For any vector $x\in\R^d$, we denote $x_i:=x^\top e_i$ for $i\in\{1,\cdots,d\}$. 

    \begin{figure}[htp]
        \centering
        \includegraphics[height=2in]{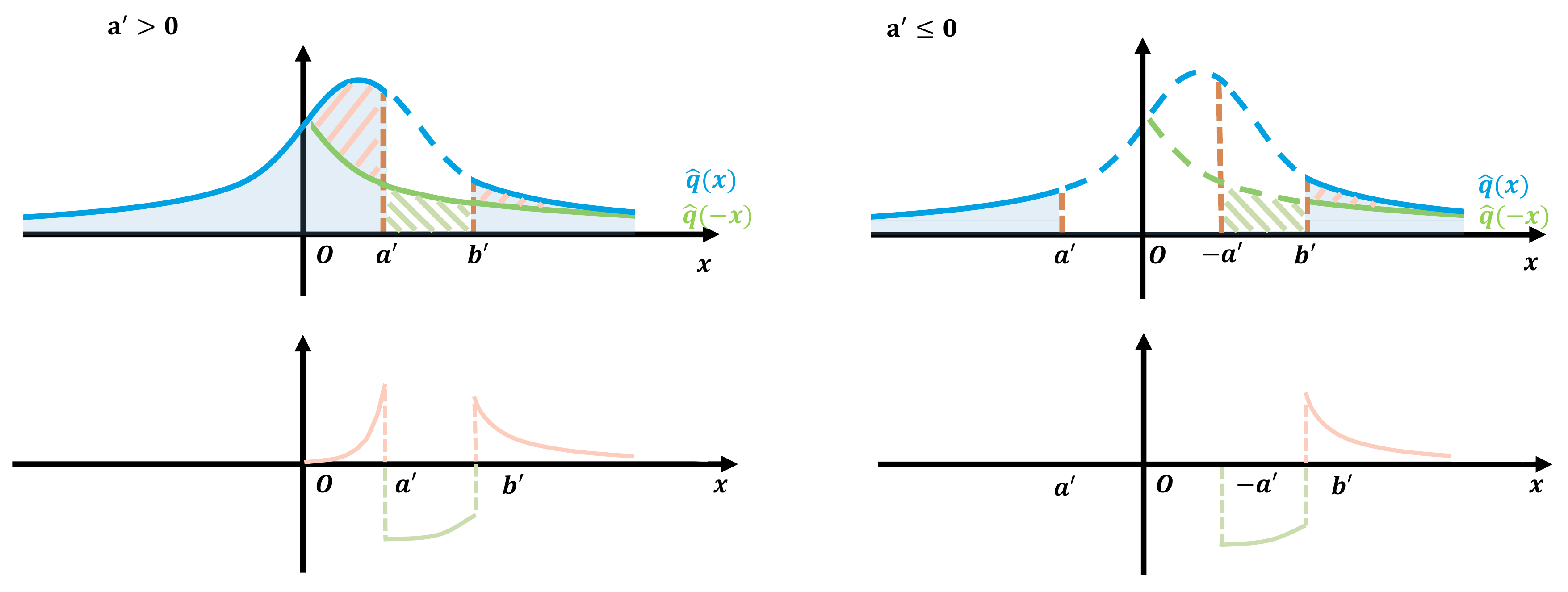}
        \caption{Coefficients of $F(\alpha)$ ahead of $e^{\alpha x^2} $ term in Lemma~\ref{lemma_contrastive_mean_qual}'s proof. By combining $e^{\alpha x^2}$ terms, we flip $\hat{q}(x)$ horizontally.
        For $a'>0$, the coefficient is negative when $x\in(a',b')$ and non-negative outside the interval. For $a'\leq 0$, it is negative when $x\in(-a',b') $ and positive when $x>b'$.}
        \label{fig_lemma_first_moment.}
    \end{figure}

\paragraph{Contrastive Mean.}
We can write the contrastive mean as a linear combination of exponential functions of $\alpha$. By Descartes' rule of signs, the number of zeros of this function is at most two. Since $\alpha=0$ is one root and corresponds to mean zero, there is at most one nonzero root. And thus we have that for any two distinct nonzero $\alpha$'s, at least one of them achieves nonzero contrastive mean.

\contrastivemeanqual*
\begin{proof}
% [Proof of Lemma~\ref{lemma_contrastive_mean_qual} (Contrastive Mean)]
    % We assume wlog that $a+b>0$.
    % Let $\mu_1,\sigma_1^2$ be the mean and covariance of the density obtained after restricting $q$ to $\R\setminus [a,b]$, so we have $\mu_1<0$. Since $\hat{q}$ is isotropic, we denote its support as $\R\backslash [a',b']$, where $a' = (a-\mu_1)/\sigma_1, b'=(b-\mu_1)/\sigma_1$. Then we have $a'+b'= (a+b-2\mu_1)/\sigma_1>0$. 
     $|b|>|a|$ implies that $\mu_1<0$.
    For any $x\geq 0$, we have 
    \[
    \hat{q}(x) = \frac{\sigma_1 q(x\sigma_1+\mu_1)}{1-\int_a^b q(x)\,dx}
    \geq
    \frac{\sigma_1 q(-x\sigma_1+\mu_1)}{1-\int_a^b q(x)\,dx} = 
    \hat{q}(-x).
    \]    
    Since $P$ is a product distribution, we have
    \[
    \E\limits_{x\sim P}e^{\alpha \|x\|^2}x_1 = \E\limits_{x_1\sim \hat{q}}e^{\alpha x_1^2}x_1 \cdot \prod_{i=2}^d \E\limits_{x_i\sim q} e^{\alpha x_i^2}
    \]
    We denote
    \begin{align}
        F(\alpha)=\E\limits_{x\sim \hat{q}}e^{\alpha x_1^2 }x_1
    \end{align}
    By calculation, we have
    \begin{align*}
        F(\alpha) = \int_{\R\backslash [a',b']}e^{\alpha x^2}x\hat{q}(x)\,dx
    \end{align*}
    Then we rearrange $F(\alpha)$ by combining $e^{\alpha x^2}$ as in Figure~\ref{fig_lemma_first_moment.}.
    
    If $a'\leq 0$, we rewrite $F(\alpha)$ as
    \[
    F(\alpha) = -\int_{-a'}^{b'}x\hat{q}(-x) e^{\alpha x^2}\,dx +
    \int_{b'}^\infty x(\hat{q}(x) - \hat{q}(-x))e^{\alpha x^2}\,dx
    \]
    We treat $F(\alpha)$ as the integral of $a(x)e^{\alpha x^2}$ for $x \geq -a'$.
    Since $\hat{q}(x)-\hat{q}(-x) > 0$ for $x>b'$, we have $a(x)>0$ for $x\in(-a',b')$ and $a(x)<0$ for $x>b'$. In other words, for increasing $x$, the sign of $a(x)$ only changes once. By Theorem~\ref{thm_descartes_integral}, $F(\alpha)=0$ has at most one root.
    
    If $a' > 0$, 
    we arrange $F(\alpha)$ in the same way and get
    \[
    F(\alpha) = \int_0^{a'}x(\hat{q}(x) - \hat{q}(-x))e^{\alpha x^2}\,dx
    -\int_{a'}^{b'}x\hat{q}(-x) e^{\alpha x^2}\,dx +
    \int_{b'}^\infty x(\hat{q}(x) - \hat{q}(-x))e^{\alpha x^2}\,dx
    \]
    Similarly, we treat $F(\alpha)$ as the integral of $a(x)e^{\alpha x^2}$ for $x \geq 0$. For increasing $x$, the sign of $a(x)$ changes twice. By Descartes' rule of signs, $F(\alpha)=0$ has at most two roots.
    In addition, we know $F(\alpha) = \E\limits_{P} x_1=0$ by definition of $P$. So $\alpha=0$ is one root of $F(\alpha)=0$. So there is at most one nonzero root of $F(\alpha)=0$. In other words, for any two distinct nonzero $\alpha_1,\alpha_2$, at least one of $F(\alpha_1),F(\alpha_2)$ is nonzero. This implies that
    \[
    \max\left(\left\lvert\E\limits_{x\sim P} e^{\alpha_1 \|x\|^2}x\right\lvert,  \left\lvert\E\limits_{x\sim P} e^{\alpha_2 \|x\|^2}x\right\rvert
    \right)> 0.
    \]
\end{proof}

\paragraph{Moment Ratio.} 
To prove Lemma~\ref{lemma_contrastive_covariance_qual}, we develop a new monotonicity property of the moment ratio of logconcave distributions. Moment ratio is specifically defined as the ratio of the fourth moment to the square of the second moment of truncated versions of the distribution. This measurement essentially reflects the uncentered kurtosis of the distribution.
The formal definition is detailed in Definition \ref{def:variance_ratio}. 
% The critical insight from the monotonicity of this moment ratio is that for logconcave distributions with positive support, when the distribution is restricted to a smaller interval right to the origin, it necessitates a larger sample size to accurately estimate its second moment.

\begin{defn}[One-side $t$-restriction distribution]
    Let $q$ be a distribution in one dimension with nonnegative support.
     For any $t \geq 0$, define $q_t$ as the one-side $t$-restriction distribution on $q$ obtained by restricting $q$ to $[t,\infty)$.
\end{defn}

\begin{defn}[Moment Ratio] \label{def:variance_ratio}
    Let $q$ be a distribution in one dimension with nonnegative support. For any $t\geq 0$,
    define $q$'s moment ratio as a function of $t$, given by 
    \[
    \mr_q(t):= \frac{\var_{q_t}(X^2)}{(\E_{q_t}X^2)^2},
    \quad \text{ where }q_t \text{ is the one-side }t \text{-restriction distribution on }q.
    \]
\end{defn}

% \begin{restatable}[Monotonicity of Variance Ratio]{lemma}{varianceratio}
% \label{lemma_variance_ratio_formal}
%    Let $q$ be a logconcave distribution in one dimension with nonnegative support. Its moment ratio $\mr_q(t)$ is stricting decreasing with $t$ for any $t\geq 0$.
% \end{restatable}

% For a one-dimensional density function $q$, and $t\geq 0, k\in\mathbb{N}$, we define the $k$'th moment of the $t-$restriction of $q$ as $M_k(t) = \int_t^\infty x^k q(x)\,dx$. Define the moment ratio $G(t) = M_0(t)M_4(t)/M_2^2(t)$. 
 We will prove the monotonicity of the moment ratio (Lemma~\ref{lemma_variance_ratio}) by reducing general logconcave distributions to exponential distributions. The monotonicity of the moment ratio for exponential distribution is detailed in Lemma~\ref{lemma_exp_inf_h_pos}. 

% \begin{lemma}
% \label{lemma_vr_calculate}
%     Denote $M_k(t)= \int_t^\infty x^k q(x)\,dx$. Then we can compute the moment ratio as
%     \[
%     \mr_q(t) = \frac{M_0(t)M_4(t)}{M_2^2(t)}-1
%     \]
% \end{lemma}
% \begin{proof}
%     By definition, 
%     \[
%     \mr_q(t) =
%     \frac{\var_{q_t}(X^2)}{(\E_{q_t}X^2)^2}
%     =
%     \frac{\frac{\int_{t}^\infty x^4q(x)\,dx }{\int_{t}^\infty q(x)\,dx}}
%     {\left(\frac{\int_{t}^\infty x^2q(x)\,dx}
%     {\int_{t}^\infty q(x)\,dx}
%     \right)^2}-1
%     = \frac{M_0(a)M_4(a)}{M_2(t)^2} - 1 
%     \]
% \end{proof}

\begin{lemma}[Monotonicity of Moment Ratio of Exponential Distribution]\label{lemma_exp_inf_h_pos}
    Define $h(x)=\beta e^{-\gamma x},x\geq 0,\beta,\gamma>0$. Denote $N_k(t) = \int_t^\infty x^k h(x)\,dx$. Then for any $t\geq 0$, we have 
    \[
    t^4N_0(t)N_2(t)+N_2(t)N_4(t)-2t^2N_0(t)N_4(t)> 0.
    \]
    \end{lemma}
    \begin{proof}
    By calculation, we have
    \begin{align*}
        N_0(t)=& \int_t^\infty \beta e^{-\gamma x}\,dx
        =
        \frac{\beta}{\gamma}e^{-\gamma t}
    \end{align*}
    \begin{align*}
        N_1(t) = \int_t^\infty \beta xe^{-\gamma x}\,dx
        =\frac{\beta}{\gamma}
        te^{-\gamma t} + \frac{1}{\gamma}N_0(t)
        = \left(t + \frac{1}{\gamma}\right)\frac{\beta}{\gamma}e^{-\gamma t}
    \end{align*}
    \begin{align*}
        N_2(t) =&  \int_t^\infty \beta x^2e^{-\gamma x}\,dx
        =\frac{\beta}{\gamma}t^2e^{-\gamma t} + \frac{2}{\gamma}N_1(t)
        =\left(
       t^2 + \frac{2}{\gamma}t + \frac{2}{\gamma^2}
        \right)\frac{\beta}{\gamma}e^{-\gamma t}
    \end{align*}
    \begin{align*}
        N_3(t) = \int_t^\infty \beta x^3e^{-\gamma x}\,dx
        =\frac{\beta}{\gamma}t^3e^{-\gamma t}
        + \frac{3}{\gamma}N_2(t)
        = \left(
        t^3 + \frac{3}{\gamma}t^2
        + \frac{6}{\gamma^2}t + \frac{6}{\gamma^3}
        \right)\frac{\beta}{\gamma}e^{-\gamma t}
    \end{align*}
    \begin{align*}
        N_4(t) = \int_t^\infty \beta x^4 e^{-\gamma x}\,dx
        = \frac{\beta}{\gamma }t^4e^{-\gamma t} + \frac{4}{\gamma}N_3(t)
        = \left(
        t^4 + \frac{4}{\gamma}t^3 + \frac{12}{\gamma^2}t^2 + \frac{24}{\gamma^3}t + \frac{24}{\gamma^4}
        \right)\frac{\beta}{\gamma}e^{-\gamma t}
    \end{align*}
    Then we can plug them and get
    \begin{align*}
        &t^4N_0(t)N_2(t)+N_2(t)N_4(t)-2t^2N_0(t)N_4(t)\\
        = & \frac{\beta^2}{\gamma^2}e^{-2\gamma t}
        \left(
        t^4\left(
       t^2 + \frac{2}{\gamma}t + \frac{2}{\gamma^2}
        \right)
        + \left(
       -t^2 + \frac{2}{\gamma}t + \frac{2}{\gamma^2}
        \right)\cdot
        \left(
        t^4 + \frac{4}{\gamma}t^3 + \frac{12}{\gamma^2}t^2 + \frac{24}{\gamma^3}t + \frac{24}{\gamma^4}
        \right)
        \right)\\
        =&\frac{8\beta^2}{\gamma^2}e^{-2\gamma t}\left(
        \frac{t^3}{\gamma^3}+\frac{6t^2}{\gamma^4} + \frac{12t}{\gamma^5} + \frac{6}{\gamma^6}
        \right)>0 
    \end{align*}
    \end{proof}

\noindent Next, we will prove the monotonicity of moment ratio for logconcave distributions.
\varianceratio*
\begin{proof}
% [Proof of Lemma~\ref{lemma_variance_ratio}]
Denote $M_k(t)= \int_t^\infty x^k q(x)\,dx$. By Definition~\ref{def:variance_ratio},
    \[
    \mr_q(t) =
    \frac{\var_{q_t}(X^2)}{(\E_{q_t}X^2)^2}
    =
    \frac{\frac{\int_{t}^\infty x^4q(x)\,dx }{\int_{t}^\infty q(x)\,dx}}
    {\left(\frac{\int_{t}^\infty x^2q(x)\,dx}
    {\int_{t}^\infty q(x)\,dx}
    \right)^2}-1
    = \frac{M_0(t)M_4(t)}{M_2(t)^2} - 1 
    \]
    Next we will show that $\mr'(t) < 0,\forall t\geq 0$. By taking the derivative,
    \[
    \mr'(t) = \frac{-q(t)}{M_2(t)}(t^4M_0(t)M_2(t)+M_4(t)M_2(t)-2t^2M_0(t)M_4(t))
    \]
    Define $H(t) = t^4M_0(t)M_2(t)+M_4(t)M_2(t)-2t^2M_0(t)M_4(t)$. We will show that $H(t)> 0,\forall x\geq 0$. Clearly $H(0)=M_4(0)M_2(0)>0$, So we only consider $t>0$ in the following proof.
    
    Let $h(x)=\beta e^{-\gamma x}$ be an exponential function ($\beta,\gamma>0$) such that 
    \[
    M_0(t)=N_0(t),M_2(t)=N_2(t),\text{ where }N_k(t) = \int_t^\infty x^k h(x)\,dx,k\in\mathbb{N}.
    \]
    Then we have
    \[
    \int_t^\infty (h(x)-q(x))\,dx=0,\int_t^\infty x^2(h(x)-q(x))=0
    \]
    By the logconcavity of $q$, the graph of $h$ intersects with the graph of $q$ at exactly two points $u'<v$, where $v > 0$. 
    Also we have $h(x)\leq q(x)$ at the interval $[u',v]$ and $h(x)>q(x)$ outside the interval. Let $u=\max\{0,u'\}$. So for $x\geq 0$, $(x-u)(x-v)$ has the same sign as $h(x)-q(x)$. Since $t\geq 0$, we have
    \[
    \int_t^\infty(x^2-u^2)(x^2-v^2)(h(x)-q(x))\geq 0
    \]
    Expanding and we get
    \[
    \int_t^\infty x^4 (h(x)-q(x)) \geq (u^2+v^2)\int_t^\infty x^2(h(x)-q(x))\,dx - u^2v^2\int_t^\infty(h(x)-q(x))\,dx=0
    \]
    This shows that $N_4(t)\geq  M_4(t)$. To show that $H(t)>0$, we consider two cases. 
    
    Firstly if $M_2(t)-2t^2\geq 0$, we have
    \[
    H(t)=t^4M_0(t)M_2(t)+M_4(t)(M_2(t)-2t^2)> 0.
    \]
    
    Secondly if $M_2(t)-2t^2< 0$, by calculation of the exponential function's moments (Lemma~\ref{lemma_exp_inf_h_pos}), we have
    \[
    t^4N_0(t)N_2(t)> -N_4(t)(N_2(t)-2t^2)
    \]
    Then we have
    \begin{align*}
        H(t)=&t^4M_0(t)M_2(t)+M_4(t)(M_2(t)-2t^2)\\
        =& t^4 N_0(t)N_2(t) + M_4(t)(M_2(t)-2t^2)\\
        \geq &-N_4(t) (N_2(t)-2t^2) + M_4(t)(M_2(t)-2t^2)\\
        = &(M_2(t)-2t^2)(M_4(t)-N_4(t))\\
        \geq& 0
    \end{align*}
    The equality holds if and only if $M_4(t)=N_4(t)$. This implies that 
    \[
    H(t) = t^4N_0(t)N_2(t)+N_4(t)(N_2(t)-2t^2) > 0.
    \]
    Combining both cases, $\mr'(t)< 0,\forall t\geq 0$, which implies that the moment ratio of $q$ is strictly decreasing with respect to $t$.
    \end{proof}

\paragraph{Contrastive Covariance.}
View the spectral gap of the re-weighted covariance, denoted as $\lambda_1(\tilde{\Sigma})-\lambda_2(\tilde{\Sigma})$, as $S(\alpha)$. By calculation, $S(0)=0$ and $S'(0) $ is proportional to $\mr(b)-\mr(0)$, which is negative by the monotonicity property of moment ratio. Then we can prove Lemma~\ref{lemma_contrastive_covariance_qual} using Taylor expansion.

\contrastivecovariancequal*
\begin{proof}
% [Proof of Lemma~\ref{lemma_contrastive_covariance_qual} (Constrastive Covariance)]
    Denote $M_k(t)= \int_t^\infty x^k q(x)\,dx$. The variance $\sigma_1^2$ of $q$ restricted to $\R\backslash [-b,b]$ is
    \begin{align*}
    \sigma_1^2
    = \frac{\int_b^\infty x^2p(x)\,dx}{\int_b^\infty p(x)\,dx}
    =  \frac{M_2(b)}{M_0(b)}
    \end{align*}
    Since $\hat{q}$ is isotropic,
    the density on the support $\R\backslash [-b/\sigma_1,b/\sigma_1]$ is
    \[
    \mathbb{P}_{\hat{q}}(x) 
    %= \sigma_1 \mathbb{P}_{\tilde{P}_2}(x\sigma_1)
    =\frac{\sigma_1 q(x\sigma_1)}{2\int_b^\infty q(x)\,dx}
    \]
    Let
    \[
    S(\alpha):= \mathop{\E}\limits_{x\sim \hat{q}}e^{\alpha x^2}x^2 \mathop{\E}\limits_{x\sim q}e^{\alpha x^2} - \mathop{\E}\limits_{x\sim q}e^{\alpha x^2}x^2 \mathop{\E}\limits_{x\sim \hat{q}}e^{\alpha x^2}
    \]
    Since $q$ and $\hat{q}$ are both isotropic, $S(0)= \mathop{\E}\limits_{x\sim \hat{q}}x^2 -\mathop{\E}\limits_{x\sim q}x^2   =0$.
    Then,
\begin{align*}
    S'(0)
    % =& \frac{2}
    % {\sigma_1^2\int_b^\infty q(x)\,dx}
    % \left(
    % \frac{1}{\sigma_1^2} \int_b^\infty y^4q(y)\,dy\cdot \int_0^\infty q(x)\,dx+\int_b^\infty y^2q(y)\,dy \cdot\int_0^\infty x^2q(x)\,dx
    % \right)\\
    % &-\frac{2}{\int_b^\infty q(x)\,dx}
    % \left(
    % \frac{1}{\sigma_1^2}\int_b^\infty y^2q(y)\,dy \cdot \int_0^\infty x^2q(x)\,dx
    % + \int_b^\infty q(y)\,dy \cdot \int_0^\infty x^4q(x)\,dx
    % \right)\\
    % =& 
    % \frac{2M_0(b)M_4(b)M_0(0)}{M_2^2(b)}
    % -2M_4(0)
    % \\
    =& 
    \frac{M_4(b)M_0(b)}{M_2^2(b)}
    - \frac{M_4(0) M_0(0)}{M_2^2(0)}
\end{align*}
The last step is because the $q$ is isotropic. 
% Define $G(t)=M_0(t)M_4(t)/M_2(t)^2$. 
By Lemma~\ref{lemma_variance_ratio}, we know $S'(0) = \mr'(0) < 0$.

On the other hand, $\forall \alpha \leq 0$, $S''(\alpha)$ can be bounded.
\begin{align*}
    S''(\alpha)
    \leq &
    \frac{M_6(b)}{M_0(b)M_2^3(b)} + \frac{M_4(b)}{M_0(b)M_2^2(b)}
    <\text{poly}(1/\eps)
\end{align*}
By Taylor expansion, we know there exists $\alpha<0$ such that
\begin{align*}
    S(\alpha) =& S(0) + \alpha S'(0) + \frac{{\alpha}^2}{2}S''(\alpha') >0 \text{, where }\alpha'\in[\alpha,0]
\end{align*}
Then we have for $2\leq j\leq d$,
\begin{align*}
    \E_{x\sim P} e^{\alpha \|x\|^2}x_1^2 - \E_{x\sim P} e^{\alpha\|x\|^2}x_j^2
    = S(\alpha)
    \left(
    \mathop{\E}\limits_{x\sim q} e^{\alpha x^2}
    \right)^{d-2}>0
\end{align*}
For any $v\in\R^d$, define $\phi(v)$ as
\[
\phi(v):=\mathop{\E}\limits_{x\sim P} \frac{e^{\alpha\|x\|^2}v^\top xx^\top v}{v^\top v}
\]
Substituting $v$ with $e_1$ and $e_j,2\leq j\leq d$, we have
\[
\phi(e_1) = \E_{x\sim P} e^{\alpha \|x\|^2}x_1^2 , \phi(e_j) = \E_{x\sim P} e^{\alpha \|x\|^2}x_j^2 
\]
This implies that for $2\leq j\leq d$,
\[
\phi(e_1)-\phi(e_j)= \E_{x\sim P} e^{\alpha \|x\|^2}x_1^2 - \E_{x\sim P} e^{\alpha \|x\|^2}x_j^2>0 
\]
For any vector $v=\sum_{i=1}^d \gamma_ie_i$, we have
\begin{align*}
    \phi(v) = \frac{1}{\sum_{i=1}^d \gamma_i^2}\E e^{\alpha\|x\|^2}(\sum_{i=1}^d \gamma_i x_i)^2
    =&\frac{\sum_{i=1}^d \gamma_i^2 \phi(e_i)}{\sum_{i=1}^d \gamma_i^2}
    \leq  \phi(e_1)
\end{align*}
This shows that the top eigenvalue of $\tilde{\Sigma}$ is $\lambda_1(\tilde{\Sigma})=\max_v \phi(v)=\phi(e_1)$. Similarly,  $\lambda_2(\tilde{\Sigma}) =\max_{v:v\bot e_1}\phi(v) = \phi(e_j),2\leq j\leq d$. Therefore, $\lambda_1(\tilde{\Sigma}) > \lambda_2(\tilde{\Sigma})$ and the top eigenvector is $e_1$, which is essentially $u$.
\end{proof}

\subsubsection{Quantitative Bounds for Contrastive Mean}
\label{sec:proof_mean_asym}

We will prove Lemma~\ref{lemma_reweighted_mean} in this section. Here we consider the case when $|a+b|\geq  \eps^{5}$.
We compute the contrastive mean of $P$ given $\alpha<0$ as $\E\limits_{x\sim P}e^{\alpha x^2}x$ using two different $\alpha$'s. 

\begin{defn}
    We define $F(\alpha)$ as re-weighted mean for the one-dimensional distribution $\hat{q}$.
    \begin{align}
        F(\alpha)=\E\limits_{x\sim \hat{q}}e^{\alpha x^2 }x = \int_{\R\backslash [a',b']} e^{\alpha x^2}x\hat{q}(x)\,dx
    \end{align}
Since $P$ is isotropic, $F(0) = \E\limits_{x\sim \hat{q}} x=0$.
\end{defn}

To prove Lemma~\ref{lemma_reweighted_mean}, we need to show that for given $\alpha_1,\alpha_2$, the maximum of $|F(\alpha_1)|,|F(\alpha_2)|$ exceeds a certain positive threshold. We follow the same idea of bounding the number of roots of $F(\alpha)$ as in the qualitative lemma (Lemma~\ref{lemma_contrastive_mean_qual}). By taking the derivative of $F(\alpha)$, we can show that either $F'(0)\neq 0$ or $F''(0)\neq 0$. Then by Taylor expansion, we can choose two distinct $\alpha$'s (near zero) so that one of the corresponding contrastive means is bounded away from zero. 

In the process of proving the quantitative bounds, 
similar to our approach with qualitative bounds, we must consider two distinct scenarios based on the sign of $a'$,  as illustrated in Figure~\ref{fig_lemma_first_moment.}. 
\begin{itemize}
    \item In the case where $a'$ is negative, Lemma~\ref{lemma_cm_a_neg} asserts that the first derivative of $F$ at zero, $F'(0)$, is always positive.
    \item Conversely, when $a'$ is nonnegative, Lemma~\ref{lemma_cm_a_pos} reveals an essential characteristic of the function $F(\alpha)$: it's not possible for both $F'(0)$ and $F''(0)$ to be zero at the same time.
    \item Lemma~\ref{lemma_mean_upper_derivatives} provides upper bounds for the derivatives of $F(\alpha)$. These upper bounds are crucial as they help in managing the extra terms that emerge during the Taylor expansion of $F(\alpha)$.
    \item The section concludes with the proof of Lemma~\ref{lemma_reweighted_mean}, which is the quantitative lemma for the contrastive mean. 
\end{itemize}

\noindent We start with Lemma~\ref{lemma_cm_mu1_lowerbound} showing that $|\mu_1|$ is away from zero provided that $|a+b|$ is also different from zero.

\begin{lemma}[Lower Bound of $|\mu_1|$]\label{lemma_cm_mu1_lowerbound}
    If $|a+b|\geq \eps^{s}$ for $s \geq 2$, then $|\mu_1|\geq\eps^{s}/2e$.
    % then for $a\leq 0$, we have 
    % $|\mu_1| > c_\mu\eps$, where $c_\mu = c/(e(1+c))$. For $a>0$, we have $|\mu_1| > \eps^{12}/2e$.
\end{lemma}
\begin{proof}
    Firstly, let's consider the case when $a \leq 0$.
    By Lemma~\ref{lemma_logconcave_density_upper_bound}, we know $q(x)$ is upper bounded by $1$.
    Since $|a+b| \geq \eps^{s}$ and $q$ is logconcave, by Lemma~\ref{lemma_logconcave_mean}, the mean of the density restricted $q$ in $[-a,b]$ satisfies $\mu_{[-a,b]} \geq 1/e$. Then 
    \[
    |\mu_1| = \frac{\mu_{[-a,b]} \int_{-a}^b q(x)\,dx }
    {1- \int_{a}^b q(x)\,dx}
    \geq \frac{1}{e} \frac{\int_{-a}^b q(x)\,dx}{\int_{-a}^b q(x)\,dx + 2\int_b^\infty q(x)\,dx}
    \]
    By Lemma~\ref{lemma_logconcave_tail}, 
    \[
    q(b) \geq 2\int_{b}^\infty q(x)\,dx
    \]
    Since $|a|<|b|$, we have
    \[
    \int_{-a}^b q(x)\,dx \geq (b+a)q(b) \geq 2(b+a)\int_b^\infty q(x)\,dx
    \]
    So we have
    \[
    |\mu_1| \geq \frac{1}{e} \frac{2(b+a)}{2(b+a)+2} 
    \geq \frac{1}{e} \frac{\eps^{s}}{1+\eps^{s}} > \frac{\eps^{s}}{2e}
    \]
    Secondly, when $a>0$, since $b-a>\eps$,
    \[
    |\mu_1| = \frac{\mu_{[a,b]}\int_a^b q(x)\,dx}{1-\int_a^b q(x)\,dx}
    > (a+\frac{\eps}{e}) \frac{\eps}{1-\eps} > \frac{\eps^2}{e}
    \]
\end{proof}

\begin{lemma}[Derivative of $F(0)$ when $a'<0$] \label{lemma_cm_a_neg}
    If $|a+b|\geq \eps^s$ for $s\geq2$ and $a'<0$, then $F'(0)>\eps^{3s+3.5}/2$.
\end{lemma}
\begin{proof}
    We rearrange $F(\alpha)$ by combining terms with same $e^{\alpha x^2}$ as in Figure~\ref{fig_lemma_first_moment.}, and get
    \[
    F(\alpha) = -\int_{-a'}^{b'}x\hat{q}(-x) e^{\alpha x^2}\,dx +
    \int_{b'}^\infty x(\hat{q}(x) - \hat{q}(-x))e^{\alpha x^2}\,dx
    \]
    Define $r(x) = 
    \begin{cases}
    -\hat{q}(-x) & x\in[-a',b'] \\
    \hat{q}(x) - \hat{q}(-x) & x\in [b',\infty)
    \end{cases}$.
    Then we have
    \[
    F(\alpha) = \int_{-a'}^\infty xr(x)e^{\alpha x^2}\,dx
    \]

    By calculating the derivative of $F(\alpha)$, we have
    \begin{align*}
        F'(\alpha)
        =& \int_{-a'}^\infty x^3r(x)e^{\alpha x^2}\,dx \\
    =& \int_{-a'}^\infty x (x^2-{b'}^2)r(x)e^{\alpha x^2}\,dx + {b'}^2 \int_{-a'}^\infty xr(x)e^{\alpha x^2}\,dx\\
    =& \int_{-a'}^\infty x (x^2-{b'}^2)r(x)e^{\alpha x^2}\,dx + {b'}^2 F(\alpha)
    \end{align*}
    Since $r(x)$ is nonnegative for $x\geq b'$ and negative otherwise, then for any $x\geq -a'$, we have
    \[
    x(x^2-{b'}^2)r(x)e^{\alpha x^2}\geq 0
    \]
    Since $F(0)=0$, we have
    \begin{align*}
        F'(0) = \int_{-a'}^\infty x (x^2-{b'}^2)r(x)\,dx
        \geq \int_{-a'}^{b'} x ({b'}^2-x^2)\hat{q}(-x)\,dx
    \end{align*}
    By calculation,
    \begin{align*}
        F'(0) 
        \geq& \frac{1}{1-\int_a^b q(x)\,dx}\int_{-\frac{a-\mu_1}{\sigma
        _1}}^{\frac{b-\mu_1}{\sigma_1}}
        x\left(
        \left(\frac{b-\mu_1}{\sigma_1}\right)^2 -x^2
        \right) \sigma_1 q(-x\sigma_1+\mu_1)\,dx\\
        \geq & \frac{1}{\sigma_1^3(1-\int_a^b q(x)\,dx)} \int_{-a}^{b-2\mu_1}(x+\mu_1) 
        \left(
        (b-\mu_1)^2-(x+\mu_1)^2
        \right) q(x)\,dx\\
        \geq &  \frac{1}{\sigma_1^3(1-\int_a^b q(x)\,dx)} \int_{-a}^{b} (x+\mu_1)(x+b)(b-x-2\mu_1)q(x)\,dx \\
        \geq & \frac{1}{\sigma_1^3
        (1-\int_a^b q(x)\,dx)} (b-a)2|\mu_1| \int_{-a}^b (x+\mu_1) q(x)\,dx
    \end{align*}
    Choose $t_0\in[-a,b]$ such that $\int_{-a}^{t_0}q(x)\,dx = \int_{t_0}^b q(x)\,dx$. Since $q(x)$ is bounded by 1 by Lemma~\ref{lemma_logconcave_density_upper_bound}, we have
    \[
    t_0 + a \geq \frac{\int_{-a}^b q(x)\,dx}{2}
    \]
    On the other hand, similar to the proof of Lemma~\ref{lemma_cm_mu1_lowerbound}, we have
    \begin{align*}
        \int_{-a}^b q(x)\,dx
        \geq 2(b+a )\int_b^\infty q(x)\,dx
        \geq 2\eps^{s+1}
    \end{align*}
    So we have
    \begin{align*}
        \int_{-a}^b (x+\mu_1)q(x)\,dx
        \geq &\int_{t_0}^b (x+a)q(x)\,dx
        \geq (t_0+a)\frac{1}{2}\int_{-a}^b q(x)\,dx
        \geq \frac{1}{4} \left(
        \int_{-a}^b q(x)\,dx
        \right)^2
        \geq \eps^{2s+2}
    \end{align*}
    By definition, we have
    \[
    \sigma_1^2 \leq \frac{\int_{\R\backslash[a,b]} x^2q(x)\,dx}{1-\int_{a}^b q(x)\,dx} \leq \frac{\E\limits_{x\sim q} x^2}{1-\int_a^b q(x)\,dx} \leq \frac{1}{1-\int_a^b q(x)\,dx}
    \]
    Applying Lemma~\ref{lemma_cm_mu1_lowerbound}, we know $|\mu_1| >\eps^{s}/2e$.  Using these results to estimate $F'(0)$, we get
    \begin{align*}
        F'(0)
        \geq & (1-\int_a^b q(x)\,dx)^{0.5} \cdot \eps \cdot \frac{\eps^{s}}{e} \eps^{2s+2}
        > \frac{\eps^{3s+3.5}}{2}
    \end{align*}
    
\end{proof}

\begin{lemma}\label{lemma_a-2mu_b-2mu}
    If $a\geq \mu_1,\mu_1\leq0$, then we have $\int_{a-2\mu_1}^{b} xq(x)\,dx\geq |\mu_1|\int_b^\infty q(x)\,dx$.
\end{lemma}
\begin{proof}
    Firstly we will show that $\int_{2\mu_1-a}^a (x-\mu_1)q(x)\,dx \geq 0$.
    For $x\in[\mu_1,a]$, we have $(2\mu_1-x)^2-x^2 = 4\mu_1(\mu_1-x)\geq0$. Since $q(x)$ is symmetric and uni-modal, we have $q(2\mu_1-x)\leq q(x),\forall x\in[\mu_1,a]$.
    \begin{align*}
        \int_{2\mu_1-a}^a (x-\mu_1)q(x)\,dx
        =& \int_{2\mu_1-a}^\mu (x-\mu_1)q(x)\,dx + \int_{\mu_1}^{a}(x-\mu_1)q(x)\,dx\\
        =&\int_{\mu_1}^a(\mu_1-x)q(2\mu_1-x)\,dx + \int_{\mu_1}^{a}(x-\mu_1)q(x)\,dx\\
        = &\int_{\mu_1}^a (x-\mu_1)(q(x)-q(2\mu_1-x))\,dx\\
        \geq& 0
    \end{align*}
    And then, we have
    \begin{align*}
        \int_{(-\infty,2\mu_1-b]\cup[b-2\mu_1,\infty)}(x-\mu_1)q(x)\,dx
        = -\mu_1 \int_{(-\infty,2\mu_1-b]\cup[b-2\mu_1,\infty)}q(x)\,dx
        =-2\mu_1 \int_{b-2\mu_1}^\infty q(x)\,dx
    \end{align*}
    Since $\mu_1$ is the mean of the distribution $\tilde{q}$. We have
    \[
    \int_{\R\backslash[a,b]}(x-\mu_1)q(x)\,dx=0
    \]
    Then we know
    \begin{align*}
        \int_{2\mu_1-b}^{2\mu_1-a}(x-\mu_1)q(x)\,dx + \int_{b}^{b-2\mu_1}(x-\mu_1)q(x)\,dx \leq 2\mu_1\int_{b-2\mu_1}^\infty q(x)\,dx
    \end{align*}
    Then we have
    \begin{align*}
    0\geq & \int_{2\mu_1-b}^{2\mu_1-a}(x-\mu_1)q(x)\,dx + \int_{b}^{b-2\mu_1}(x-\mu_1)q(x)\,dx 
    -2\mu_1\int_{b-2\mu_1}^\infty q(x)\,dx\\
        =& -\int_{a-2\mu_1}^{b-2\mu_1}xq(x)\,dx - \mu_1 \int_{a-2\mu_1}^{b-2\mu_1}q(x)\,dx + \int_{b}^{b-2\mu_1}xq(x)\,dx - \mu_1\int_b^{b-2\mu_1}q(x)\,dx-2\mu_1\int_{b-2\mu_1}^\infty q(x)\,dx\\
        =&-\int_{a-2\mu_1}^{b-2\mu_1}xq(x)\,dx 
         - \mu_1 \int_{a-2\mu_1}^{\infty}q(x)\,dx
         +\int_{b}^{b-2\mu_1}xq(x)\,dx
         - \mu_1\int_b^{\infty}q(x)\,dx
    \end{align*}
    This derives that
    \begin{align*}
        \int_{a-2\mu_1}^{b}xq(x)\,dx
        \geq& -\mu_1 \int_{a-2\mu_1}^{\infty}q(x)\,dx
        -\mu_1\int_b^{\infty}q(x)\,dx\\
         \geq & -\mu_1 \int_b^{\infty}q(x)\,dx
    \end{align*}
    
\end{proof}

\begin{lemma}[Second Derivative of $F(0)$ when $a'\geq 0$]\label{lemma_cm_F_2derivative_H}
    If $a' \geq 0$, we define
    the following functions,
    \[
    r(x) = 
    \begin{cases}
    \hat{q}(x)-\hat{q}(-x) & x\in[0,a']\cup[b',\infty) \\
    -\hat{q}(-x) & x\in (a',b')
    \end{cases},
    H(\alpha)=\int_0^\infty x(x^2-{a'}^2)(x^2-{b'}^2)r(x)e^{\alpha x^2}\,dx.
    \]
    Then we have
    \[
    F''(\alpha) =  H(\alpha) + ({a'}^2+{b'}^2) F'(\alpha)+{a'}^2{b'}^2 F(\alpha).
    \]
    % \[
    % F''(0) = H(0) + ({a'}^2+{b'}^2) F'(0)
    % \]
\end{lemma}
\begin{proof}
    We rearrange $F(\alpha)$ by combining terms with same $e^{\alpha x^2}$ as in Figure~\ref{fig_lemma_first_moment.}, and get
    \[
    F(\alpha) = \int_{0}^{a'}x(\hat{q}(x) - \hat{q}(-x))e^{\alpha x^2}\,dx
    -\int_{a'}^{b'}x\hat{q}(-x) e^{\alpha x^2}\,dx
    + \int_{b'}^\infty x (\hat{q}(x) - \hat{q}(-x))e^{\alpha x^2}\,dx
    \]
    By the definition of $r(x)$, we naturally have
    \[
    F(\alpha) = \int_0^\infty xr(x) e^{\alpha x^2}\,dx
    \]
    Then we can calculate its first and second derivative as follows.
    \begin{align*}
        F'(\alpha) = \int_0^\infty x^3 r(x)e^{\alpha x^2}\,dx,\quad
        F''(\alpha) = \int_0^\infty x^5 r(x)e^{\alpha x^2}\,dx
    \end{align*}
    By the definition of $H(\alpha)$, we have
    \begin{align*}
        H(\alpha)
        =& \int_0^\infty x^5 r(x) e^{\alpha x^2}\,dx
        - ({a'}^2+{b'}^2) \int_0^\infty x^3 r(x) e^{\alpha x^2}\,dx
         + {a'}^2{b'}^2 \int_0^\infty xr(x)e^{\alpha x^2}\,dx\\
        =& F''(\alpha) - ({a'}^2+{b'}^2) F'(\alpha) + {a'}^2{b'}^2 F(\alpha)
    \end{align*}
    
\end{proof}

\begin{lemma}[First and Second Derivatives of $F(0)$ when $a'\geq 0$]\label{lemma_cm_a_pos}
    If $|a+b|\geq \eps^s$ for $s\geq2$ and $a'\geq 0$, then we  have either $F'(0) < -\frac{C_2\eps^{6s+6.5}}{\log^4{(1/\eps)}}$ or $F''(0) >\frac{ C_3 \eps^{6s+4.5}}{\log^4(1/\eps)}$ for constants $C_2,C_3>0$.
\end{lemma}
\begin{proof}
    We prove the lemma by showing that $H(0) > \frac{C_1 \eps^{6s+4.5}}{\log^5 (1/\eps)} $. We calculate $H(0)$ as follows.
    \begin{align*}
        H(0) =& \int_0^\infty x(x^2-{a'}^2)(x^2-{b'}^2)r(x)\,dx\\
        \geq & \int_{a'}^{b'} x(x^2-{a'}^2)({b'}^2-x^2) \hat{q}(-x)\,dx\\
        =&\int_{a'}^{b'}x(x^2-{a'}^2)({b'}^2-x^2) \frac{\sigma_1 q(x\sigma_1-\mu_1)}{1-\int_a^b q(x)\,dx}\,dx\\
        =& \frac{\int_{a-2\mu_1}^{b-2\mu_1} 
        (x+\mu_1)\left((x+\mu_1)^2-(a-\mu_1)^2 \right) \left( (b-\mu_1)^2-(x+\mu_1)^2 \right) q(x)\,dx}{\sigma_1^3(1-\int_a^b q(x)\,dx)}\\
        = & \frac{\int_{a-2\mu_1}^{b-2\mu_1} (x+\mu_1) (x+a) (x-a+2\mu_1)(b+x)(b-x-2\mu_1) q(x)\,dx }
        {\sigma_1^3(1-\int_a^b q(x)\,dx)}
    \end{align*}
    Denote $\rho:=\int_{a-2\mu_1}^{b} q(x)\,dx$, by Lemma~\ref{lemma_a-2mu_b-2mu} and the bound of $b$, we know 
    \[
    \rho =\int_{a-2\mu_1}^b q(x)\,dx
    \geq \frac{\eps^{s}/2e \cdot \eps}{1+\ln(1/\eps)}
    \geq  \frac{\eps^{s+1}}{6\ln(1/\eps)}
    \]
    Choose $t_1<t_2$ such that $\int_{a-2\mu_1}^{t_1} q(x)\,dx = \int_{t_1}^{t_2}q(x)\,dx = \int_{t_2}^{b}q(x)\,dx$. Since $q(x)$ is upper bounded by $1$ by Lemma~\ref{lemma_logconcave_density_upper_bound}, we have 
    \[
    t_1 - (a-2\mu_1) \geq \frac{\rho}{3}, b - t_2 \geq \frac{\rho}{3}.
    \]
    Using this, we can bound 
    \begin{align*}
        &\int_{t_1}^{t_2}  (x+\mu_1) (x+a) (x-a+2\mu_1)(b+x)(b-x-2\mu_1) q(x)\,dx\\
        \geq & 
        |\mu_1|^2 \frac{\rho}{3}\frac{\rho}{3}\frac{\rho}{3}\frac{\rho}{3}
        \geq \frac{C_1\eps^{6s+4}}{\ln^4(1/\eps)} \text{  for some const }C_1>0
    \end{align*}
    We have shown that $\sigma_1 \leq 1/\sqrt{1-\int_a^b q(x)\,dx}$. Combining all results, we can compute $H(0)$ as
    \begin{align*}
        H(0) > \sqrt{1-\int_a^b q(x)\,dx} \frac{C_1\eps^{6s+4}}{\log^4(1/\eps)} > \frac{C_1\eps^{6s+4.5}}{\log^4(1/\eps)}
    \end{align*}
    Given $F(0)=0$,
    by Lemma~\ref{lemma_cm_F_2derivative_H}, 
    \[
    F''(0) = H(0) + ({a'}^2+{b'}^2)F'(0)
    \]
    Since the distribution is symmetric, we know
    \[
    |\mu_1| = \frac{|\int_{|a|}^b xq(x)\,dx|}{1-\int_a^b q(x)\,dx} 
    \leq \frac{\E\limits_{x\sim q} |x|}{\eps}
    \leq \frac{\sqrt{\E\limits_{x\sim q}x^2}}{\eps} = \frac{1}{\eps}
    \]
    So we know $0\leq a'<b' =b-2\mu_1 < 2.5/\eps$, thus $({a'}^2+{b'}^2)^2 < 25/(2\eps^2)$. Thus we have either $F'(0) < -\frac{C_2\eps^{6s+6.5}}{\log^4{(1/\eps)}}$ or $F''(0) >\frac{ C_3 \eps^{6s+4.5}}{\log^4(1/\eps)}$ for constants $C_2,C_3>0$.
    
\end{proof}

\begin{lemma}[Upper Bound of $F$'s derivatives]\label{lemma_mean_upper_derivatives}
For $\alpha \leq 0$, the derivatives of $F(\alpha)$ are bounded as 
\[
|F'(\alpha)| \leq C_4/\eps^{3}, |F''(\alpha)| \leq C_5/\eps^{5}, |F'''(\alpha)| \leq C_6 /\eps^{7} \quad \text{ for constants }C_4,C_5,C_6>0.
\]   
\end{lemma}
\begin{proof}
    Define $M_k=\E\limits_{x\sim q}|x|^k$. By Cauchy-Schwarz Inequality, $M_1\leq \sqrt{M_2M_0} = 1$. By Lemma~\ref{lemma_logconcave_moments}, 
    \[
    M_k \leq (2k)^k (M_1)^k \leq (2k)^k
    \]
    Also we have proved that $|\mu_1| \leq  1/\eps$.
    By definition of $F(\alpha)$, we calculate its first derivative as follows.
    \begin{align*}
        |F'(\alpha)| 
        \leq&
        \int_{-\infty}^\infty q(x)|x-\mu_1|^3e^{\alpha (x-\mu)^2\,dx}\\
        \leq & \int_{-\infty}^\infty q(x) |x-\mu_1|^3\,dx\\
        \leq & \int_{-\infty}^\infty q(x)(|x|^3-3\mu_1 x^2 + 3\mu_1^2|x|-\mu_1^3)\,dx\\
        =& M_3 - 3\mu_1 M_2 + 3\mu_1^2 M_1-\mu_1^3\\
        \leq & C_4/\eps^{3} \quad \text{  for some constant }C_4>0
    \end{align*}
    Similarly, we calculate its second and third derivatives as follows.
    \begin{align*}
        |F''(\alpha)|
        \leq & \int_{-\infty}^\infty q(x)|x-\mu_1|^5\,dx\\
        \leq& M_5 - 5\mu_1 M_4 + 10\mu_1^2M_3-10\mu_1^3 M_2+5\mu_1^4 M_1 - \mu^{5}\\
        \leq & C_5/\eps^{5}\quad \text{  for some constant }C_5>0\\
        |F'''(\alpha)|
        \leq & \int_{-\infty}^\infty q(x) |x-\mu_1|^7\,dx\\
        =& M_7-7\mu_1 M_6+21\mu_1^2 M_5-35\mu_1^3 M_4 + 35\mu_1^4 M_3
        -21\mu_1^5M_2 + 7\mu_1^6M_1 - \mu_1^7\\
        \leq & C_6/\eps^{7}\quad \text{  for some constant }C_6>0
    \end{align*}
\end{proof}

Now we are ready to prove Lemma~\ref{lemma_reweighted_mean}.
\reweightedmean*
\begin{proof}[Proof of Lemma~\ref{lemma_reweighted_mean}]

    For any $2\leq k\leq d$, for any $\alpha$, by symmetry of $q$,
    the contrastive mean is
    \[
    \E\limits_{x\sim P}e^{\alpha \|x\|^2}x_k = \E\limits_{x\sim q}e^{\alpha x^2}x \cdot \E_{x\sim \hat{q}}e^{\alpha x^2} \cdot \left(
    \E_{x\sim q}e^{\alpha x^2} 
    \right)^{d-2}
     = 0
    \]
    Next we will consider $\E\limits_{x\sim P}e^{\alpha \|x\|^2}x_1$.
    For any $x\geq 0$, we have 
    \[
    \hat{q}(x)  
    =\frac{\sigma_1 q(x\sigma_1+\mu_1) }{1-\int_a^b q(x)\,dx}
    \geq \frac{\sigma_1 q(-x\sigma_1+\mu_1)}{1-\int_a^b q(x)\,dx}
    =\hat{q}(-x)
    \]
    Since $P$ is a product distribution, we have
    \[
    \E\limits_{x\sim P}e^{\alpha \|x\|^2}x_1 = \E\limits_{x_1\sim \hat{q}}e^{\alpha x_1^2}x_1 \cdot \prod_{i=2}^d \E\limits_{x_i\sim q} e^{\alpha x_i^2}
    = F(\alpha)
    \left(
     \E\limits_{x\sim q}e^{\alpha x^2}
    \right)^{d-1}
    \]

    We will consider two cases depending on whether $a'\geq 0$. See Figure~\ref{fig_lemma_first_moment.}.
    
    Firstly, if $a'\leq 0$. We use $\alpha_2$ in this case.
    By Lemma~\ref{lemma_cm_a_neg}, $F'(0) > \eps^{18.5}/2$. By Taylor expansion, there exists $\alpha_2 < \eta_0<0$ such that
    \[
    F(\alpha_2) = F(0) +\alpha_2 F'(0) + \frac{\alpha_2^2}{2}F''(\eta_0).
    \]
    By Lemma~\ref{lemma_mean_upper_derivatives}, we know $F''(\eta_0)\leq C_5/\eps^{5}$. Since $F(0)=0$, we know for $\alpha_2 = -c_2\eps^{42}/d$,
    \begin{align*}
        F(\alpha_2) =& \alpha_2F'(0) + \frac{\alpha_2^2}{2}F''(\eta_0)\\
        \leq & \frac{c_2 \eps^{42}}{d} \left(- \frac{\eps^{18.5}}{2} 
        + \frac{c_2\eps^{42}}{2d} \frac{C_5}{\eps^{5}}
        \right)\\
        <& - \frac{C_7 \eps ^{61}}{d} \text{ for some constant }C_7>0
    \end{align*}
    Then we consider the case when $a'>0$. By Lemma~\ref{lemma_cm_a_pos}, we have either $F'(0)<-\frac{C_2\eps^{36.5}}{\log^4(1/\eps)}$ or $F''(0) > \frac{C_3 \eps^{34.5}}{\log^4(1/\eps)}$. Here we consider three cases with respect to $F'(0)$.
    
    \textbf{Case 1}: $F'(0)<-\frac{C_2\eps^{36.5}}{\ln^4(1/\eps)}$. We use $\alpha_2$ in this case.
    By Taylor expansion, there exists $\eta_1$ such that $\alpha_2<\eta_1<0$ and
    \[
    F(\alpha_2) = F(0) + \alpha_2F'(0)+\frac{\alpha_2^2}{2}F''(\eta_1)
    \]
    By Lemma~\ref{lemma_mean_upper_derivatives}, we know $|F''(\eta_1)|\leq C_5/\eps^{5}$. By choosing $\alpha_2=-c_2\eps^{42}/d$,
    \begin{align*}
        F(\alpha_1) =& \alpha_2 F'(0) + \frac{\alpha_2^2}{2}F''(\eta_1)\\
        \geq &
        \frac{c_2\eps^{42}}{d}
        \left(
        \frac{C_2\eps^{36.5}}{\ln^4(1/\eps)}- \frac{c_2\eps^{42}}{2d} \frac{C_5}{\eps^{5}}
        \right)
        \\
        >& \frac{C_8 \eps^{79} }{d} \text{  for some constant }C_8>0
    \end{align*}

     \textbf{Case 2}: $- \frac{C_2\eps^{36.5}}{\ln^4(1/\eps)} \leq F'(0) \leq \frac{c_s \eps^{77}}{d}$ with some constant $c_s >0$. We use $\alpha_2$ in this case.
     By Lemma~\ref{lemma_cm_a_pos}, we know $F''(0)> C_3\eps^{34.5}/\ln^4(1/\eps)$. Then there exists $\eta_2$ satisfying $\alpha_2<\eta_2<0$ and
     \[
     F(\alpha_2)=F(0)+\alpha_2F'(0)+\frac{\alpha^2}{2}F''(0)+ \frac{\alpha_2^3}{6}F'''(\eta_2).
     \]
     By Lemma~\ref{lemma_mean_upper_derivatives}, we know $|F'''(\eta_2)|\leq C_6/\eps^{7}$.
     Thus by choosing $\alpha_2 = -c_2\eps^{42}/d$,
     \begin{align*}
         F(\alpha_2) =& \alpha_2F'(0) + \frac{\alpha_2^2}{2}F''(0) + \frac{\alpha_2^3}{6}F'''(\eta_2)\\
         >& \frac{c_2\eps^{42}}{d}\left(-\frac{c_s\eps^{77}}{d} +\frac{c_2 \eps^{42}}{2d} \frac{C_3 \eps^{34.5}}{\ln^4(1/\eps)} - \frac{c_2^2\eps^{84}}{6d^2}\frac{C_6}{\eps^{7}} \right)\\
         \geq& \frac{C_9 \eps^{119}}{d^2}\text{  for some constant }C_9>0
     \end{align*}

    \textbf{Case 3}: $F'(0) > \frac{c_s \eps^{77}}{d}$. We use $\alpha_1$ in this case.
    Then there exists $\eta_3$ satisfying $\alpha_1<\eta_3<0$ and
    \[
    F(\alpha_1)=F(0)+\alpha_1F'(0) + \frac{\alpha_1^2}{2}F''(\eta_3).
    \]
    By Lemma~\ref{lemma_mean_upper_derivatives}, we know $|F''(\eta_3)| \leq C_5/\eps^{5}$. For $\alpha_1=-c_1\eps^{82}/d$,  we have
    \begin{align*}
        F(\alpha_1) =& \alpha_1 F'(0) + \frac{\alpha_1^2}{2}F''(\eta_3)\\
        \leq& \frac{c_1 \eps^{82}}{d} 
        \left(
        - \frac{c_s \eps^{77}}{d} + \frac{c_1\eps^{82}}{2d} \frac{C_5}{\eps^{5}}
        \right)\\
        \leq & -\frac{C_{10} \eps^{159}}{d^2} \text{  for some constant }C_{10}>0
    \end{align*}
    Then we know for all cases, there exists a constant $C' = \min(C_7,C_8,C_9,C_{10})$ such that
    \[
    \max(|F(\alpha_1)|, |F(\alpha_2)|) \geq \frac{C'\eps^{159}}{d^2}
    \]
    Finally we will lower bound $\left(\E\limits_{x\sim q}e^{\alpha x^2}\right)^{d-1}$ as follows.
    \begin{align*}
        \left(\E\limits_{x\sim q}e^{\alpha x^2}\right)^{d-1}
        \geq \left(\E\limits_{x\sim q} (1+\alpha x^2)\right)^{d-1}
        = \left( 1 +\alpha \right)^{d-1}
        \geq \left( 1 -\frac{1}{d} \right)^{d-1} \geq 1/e
    \end{align*}
    Let $C = C'/e$, and we will get
    \[
    \max(|u^\top \mu_{\alpha_1}|,|u^\top \mu_{\alpha_2}|) > \frac{C\eps^{159}}{d^2}.
    \]
\end{proof}

\subsubsection{Quantitative Bounds for Contrastive Covariance: Symmetric Case}
\label{sec:proof_convariance_sym}

Before addressing Lemma~\ref{lemma_reweighted_cov} which is applicable in the scenario where $|a+b|<\eps^{5}$, we first demonstrate that contrastive covariance works for the case where the removed band $[a,b]$ is symmetric around the origin.
That is, $a=-b$. 
In such cases, we aim to establish that there's a noticeable difference between the top two eigenvalues ($\lambda_1$ and $\lambda_2$) of the contrastive covariance matrix $\tilde{\Sigma}$, stated in  Lemma~\ref{lemma_reweighted_cov_sym}. We will then 
extend the lemma to the near-symmetric scenario in Section~\ref{sec:proof_covariance_almost_sym}.

\begin{lemma}[Quantitative Spectral Gap of Contrastive Covariance - Symmetric Case]\label{lemma_reweighted_cov_sym}
    Suppose $a+b=0$. Choose $\alpha_3 = -c_3\eps^2$ for some constant $c_3>0$. 
    Then, for an absolute constant $C$, the top two eigenvalues $\lambda_1 \ge \lambda_2$ of the corresponding re-weighted covariance of $P$ satisfy
    \[
    \lambda_1 - \lambda_2 \geq C\eps^3 \lambda_1.
    \]
\end{lemma}
We recall the definition of moment ratio as in Definition~\ref{def:variance_ratio}.
\[
\mr_q(t) = \frac{\var_{q_t}(X^2)}{(\E_{q_a}X^2)^2} = \frac{M_0(t)M_4(t)}{M_2^2(t)}-1
,\text{ where }M_k(t) = \int_t^\infty x^kq(x)\,dx.
\]
For simplicity, in the remaining section, we'll use $\mr(t)$ as a shorthand notation for this moment ratio.
Just as in the proof of the qualitative bound in Section~\ref{section:proof_qualitative}, we consider
the difference between the first and second eigenvalues (the spectral gap) of the re-weighted covariance matrix, denoted as $\lambda_1-\lambda_2$, as a function of $\alpha$. Then the function is valued zero when $\alpha=0$, and its derivative at $\alpha=0$ is
is proportional to the difference in the moment ratio of the distribution $q$ at $b$ and $0$, denoted as $\mr(b)-\mr(0)$.
To prove the quantitative result Lemma~\ref{lemma_reweighted_cov_sym}, our proof strategy involves several steps. 
\begin{itemize}
    \item Proving monotonicity of Moment Ratio $\mr(t)$. This property is stated in Lemma~\ref{lemma_variance_ratio} and its proof is presented in Section~\ref{section:proof_qualitative}. The proof involves reducing the case of general logconcave distributions to that of exponential distributions.
    \item Establishing a positive gap $\mr(0)-\mr(b)$ for small $b$. With Lemma~\ref{lemma_g_gap_small_eps}, we focus on illustrating that for values of $b$ which are relatively small (less than a certain constant), there is a guaranteed positive gap $\mr(0)-\mr(b)$. 
    \item Generalizing to any $b$ satisfying $\int_{-b}^b q(x)\,dx \geq \eps$ in Lemma~\ref{lemma_secondmoment_logconcave_1dim_gap}. The lemma will directly imply the quantitative result lemma for the symmetric case (Lemma~\ref{lemma_reweighted_cov_sym}) using Taylor expansion.
    
\end{itemize}

Having demonstrated the monotonicity of the moment ratio in Section~\ref{section:proof_qualitative}, we now begin by illustrating the positive gap, denoted as $\mr(0)-\mr(b)$, for small $b$
in Lemma~\ref{lemma_g_gap_small_eps}.
The proof relies on the properties of moments derived by the unimodality of the distribution (Lemma~\ref{lemma_uniform_small_secondmoment}, Lemma~\ref{lemma_uniform_0t_secondmoment}).

    \begin{lemma}[Gap for small $t$]\label{lemma_g_gap_small_eps} 
        For $t>0$ such that $\int_0^t q(x)\,dx \leq 0.4$, we have the following gap
    \[
    \mr(0)-\mr(t) \geq \frac{\int_0^t q(x)\,dx}{100}.
    \]
    \end{lemma}
    \begin{proof}
        Denote $\int_0^t q(x)\,dx = \nu, \int_0^t x^2q(x)\,dx = \omega$. Then we have
    \begin{align*}
        \mr(0)-\mr(t)
        =& \frac{M_0(0)M_4(0)}{M_2^2(0)} - \frac{M_0(t)M_4(t)}{M_2^2(t)}\\
        =& 2\int_0^t x^4q(x)\,dx + \frac{M_4(t)}{2(\omega+M_2(t))^2} - \frac{(\frac{1}{2}-\nu)
        M_4(t)}{M_2^2(t)}\\
        =& 2\int_0^t x^4q(x)\,dx + \frac{M_4(t)}{2M_2^2(t)} \left(\frac{1}{\left(
        1+\frac{\omega}{M_2(t)}
        \right)^2} -1+2\nu\right)
    \end{align*}
    Fix $q(t)$ and $\nu$, we apply Lemma~\ref{lemma_uniform_small_secondmoment} and have
    \[
    M_2(t)\geq \left(\frac{1}{2}-\nu\right)t^2 
    \left(
    1+ \frac{\frac{1}{2}-\nu}{q(t)t}
    \right)
    \]
    Since $q(x)$ is monotonically decreasing, for any $x\leq t,q(x)\geq q(t)$. So we have the constraint that 
    \[
    \nu = \int_0^tq(x)\,dx \geq tq(t)
    \]
    Plug into the previous inequality and we get
    \[
    M_2(t)
    \geq \left(\frac{1}{2}-\nu\right)t^2 
    \left(
    1+ \frac{\frac{1}{2}-\nu}{q(t)t}
    \right)
    \geq \left(\frac{1}{2}-\nu\right)t^2 
    \left(
    1+ \frac{\frac{1}{2}-\nu}{\nu}
    \right)
    = \left(\frac{1}{2}-\nu\right)\frac{t^2}{2\nu}
    \]
    In addition, by fixing $v$ and $t$, we apply
    Lemma~\ref{lemma_uniform_0t_secondmoment} and get 
    \[
    \omega \leq \frac{\nu t^2}{3}
    \]
     So we know
    \[
    \frac{\omega}{M_2(t)} \leq \frac{\frac{\nu t^2}{3}}
    {\left(\frac{1}{2}-\nu\right)\frac{t^2}{2\nu}} = \frac{2\nu^2}{3\left(\frac{1}{2}-\nu\right)}
    \]
    By calculation, we will get
    \begin{align*}
        \frac{1}{\left(
        1+\frac{\omega}{M_2(t)}
        \right)^2} -1+2\nu
        \geq &\frac{9\left(\frac{1}{2}-\nu\right)^2
        -2\left(\frac{1}{2}-\nu\right)\left(
        \frac{3}{2}-3\nu+2\nu^2
        \right)^2
        }
        {\left(
        \frac{3}{2}-3\nu+2\nu^2
        \right)^2}\\
        \geq & 
        \frac{4}{9}\left(\frac{1}{2}-\nu\right)
        \left(
        9\left(\frac{1}{2}-\nu\right) - 2\left(
        \frac{3}{2}-3\nu+2\nu^2
        \right)^2\right)\\
        = & \frac{4}{9}\nu\left(\frac{1}{2}-\nu\right)
        \left(
        -8\nu^3+24\nu^2-30\nu+9
        \right)
    \end{align*}
    Let $T(\nu) = -8\nu^3+24\nu^2-30\nu+9$, then we know its derivative is
    \[
    T'(\nu) = -24\nu^2+48\nu-30 = -24(\nu-1)^2-6 < 0
    \]
    So $T(\nu)$ is monotonically decreasing. Since $\nu < 0.4$, we know
    \[
    T(\nu) \geq T(0.4) > 0.3, \frac{1}{2}-\nu > 0.1
    \]
    Plugging in and we will get
    \[
    \frac{1}{\left(
        1+\frac{\omega}{M_2(t)}
        \right)^2} -1+2\nu
        \geq \frac{4}{9}\nu \cdot 0.1\cdot 0.3 > 0.01\nu
    \]
    Finally by Cauchy-Schwarz Inequality, we have $M_0(t)M_4(t) \geq M_2^2(t)$.
    \begin{align*}
        \mr(0)-\mr(t)
        \geq & \frac{M_4(t)}{2M_2^2(t)}0.01\nu
        \geq \frac{\nu}{200M_0(t)} 
        \geq \frac{\nu}{100}
    \end{align*}
    \end{proof}

    \begin{lemma}\label{lemma_uniform_small_secondmoment}
    Let $t,r\geq 0,0<s\leq 1$. Define $\mathcal{P}=\{p(x):[t,\infty)\rightarrow[0,1) \text{, logconcave },p'(x)\leq 0,\int_t^\infty p(x)\,dx = s, p(t) = r\}$.
    Then we have
    \[
    \min_{p\in\mathcal{P} }\int_t^\infty x^2p(x)\,dx \geq st^2\left(1+\frac{s}{rt}\right).
    \]
    \end{lemma}
    \begin{proof}
    For any $p\in\mathcal{P}$, we denote $M_2(p) = \int_t^\infty x^2p(x)\,dx$. Define 
    \[
    u(x) = r\cdot\mathds{1}_{x\in [t,t+s/r]} 
    \]
    Clearly $u\in \mathcal{P}$. We will show that $u(x) = \argmin_{p\in P}M_2(p)$.
    For any $p\in\mathcal{P}$, we have $\int_t^\infty u_(x)\,dx = \int_t^\infty p(x)\,dx$, $u(t)=p(t)$ and $p'(x)\leq 0$. So we know the graph of $u$ and $p$ intersects at points $t$ and $t+s/r$, where $u_(x)\geq p(x)$ in the interval $[t,t+s/r]$ and $u(x)< p(x)$ outside the interval. So we know
    \[
    \int_{t}^{t+s/r}u(x)- p(x) \,dx = \int_{t+s/r}^\infty p(x) - u(x)\,dx
    \]
    Since for any $x\in[t,t+s/r]$, any $y\in[t+s/r,\infty)$, we have $x\leq y$. So we have
    \[
    M_2(u)-M_2(p) = \int_{t}^{t+s/r} (u(x)- p(x)) x^2\,dx - \int_{t+s/r}^\infty (p(x) - u(x))x^2\,dx \leq 0
    \]
    This shows that 
    \[
    \min_{p\in\mathcal{P} }M_2(p) =  M_2(u).
    \]
    By calculating  $M_2(u)$, we have
    \begin{align*}
    \min_{p\in\mathcal{P} }M_2(p)=
        M_2(u) 
        = \int_t^{t+s/r} rx^2\,dx
        = r(t^2\frac{s}{r}+t\frac{s^2}{r^2}+\frac{s^3}{3r^3})
        \geq st^2\left(1+\frac{s}{rt}\right)
    \end{align*}

\end{proof}

    \begin{lemma}\label{lemma_uniform_0t_secondmoment}
    Let $t\geq 0,0<s\leq 1$. Define $\mathcal{P}=\{p(x):[0,t]\rightarrow[0,1)\text{, logconcave, }p'(x)\leq 0,\int_0^tp(x)\,dx=s\}$. Then we have
    \[
    \max_{p\in\mathcal{P}} \int_0^t x^2p(x)\,dx = \frac{st^2}{3}.
    \]
\end{lemma}
\begin{proof}
    For any $p\in\mathcal{P}$, we denote $M_2(p)=\int_0^t x^2p(x)\,dx$. Define $u(x) = \frac{s}{t} \cdot \mathds{1}_{x\in[0,t]}$. Clearly, $u\in\mathcal{P}$. Then for any $p\in\mathcal{P}$, because it is monotonically decreasing and $\int_0^t p(x)\,dx = \int_0^t u(x)\,dx$, the graphs of $p$ and $u$ intersect at point $l\in[0,t]$. Also $p(x)\geq u(x)$ for $x\in[0,l]$ and $p(x)\leq u(x)$ for $x\in[l,t]$. Since for any $x\in[0,l]$ and any $y\in[l,t]$, we have $x\leq y$. So we know
    \[
    M_2(p)-M_2(u) = \int_0^l x^2(p(x)-u(x))\,dx - \int_l^t x^2(u(x)-p(x))\leq 0
    \]
    By calculating $M_2(u)$, we have
    \[
    \max_{p\in \mathcal{P}}M_2(p)
    = M_2(u) = \int_0^t \frac{s}{t}x^2\,dx
    = \frac{s}{t}\frac{1}{3}t^3 = \frac{st^2}{3}
    \]
    
    \end{proof}

    Using Lemma~\ref{lemma_variance_ratio} and Lemma~\ref{lemma_g_gap_small_eps}, we will show the positive gap $\mr(0)-\mr(b)$ for any $b$ satisfying $\int_{-b}^b q(x)\,dx \geq \eps$ as in Lemma~\ref{lemma_secondmoment_logconcave_1dim_gap}.

    \begin{lemma}[Gap for Log-concave Distribution]\label{lemma_secondmoment_logconcave_1dim_gap}
    Let $0<\eps<0.1$, let $b>0$ satisfying $\int_{0}^b q(x)\,dx \geq \eps/2$.
    Then we have for $\mr(0)-\mr(b) \geq \eps/200$.
    \end{lemma}
    \begin{proof}
        Let $b_0$ such that $\int_0^{b_0}q(x)\,dx = \eps/2$. By Lemma~\ref{lemma_g_gap_small_eps}, we know
        \[
        \mr(0)-\mr(b_0) \geq \frac{\int_0^{b_0}q(x)\,dx}{100} \geq \frac{\eps}{200}
        \]
         By Lemma~\ref{lemma_variance_ratio}, we know $\mr'(t) \leq 0,\forall t \geq 0$. So for any $b>0$ such that $\int_0^b q(x)\,dx \geq \eps/2$, $\mr(b)\leq \mr(a_0)$. Therefore,
         \[
         \mr(0)-\mr(b) \geq \mr(0)-\mr(b_0) \geq \frac{\eps}{200}
         \]
    \end{proof}

   Before moving on to the proof of the quantitative lemma, we will first present a helper lemma that can be directly applied.
    \begin{lemma}\label{lemma_cov_sym_S_gap}
        Define
        \begin{align}
        \label{eq:S_alpha_def}
            S(\alpha):= \mathop{\E}\limits_{x\sim \hat{q}}e^{\alpha x^2}x^2 \mathop{\E}\limits_{x\sim q}e^{\alpha x^2} - \mathop{\E}\limits_{x\sim q}e^{\alpha x^2}x^2 \mathop{\E}\limits_{x\sim \hat{q}}e^{\alpha x^2}
        \end{align} 
        For a given $\alpha = -c\eps^s$ with $s \geq 2$ and a certain positive constant $c$, it can be established that 
        \[
        S(\alpha) > C\eps^{s+1},\quad \text{ where }C\text{ is a positive constant.}
        \]
    \end{lemma}
    \begin{proof}
    We show the lower bound of $S(\alpha)$ using Taylor expansion.
    
        Firstly, since $q$ and $\hat{q}$ are both isotropic, 
        \begin{align*}
            S(0)= \mathop{\E}\limits_{x\sim \hat{q}}x^2 -\mathop{\E}\limits_{x\sim q}x^2   =0
        \end{align*}

        Secondly, we will lower bound $|S'(0)|$ using the monotonicity of moment ratio.
        The variance $\sigma_1^2$ of $q$ restricted to $\R\backslash [-b,b]$ is
    \begin{align*}
    \sigma_1^2
    = \frac{\int_b^\infty x^2p(x)\,dx}{\int_b^\infty p(x)\,dx}
    =  \frac{M_2(b)}{M_0(b)}
    \end{align*}
    Since $\hat{q}$ is isotropic,
    the density on the support $\R\backslash [-b/\sigma_1,b/\sigma_1]$ is
    \[
    \mathbb{P}_{\hat{q}}(x) 
    %= \sigma_1 \mathbb{P}_{\tilde{P}_2}(x\sigma_1)
    =\frac{\sigma_1 q(x\sigma_1)}{2\int_b^\infty q(x)\,dx}
    \]
         By calculation, we have
    \begin{align*}
        S(\alpha)
        =& \frac{2\int_{b/\sigma_1}^\infty e^{\alpha x^2}x^2 \sigma_1 p(x\sigma_1)\,dx}{2\int_b^\infty q(x)\,dx}
        \cdot 2\int_0^\infty e^{\alpha x^2}q(x)\,dx
        - \frac{2\int_{b/\sigma_1}^\infty e^{\alpha x^2}\sigma_1 p(x\sigma_1)\,dx}{2\int_b^\infty q(x)\,dx}
        \cdot 2\int_0^\infty e^{\alpha x^2}x^2q(x)\,dx\\
        =& \frac{2\int_b^\infty e^{\alpha y^2/\sigma_1^2 } y^2/\sigma_1^2 q(y)\,dy}{\int_b^\infty q(x)\,dx}\cdot \int_0^\infty e^{\alpha x^2}q(x)\,dx
        - \frac{2\int_b^\infty e^{\alpha y^2/\sigma_1^2}q(y)\,dy}{\int_b^\infty q(x)\,dx}\cdot \int_0^\infty e^{\alpha x^2}x^2q(x)\,dx\\
        =& \frac{\frac{2}{\sigma_1^2} \int_b^\infty e^{\alpha y^2/\sigma_1^2}y^2q(y)\,dy \cdot \int_0^\infty e^{\alpha x^2}q(x)\,dx - 2\int_b^\infty e^{\alpha y^2/\sigma_1^2}q(y)\,dy\cdot \int_0^\infty e^{\alpha x^2}x^2q(x)\,dx}
        {\int_b^\infty q(x)\,dx}
        \end{align*}
        Then we can compute $S'(0)$ as
        \begin{align*}
        S'(0)
            =& \frac{2}
        {\sigma_1^2\int_b^\infty q(x)\,dx}
        \left(
        \frac{1}{\sigma_1^2} \int_b^\infty y^4q(y)\,dy\cdot \int_0^\infty q(x)\,dx+\int_b^\infty y^2q(y)\,dy \cdot\int_0^\infty x^2q(x)\,dx
        \right)\\
        &-\frac{2}{\int_b^\infty q(x)\,dx}
        \left(
        \frac{1}{\sigma_1^2}\int_b^\infty y^2q(y)\,dy \cdot \int_0^\infty x^2q(x)\,dx
        + \int_b^\infty q(y)\,dy \cdot \int_0^\infty x^4q(x)\,dx
        \right)\\
        =& 
        \frac{2M_0(b)M_4(b)M_0(0)}{M_2^2(b)}
        -2M_4(0)
        \\
        =& 
    \frac{M_4(b)M_0(b)}{M_2^2(b)}
    - \frac{M_4(0) M_0(0)}{M_2^2(0)}\\
    = & \mr(b)-\mr(0)
    \end{align*}
    The last step is because the $q$ is isotropic. By Lemma~\ref{lemma_secondmoment_logconcave_1dim_gap}, $\mr(0)-\mr(b) \geq \eps/200$. This indicates that 
    \begin{align*}
        S'(0) \leq -\eps/200
    \end{align*}

    Next, we can upper bound $S''(\alpha)$ for any $\alpha\leq 0$ as 
    \begin{align*}
    S''(\alpha)
    =&\frac{2}{\int_b^\infty p(x)\,dx}
    \Bigg(
    \frac{1}{\sigma_1^6}\int_b^\infty e^{\alpha y^2/\sigma_1^2}y^6q(y)\,dy \cdot \int_0^\infty e^{\alpha x^2}q(x)\,dx\\
    &+ \frac{1}{\sigma_1^4} \int_b^\infty e^{\alpha y^2/\sigma_1^2}y^4q(y)\,dy
    \cdot \int_0^\infty e^{\alpha x^2}x^2q(x)\,dx\\
    &- \frac{1}{\sigma_1^2}\int_b^\infty e^{\alpha y^2/\sigma_1^2}y^2q(y)\,dy
    \cdot \int_0^\infty e^{\alpha x^2} x^4q(x)\,dx\\
    &-\int_b^\infty e^{\alpha y^2/\sigma_1^2}q(y)\,dy\cdot \int_0^\infty e^{\alpha x^2}x^6q(x)\,dx
    \Bigg)\\
    \geq & -\frac{2}{M_0(b)}\left( 
    \frac{M_2(b)M_4(0)}{M_2(0)}
    + M_0(b)M_6(0)
    \right)\\
    \geq & - \frac{2}{M_0(b)}\left(M_4(0)+M_0(0)M_6(0)\right)
    % \leq & \frac{2}{M_0(b)}
    % \left(
    % \frac{M_6(b)M_0(0)}{M_2^3(b)} + \frac{M_4(b)M_2(0)}{M_2^2(b)}
    % \right)\\
    % =& \frac{M_6(b)}{M_0(b)M_2^3(b)} + \frac{M_4(b)}{M_0(b)M_2^2(b)}
    \end{align*}
    By Cauchy-Schwarz Inequality, $M_1(0)\leq \sqrt{M_2(0)M_0(0)} = 1/2$. By Lemma~\ref{lemma_logconcave_moments}, 
    \[
    M_k(0) \leq (2k)^k (2M_1(0))^k /2\leq (2k)^k/2
    \]
    Since $M_0(b) \geq \eps$, for some positive constant $c_{sec}$,
    \[
    S''(\alpha) \geq - \frac{c_{sec}}{\eps}
    \]
    % By Lemma~\ref{lemma_logconcave_tail}, we know $b\leq 1+ \ln(1/\eps)$. Then we know $M_0(b) \geq \eps$ and $M_2(b) \geq b^2 \eps \geq (1+\ln(1/\eps))^2\eps$. So we know for some positive constant $c_{sec}$,
    % \[
    % S''(\alpha) 
    % \leq \frac{{12}^6/2}{\eps(1+ \ln(1/\eps))^6 \eps^3} + \frac{8^4/2}{\eps(1+\ln(1/\eps))^4\eps^2} < \frac{c_{sec}}{\eps^3}
    % \]
    By Taylor expansion, we know for $\alpha_3 = -c\eps^s,s\geq 2,c=1/(101c_{sec})$, there exists $\alpha'\in[\alpha,0]$ such that for some constant $C>0$,
    \begin{align*}
        S(\alpha) =& S(0) +\alpha S'(0) + \frac{\alpha^2}{2}S''(\alpha')
        \geq  0 + c\eps^s
        \frac{\eps}{200} - \frac{c^2\eps^{2s}}{2} \frac{c_{sec}}{\eps }
        >C \eps^{s+1}
    \end{align*}
    \end{proof}

    Now we are ready to prove the contrastive covariance lemma (Lemma~\ref{lemma_reweighted_cov_sym}).
    \begin{proof}[Proof of Lemma~\ref{lemma_reweighted_cov_sym}]
    Define
    \begin{align*}
        S(\alpha):= \mathop{\E}\limits_{x\sim \hat{q}}e^{\alpha x^2}x^2 \mathop{\E}\limits_{x\sim q}e^{\alpha x^2} - \mathop{\E}\limits_{x\sim q}e^{\alpha x^2}x^2 \mathop{\E}\limits_{x\sim \hat{q}}e^{\alpha x^2}
    \end{align*}
    Then for $2\leq j\leq d$,
    \begin{align*}
        &\E\limits_{x\sim {P}}e^{\alpha \|x\|^2}x_1^2 -  \E\limits_{x\sim {P}}e^{\alpha \|x\|^2}x_j^2\\
        =& \left(
        \E\limits_{x_1\sim \hat{q}}e^{\alpha x_1^2}x_1^2 \E\limits_{x_2\sim q}e^{\alpha x_2^2}
        - \E\limits_{x_2\sim \hat{q}}e^{\alpha x_2^2}x_1^2 \E\limits_{x_1\sim q}e^{\alpha x_1^2}
        \right)
        \left(\E\limits_{x\sim q} e^{\alpha x^2}\right)^{d-2}\\
        =& S(\alpha ) \frac{\E\limits_{x\sim {P}}e^{\alpha \|x\|^2} x_1^2}
        {\E\limits_{x_1\sim \hat{q}}e^{\alpha x_1^2}x_1^2 \E\limits_{x_2\sim q}e^{\alpha x_2^2}}
    \end{align*}
    Since $\alpha_3<0$, we have
    \[
    \E\limits_{x_1\sim \hat{q}}e^{\alpha_3 x_1^2}x_1^2\leq \E\limits_{x_1\sim \hat{q}} x_1^2=1,\quad
    \E\limits_{x_2\sim q}e^{\alpha_3 x_2^2} \leq 1
    \]
    By Lemma~\ref{lemma_cov_sym_S_gap}, $\alpha_3 = c_3 \eps^2$ implies that
    $S(\alpha_3)>C\eps^3$. So we have
    \[
    \E\limits_{x\sim {P}}e^{\alpha \|x\|^2}x_1^2 -  \E\limits_{x\sim {P}}e^{\alpha \|x\|^2}x_j^2
    \geq C\eps^3 \E\limits_{x\sim {P}}e^{\alpha \|x\|^2}x_1^2
    \]
    Finally we will show that the first eigenvector corresponds to $e_1$.
    For any $v\in\R$, define $\phi(v)$ as
    \[
    \phi(v):=\mathop{\E}\limits_{x\sim {P}} \frac{e^{\alpha_3 \|x\|^2}v^\top xx^\top v}{v^\top v}
    \]
    Then we know for $2\leq j\leq d$, 
    \[
    \phi(e_1)-\phi(e_j) > C\eps^3\phi(e_1)
    \]
    For any vector $v=\sum_{i=1}^d \gamma_ie_i$, we have
    \begin{align*}
    \phi(v) =& \frac{1}{\sum_{i=1}^d \gamma_i^2}\E e^{\alpha_3 \|x\|^2}(\sum_{i=1}^d \gamma_i x_i)^2\\
    =& \frac{1}{\sum_{i=1}^d \gamma_i^2}
    \left(
    \sum_{i=1}^d \gamma_i^2 \phi(e_i) + 2\E e^{\alpha\|x\|^2}\sum_{i\neq j}\gamma_i\gamma_j x_i x_j
    \right)\\
    =& \frac{1}{\sum_{i=1}^d \gamma_i^2}
    \left(
    \sum_{i=1}^d \gamma_i^2 \phi(e_i) 
    + 2\sum_{i\neq j}\gamma_i\gamma_j \E e^{\alpha_3 \sum_{k\neq i,j} x_k^2}
    \E e^{\alpha_3 \left<x,e_i\right>^2}x_1
    \E e^{\alpha \left<x,e_j\right>^2}x_j
    \right)\\
    =& \frac{\sum_{i=1}^d \gamma_i^2 \phi(e_i)}{\sum_{i=1}^d \gamma_i^2}\\
    \leq & \phi(e_1)
    \end{align*}
        This shows that the top eigenvalue of $\tilde{\Sigma}$ is $\lambda_1=\max_v \phi(v)=\phi(e_1)$. In other word, the top eigenvector is $e_1$, which is essentially $u$.
        Similarly the second eigenvalue of $\tilde{\Sigma}$ is $\lambda_2 =\max_{v:v\bot e_1}\phi(v) = \phi(e_j),2\leq j\leq d$. So we get
    \[
    \lambda_1-\lambda_2 > C\eps^3\lambda_1
    \]
    
    \end{proof}

\subsubsection{Quantitative Bounds for Contrastive Covariance: Near-Symmetric Case}
\label{sec:proof_covariance_almost_sym}
We have shown the result for symmetric case $a+b=0$ in Section~\ref{sec:proof_convariance_sym}. Here we will show that we can extend the contrastive covariance lemma (Lemma~\ref{lemma_reweighted_cov_sym}) to the near-symmetric case, where $|a+b|<\eps^{5}$. In this section, we will present the proof of Lemma~\ref{lemma_reweighted_cov}, which addresses the nearly symmetric case quantitatively.
The proof idea is to approximate the re-weighted covariance of the distribution with margin $[a,b]$, by comparing it to the same distribution truncated with the symmetric interval $[-b,b]$. This enables us to generalize the result from the symmetric scenario to the near-symmetric scenario. 

Recall that $\tilde{q}$ is the distribution obtained by restricting $q$ to the set $\R\backslash [a,b]$. We denote $\tilde{r}$ as the distribution that is obtained by restricting $q$ to the set $\R\backslash [-b,b]$, and $\hat{r}$ as the isotropized distribution of $\tilde{r}$. Let $\sigma_2^2$ be the variance of $\tilde{r}$. 

% Since $\hat{r}$ is isotropic, it is supported on $\mathbb{R}\backslash [-b/\sigma_2,b/\sigma_2]$ with density
%  \[
%  \mathbb{P}_{\hat{r}}(x) = 
%  \frac{\sigma_2 q(x\sigma_2)}{2\int_b^\infty q(x)\,dx}.
%  \]

To approximate the characteristics of $\hat{q}$ using those of $\hat{r}$, we undertake the subsequent steps.
\begin{itemize}
    \item Assessing the mean. We illustrate that the mean of $\tilde{q}$ is adequately small in Lemma~\ref{lemma_almost_sym_mu_range}.
    \item Variance approximation. We approximate the variance of $\hat{q}$ by using the variance of $\hat{r}$, as elaborated in Lemma~\ref{lemma_almost_sym_var_bound}.
    \item Re-weighted second moment. We use the re-weighted second moment of $\hat{r}$ to approximate the corresponding moment in $\hat{q}$. The details are provided in Lemma~\ref{lemma_almost_sym_approx_reweighted_second_moment}.
    \item Re-weighted zeroth moment. We use the re-weighted zero moment of $\hat{r}$ to approximate the corresponding moment in $\hat{q}$, which is shown in Lemma~\ref{lemma_almost_sym_approximation_reweighted_Zero_moment}.
\end{itemize}

\begin{lemma}
\label{lemma_almost_sym_mu_range}
    For an integer $s\geq 1$, if $|a+b|<\eps^{s+1}$, then
    \begin{enumerate}
        \item[(1)]the mean of $\tilde{q}$, $|\mu_1|<\eps^{s} \ln(1/\eps)$;
        \item[(2)] a<0.
    \end{enumerate}
\end{lemma}
\begin{proof}
    We first consider the case when $a<0$.
    By Lemma~\ref{lemma_logconcave_mean}, we have
    \[
    |\mu_1|  = \frac{\int_{|a|}^b xq(x)\,dx}{1-\int_a^b q(x)\,dx} 
    \leq \frac{b \int_{|a|}^b q(x)\,dx}{\int_{-\infty}^a q(x)\,dx + \int_b^\infty q(x)\,dx}
    \leq \frac{b(b+a)}{2\eps}
    \]
    By the tail bound of logconcave distributions (Lemma~\ref{lemma_logconcave_tail}), 
    \[
    |b| < 1+\ln \frac{1}{\eps}
    \]
    Then we have
    \[
    |\mu_1| < \frac{(1+\ln \frac{1}{\eps})\eps^{s+1}}{2\eps} < \eps^{s}\ln(\frac{1}{\eps})
    \]
    Next for $a\geq 0$, we can bound $b$ by $\eps^{s+1}/2$ because $|b|\geq |a|$. This implies that $b-a \leq \eps^{s+1}/2$, which leads to a contradiction that $\int_a^b q(x)\,dx \geq \eps$. Therefore, $a$ can only be negative in this scenario.
    
\end{proof}

% \begin{lemma}
% \label{lemma_almost_sym_a_neg}
%     $|\mu_1| \leq \eps^{10}$ implies that $a<0$.
% \end{lemma}
% \begin{proof}
%     Suppose $a \geq 0$. By Lemma~\ref{lemma_logconcave_mean}, we have
%     \[
%     |\mu_1|  = \frac{\int_a^b xq(x)\,dx}{1-\int_a^b q(x)\,dx} >  \frac{1}{e} \int_a^b q(x)\,dx >\frac{\eps}{e}
%     \]
%     This leads to the contradiction.
% \end{proof}

\begin{lemma}
\label{lemma_almost_sym_var_bound}
    For $|\mu_1|\leq r$ with $r<1/6$, we can bound the variance as follows.
    \[
    \frac{\sigma_2^2}{1+2er}
    \leq \sigma_1^2
    \leq \sigma_2^2
    \]
\end{lemma}
\begin{proof}
    By Lemma~\ref{lemma_almost_sym_mu_range}, we know $a<0$. We can calculate the variance as
    \[
    \sigma_2^2
     = \frac{\int_b^\infty x^2q(x)\,dx}{\int_b^\infty q(x)\,dx},\quad
    \sigma_1^2
    = \frac{\int_b^\infty x^2q(x)\,dx
    + \frac{1}{2}\int_{-a}^b x^2 q(x)\,dx}
    {\int_b^\infty q(x)\,dx
    + \frac{1}{2}\int_{-a}^b q(x)\,dx
    }
    \]
    On one hand,
    \begin{align*}
        \frac{\int_{-a}^b x^2q(x)\,dx}{\int_b^\infty x^2q(x)\,dx}
        \leq 
        \frac{b^2\int_{-a}^b q(x)\,dx}
        {b^2 \int_b^\infty q(x)\,dx}
        = \frac{\int_{-a}^b q(x)\,dx}
        {\int_b^\infty q(x)\,dx}
    \end{align*}
    So we have
    \begin{align*}
        \frac{\sigma_1^2}{\sigma_2^2}
        =& 
        \frac{
        1 + \frac{\int_{-a}^b x^2q(x)\,dx}{2\int_b^\infty x^2q(x)\,dx}
        }
        {1 +  \frac{\int_{-a}^b q(x)\,dx}{2\int_b^\infty q(x)
        \,dx}
        }
        \leq 1
    \end{align*}
    On the other hand, since $|\mu_1| \leq r$, by  Lemma~\ref{lemma_logconcave_mean},
    \[
    r\geq
    |\mu_1| = 
    \frac{\int_{-a}^b xq(x)\,dx}{
    \int_{-a}^b q(x)\,dx + 2
    \int_b^\infty q(x)\,dx }
    > 
    \frac{
    \frac{1}{e}\int_{-a}^b q(x)\,dx
    }{
    \int_{-a}^b q(x)\,dx + 
    2\int_b^\infty q(x)\,dx }
    \]
    So we have
    \[
    \frac{\sigma_2^2}{\sigma_1^2}
    = \frac{
    1 +  \frac{\int_{-a}^b q(x)\,dx}{2\int_b^\infty q(x)
        \,dx}
    }
    {1 + \frac{\int_{-a}^b x^2q(x)\,dx}{2\int_b^\infty x^2q(x)\,dx}}
    \leq 1 + \frac{\int_{-a}^b q(x)\,dx}{2\int_b^\infty q(x)
        \,dx}
    \leq 1 + \frac{er}{1-e r}
    < 1+2er
    \]
    
\end{proof}

\begin{lemma}
    \label{lemma_almost_sym_mono_var}
    The variance $\sigma_2^2$ is monotonically increasing with respect to $b$. Furthermore, $\sigma_2^2\geq 1$.
\end{lemma}
\begin{proof}
    By taking the derivative, 
    \[
    \frac{\partial \sigma_2^2}{\partial b}
    = -\frac{q(b)}{(\int_b^\infty q(x)\,dx)^2}(b^2\int_b^\infty q(x)\,dx - \int_b^\infty x^2q(x)\,dx)>0
    \]
    So for $b>0$,
    \[
    \sigma_2^2 \geq \frac{\int_0^\infty x^2q(x)\,dx}{\int_0^\infty q(x)\,dx}=1
    \]
\end{proof}

\begin{lemma}[Approximation for Re-weighted Second Moment]
\label{lemma_almost_sym_approx_reweighted_second_moment}
    For $|a+b|\leq \eps^{5}$, by choosing $\alpha=-c\eps^2$, then for some constant $C'>0$, we have the following inequalities. 
    \begin{align}
    \label{eq:almost_sym_approx_second_moment_b}
        \int_b^\infty e^{\alpha y^2/\sigma_2^2}y^2q(y)\,dy - C'\eps^{5}
    \leq
    \int_b^\infty e^{\alpha (y-\mu_1)^2/\sigma_1^2}y^2 q(y)\,dy
    \leq 
     \int_b^\infty e^{\alpha y^2/\sigma_2^2}y^2q(y)\,dy
    \end{align}
    \begin{align}
        \label{eq:almost_sym_approx_second_moment_a}
        \int_b^\infty e^{\alpha y^2/\sigma_2^2}y^2q(y)\,dy - C'\eps^5
    \leq
    \int_{-a}^\infty e^{\alpha (y+\mu_1)^2/\sigma_1^2}y^2 q(y)\,dy
    \leq 
    \int_b^\infty e^{\alpha y^2/\sigma_2^2}y^2q(y)\,dy+ C'\eps^5
    \end{align}
\end{lemma}
\begin{proof}
    By Lemma~\ref{lemma_almost_sym_mu_range}, we can bound $|\mu_1|$ as
    \[
    |\mu_1| < \eps^{4}\ln\frac{1}{\eps} < \eps^3
    \]
    We begin with showing that $\int_b^\infty e^{\alpha y^2/\sigma_2^2}y^2 q(y)\,dy$ is close to $\int_b^\infty e^{\alpha y^2/\sigma_1^2}y^2 q(y)\,dy$.
    By Lemma~\ref{lemma_almost_sym_var_bound}, 
    \[
    \frac{\sigma_2^2}{1+2e\eps^3} \leq \sigma_1^2 \leq \sigma_2^2
    \]
    Since $\alpha<0$, $\alpha y^2/\sigma_1^2 < \alpha y^2/\sigma_2^2 $ for $y>0$. This implies 
    \[
    \int_b^\infty e^{\alpha y^2/\sigma_1^2}y^2 q(y)\,dy
    \leq \int_b^\infty e^{\alpha y^2/\sigma_2^2}y^2 q(y)\,dy.
    \]
    On the other hand,
    \begin{align*}
        &\int_b^\infty e^{\alpha y^2/\sigma_2^2}y^2 q(y)\,dy
        -\int_b^\infty e^{\alpha y^2/\sigma_1^2}y^2 q(y)\,dy\\
        \leq & 
        \int_b^\infty e^{\alpha y^2/\sigma_2^2}y^2 q(y)\,dy
        -\int_b^\infty e^{\alpha y^2(1+2e\eps^3)/\sigma_2^2}y^2 q(y)\,dy\\
        =& \int_b^\infty
        \left(
        1- e^{2e\alpha y^2\eps^3/\sigma_2^2}
        \right)
        e^{\alpha y^2/\sigma_2^2}y^2 q(y)\,dy\\
        \leq &
        \int_b^\infty 
        \frac{2e|\alpha| y^2\eps^3}{\sigma_2^2} 
        e^{\alpha y^2/\sigma_2^2}y^2 q(y)\,dy\\
        =& \frac{2e|\alpha| \eps^3}{\sigma_2^2}
        \int_b^\infty e^{\alpha y^2/\sigma_2^2}y^4 q(y)\,dy\\
        \leq & 
        \frac{2e|\alpha| \eps^3}{\sigma_2^2}\int_0^\infty y^4 q(y)\,dy\\
        \leq & 
        \frac{2e|\alpha| \eps^3}{\sigma_2^2} * 8^4/2\\
        =& \frac{8^4e c\eps^5}{\sigma_2^2}
    \end{align*}
    The last inequality is implied by Lemma~\ref{lemma_logconcave_moments}. Furthermore, by Lemma~\ref{lemma_almost_sym_mono_var}, $\sigma_2^2\geq 1$. So there exists a constant $c_1>0$ such that
    \begin{align}
    \label{eq:almost_sym_sigma1_sigma_2}
        \int_b^\infty e^{\alpha y^2/\sigma_2^2}y^2 q(y)\,dy
    - c_1\eps^5
    \leq
    \int_b^\infty e^{\alpha y^2/\sigma_1^2}y^2 q(y)\,dy
    \leq \int_b^\infty e^{\alpha y^2/\sigma_2^2}y^2 q(y)\,dy.
    \end{align}
    This also applies for the integral from $-a$ to $\infty$. To be specific,
    \begin{align}
        \label{eq:almost_sym_sigma1_sigma_2_apart}
        \int_{-a}^\infty e^{\alpha y^2/\sigma_2^2}y^2 q(y)\,dy
    - c_1\eps^5
    \leq
    \int_{-a}^\infty e^{\alpha y^2/\sigma_1^2}y^2 q(y)\,dy
    \leq \int_{-a}^\infty e^{\alpha y^2/\sigma_2^2}y^2 q(y)\,dy.
    \end{align}
    Next we will show that $\int_b^\infty e^{\alpha y^2/\sigma_1^2}y^2q(y)\,dy$ and $\int_b^\infty e^{\alpha (y-\mu_1)^2/\sigma_1^2}y^2 q(y)\,dy$ are close to each other.
    Since $\mu_1<0,\alpha<0$, we derive that $\alpha y^2>\alpha(y-\mu_1)^2$. This implies
    \[
    \int_b^\infty e^{\alpha y^2/\sigma_1^2}y^2q(y)\,dy
    >
    \int_b^\infty e^{\alpha (y-\mu_1)^2/\sigma_1^2}y^2 q(y)\,dy
    \]
    On the other hand,
    \begin{align*}
        &
         \int_b^\infty e^{\alpha y^2/\sigma_1^2}y^2q(y)\,dy
        -
        \int_b^\infty e^{\alpha (y-\mu_1)^2/\sigma_1^2}y^2 q(y)\,dy
        \\
        \leq &
        \int_b^\infty \left(1-
        e^{-(|\alpha |\mu_1^2+|\alpha \mu_1|y)/\sigma_1^2} 
        \right)  e^{\alpha y^2/\sigma_1^2}y^2q(y)\,dy\\
        \leq &
        \frac{|\alpha \mu_1|}{\sigma_1^2}
        \int_b^\infty 
         e^{\alpha y^2/\sigma_1^2} (y + |\mu_1|)y^2q(y)\,dy\\
         \leq &
         \frac{|\alpha \mu_1|}{\sigma_1^2}\int_0^\infty (y+|\mu_1|)y^2 q(y)\,dy\\
         \leq &
         c_2\eps^5 \quad \text{ for some constant }c_2>0
    \end{align*}
    The last inequality is implied by Lemma~\ref{lemma_logconcave_moments} and Lemma~\ref{lemma_almost_sym_mono_var} . Combining two inequalities, we get
    \begin{align}
    \label{eq:almost_sym_sigma1_with_mu1}
        \int_b^\infty e^{\alpha y^2/\sigma_1^2}y^2q(y)\,dy
        -c_2\eps^5
    \leq 
    \int_b^\infty e^{\alpha (y-\mu_1)^2/\sigma_1^2}y^2 q(y)\,dy
    \leq 
     \int_b^\infty e^{\alpha y^2/\sigma_1^2}y^2q(y)\,dy
    \end{align}
    Similarly, we will show the approximation inequality for $\int_{-a}^\infty e^{\alpha (y+\mu_1)^2/\sigma_1^2}y^2q(y)\,dy$. We will decompose the integral by the summation of the integral on $[-a,6\ln(1/\eps)]$ and $(6\ln(1/\eps),\infty)$ respectively.
    For the first part of the integral,
    \begin{align*}
         &\left\lvert \int_{-a}^{6\ln(1/\eps)} e^{\alpha (y+\mu_1)^2/\sigma_1^2} y^2q(y)\,dy
        - \int_{-a}^{6\ln(1/\eps)} e^{\alpha y^2/\sigma_1^2} y^2q(y)\,dy\right\rvert \\
        =& \left\lvert
        \int_{-a}^{6\ln(1/\eps)}e^{\alpha y^2/\sigma_1^2}
        \left(
        e^{\frac{\alpha\mu_1}{\sigma_1^2}(2y+\mu_1)}-1
        \right) y^2q(y)\,dy
        \right\rvert\\
        \leq & 
        \int_{-a}^{-\mu_1/2} e^{\alpha y^2/\sigma_1^2}
        \left(1-
        e^{\frac{\alpha\mu_1}{\sigma_1^2}(2y+\mu_1)}
        \right) y^2q(y)\,dy
        +
        \int_{-\mu_1/2}^{6\ln(1/\eps)} e^{\alpha y^2/\sigma_1^2}
        \left(
        e^{\frac{\alpha\mu_1}{\sigma_1^2}(2y+\mu_1)}-1
        \right) y^2q(y)\,dy
    \end{align*}
    The first term can be bounded using $e^{-t}\geq 1-t$. 
    \begin{align*}
        \int_{-a}^{-\mu_1/2} e^{\alpha y^2/\sigma_1^2}
        \left(1-
        e^{\frac{\alpha\mu_1}{\sigma_1^2}(2y+\mu_1)}
        \right) y^2q(y)\,dy
        \leq& \int_{-a}^{-\mu_1/2} \frac{\alpha \mu_1}{\sigma_1^2}(-2y-\mu_1)y^2q(y)\,dy\\
        \leq & \int_{-a}^{-\mu_1/2} |\alpha||\mu_1|^2 y^2q(y)\,dy
        \leq \eps^8
    \end{align*}
    The second term can be bounded using the upper limit of the integral. For $y \leq 6\ln(1/\eps)$, 
    \[
    e^{\frac{\alpha\mu_1}{\sigma_1^2}(2y+\mu_1)}-1
    \leq  e^{24\eps^6\ln^2(1/\eps)}-1 \leq 40\eps^5
    \]
    Substituting into the second term, we get
    \[
    \int_{-\mu_1/2}^{6\ln(1/\eps)} e^{\alpha y^2/\sigma_1^2}
        \left(
        e^{\frac{\alpha\mu_1}{\sigma_1^2}(2y+\mu_1)}-1
        \right) y^2q(y)\,dy
        \leq 40\eps^5 \int_0^\infty y^2q(y)\,dy = 20\eps^5
    \]
    Combining both terms, 
    \begin{align}
    \label{eq:almost_sym_first}
        \left\lvert \int_{-a}^{6\ln(1/\eps)} e^{\alpha (y+\mu_1)^2/\sigma_1^2} y^2q(y)\,dy
        - \int_{-a}^{6\ln(1/\eps)} e^{\alpha y^2/\sigma_1^2} y^2q(y)\,dy\right\rvert
        \leq 21\eps^5
    \end{align}
    For the remaining part of the integral, we can bound using logconcave distribution's upper bound as in Lemma~\ref{lemma_logconcave_decayfast}. For any $t \geq 3$, 
    \[
    q(t) \leq q(0) \cdot 2^{-t/3} < e^{-t/5}
    \]
    Then the following holds with some constant $c_3'>0$.
    \begin{align}
    \label{eq:almost_sym_exp_second1}
        \int_{6\ln(1/\eps)}^{\infty} e^{\alpha y^2/\sigma_1^2} y^2q(y)\,dy
        \leq \int_{6\ln(1/\eps)}^\infty  y^2 e^{-y/5}\,dy 
        \leq (1+6\ln (\frac{1}{\eps}))e^{-6\ln(\frac{1}{\eps} )}\leq c_3' \eps^{5}
    \end{align}
    Similarly, we get
    \begin{align}
    \label{eq:almost_sym_exp_second2}
        \int_{6\ln(1/\eps)}^{\infty} e^{\alpha (y+\mu_2)^2/\sigma_1^2} y^2q(y)\,dy
        \leq c_3'\eps^{5}
    \end{align}
    With Equations~\eqref{eq:almost_sym_first}, \eqref{eq:almost_sym_exp_second1}, \eqref{eq:almost_sym_exp_second2}, we get
    \begin{align}
    \label{eq:almost_sym_sigma1_with_mu1_apart}
        \left\lvert \int_{-a}^{\infty} e^{\alpha (y+\mu_1)^2/\sigma_1^2} y^2q(y)\,dy
        - \int_{-a}^{\infty} e^{\alpha y^2/\sigma_1^2} y^2q(y)\,dy\right\rvert \leq c_3\eps^5 \text{ for constant }c_3>0
    \end{align}
    % not used
    % In the next step, we would like to show  $\int_b^\infty e^{\alpha y^2/\sigma_1^2}y^2 q(y)\,dy$ and  $\int_{-a}^\infty e^{\alpha y^2/\sigma_1^2}y^2 q(y)\,dy$ are close. 
    % \begin{align*}
    %     &\int_{-a}^\infty e^{\alpha y^2/\sigma_1^2}y^2 q(y)\,dy
    %     -\int_b^\infty e^{\alpha y^2/\sigma_1^2}y^2 q(y)\,dy\\
    %     =& \int_{-a}^b e^{\alpha y^2/\sigma_1^2}y^2 q(y)\,dy
    %     %%%%%very loose, can tight later
    %     \leq b\int_{-a}^b yq(y)\,dy
    %     \leq  \ln \frac{1}{\eps} |\mu_1| < \eps^3
    % \end{align*}
    % Since $|-a|\leq b$, we have
    % \begin{align}
    % \label{eq:almost_sym_ab_similar}
    %     \int_b^\infty e^{\alpha y^2/\sigma_1^2}y^2 q(y)\,dy
    %     \leq \int_{-a}^\infty e^{\alpha y^2/\sigma_1^2}y^2 q(y)\,dy
    %     \leq \int_b^\infty e^{\alpha y^2/\sigma_1^2}y^2 q(y)\,dy
    %     + \eps^3
    % \end{align}
    Combining Equations \eqref{eq:almost_sym_sigma1_sigma_2}, 
    \eqref{eq:almost_sym_sigma1_sigma_2_apart},
    \eqref{eq:almost_sym_sigma1_with_mu1}, \eqref{eq:almost_sym_sigma1_with_mu1_apart} %,\eqref{eq:almost_sym_ab_similar} 
    , we have
    \begin{align*}
    %\label{eq:almost_sym_approx_second_moment_b}
        \int_b^\infty e^{\alpha y^2/\sigma_2^2}y^2q(y)\,dy - (c_1+c_2)\eps^{5}
    \leq
    \int_b^\infty e^{\alpha (y-\mu_1)^2/\sigma_1^2}y^2 q(y)\,dy
    \leq 
     \int_b^\infty e^{\alpha y^2/\sigma_2^2}y^2q(y)\,dy
    \end{align*}
    \begin{align*}
        %\label{eq:almost_sym_approx_second_moment_b}
        \int_b^\infty e^{\alpha y^2/\sigma_2^2}y^2q(y)\,dy - (c_1+c_3)\eps^{5}
    \leq
    \int_{-a}^\infty e^{\alpha (y+\mu_1)^2/\sigma_1^2}y^2 q(y)\,dy
    \leq 
    \int_b^\infty e^{\alpha y^2/\sigma_2^2}y^2q(y)\,dy+ c_3\eps^5
    \end{align*}
    By choosing $C'=c_1+c_2+c_3$, we prove the lemma.
   
\end{proof}

\begin{lemma}[Approximated for Re-weighted Zeroth Moment] 
\label{lemma_almost_sym_approximation_reweighted_Zero_moment}
By choosing $\alpha=-c\eps^2$, for some constant $C'>0$, we have the following inequalities.
    \begin{align}
    \label{eq:almost_sym_approx_zero_moment_b}
        \int_b^\infty e^{\alpha y^2/\sigma_2^2} q(y)\,dy - C'\eps^{5}
    \leq
    \int_b^\infty e^{\alpha (y-\mu_1)^2/\sigma_1^2} q(y)\,dy
    \leq 
     \int_b^\infty e^{\alpha y^2/\sigma_2^2} q(y)\,dy
    \end{align}
    \begin{align}
        \label{eq:almost_sym_approx_zero_moment_a}
        \int_b^\infty e^{\alpha y^2/\sigma_2^2} q(y)\,dy - C'\eps^5
    \leq
    \int_{-a}^\infty e^{\alpha (y+\mu_1)^2/\sigma_1^2} q(y)\,dy
    \leq 
    \int_b^\infty e^{\alpha y^2/\sigma_2^2} q(y)\,dy+ C'\eps^5
    \end{align}
\end{lemma}
The proof follows exactly from the proof of Lemma~\ref{lemma_almost_sym_approx_reweighted_second_moment} by replacing $y^2$ with $1$.

Now we are ready to prove Lemma~\ref{lemma_reweighted_cov}.
\reweightedcov*

\begin{proof}
    By calculation, we have
    \begin{align*}
        \E\limits_{x\sim \hat{r}} e^{\alpha x^2}x^2
        = \frac{
        \int_{b/\sigma_2}^\infty
        e^{\alpha x^2}x^2\sigma_2 q(x\sigma_2)\,dx
        }{
        \int_b^\infty q(x)\,dx
        }
        = \frac{
        \int_b^\infty e^{\alpha y^2/\sigma_1^2}y^2 q(y)\,dy
        }{
        \sigma_2^2\int_b^\infty q(x)\,dx
        }
    \end{align*}
    \begin{align*}
        &\E\limits_{x\sim \hat{q}} e^{\alpha x^2}x^2\\
        =&\frac{
        \int_{-\infty}^{(a-\mu_1)/\sigma_1} e^{\alpha x^2}x^2 \sigma_1 q(x\sigma_1+\mu_1)\,dx
        + \int_{(b-\mu_1)/\sigma_1} ^\infty e^{\alpha x^2}x^2 \sigma_1 q(x\sigma_1+\mu_1)\,dx
        }{
        \int_{-a}^\infty q(x)\,dx + \int_b^\infty q(x)\,dx
        }
        \\
        =&
        \frac{
        \int_{-\infty}^a e^{\alpha (y-\mu_1)^2/\sigma_1^2} (y-\mu_1)^2q(y)\,dy
        + \int_b^\infty e^{\alpha (y-\mu_1)^2/\sigma_1^2} (y-\mu_1)^2q(y)\,dy
        }{
        \sigma_1^2\left(
        \int_{-a}^\infty q(x)\,dx + \int_b^\infty q(x)\,dx
        \right)
        }\\
        =&
         \frac{
        \int_{-a}^\infty e^{\alpha (y+\mu_1)^2/\sigma_1^2} (y+\mu_1)^2q(y)\,dy
        + \int_b^\infty e^{\alpha (y-\mu_1)^2/\sigma_1^2} (y-\mu_1)^2q(y)\,dy
        }{\sigma_1^2\left(
        \int_{-a}^\infty q(x)\,dx + \int_b^\infty q(x)\,dx
        \right)}\\
        =& \frac{
        \int_{-a}^\infty e^{\alpha (y+\mu_1)^2/\sigma_1^2} y^2q(y)\,dy
        + \int_b^\infty e^{\alpha (y-\mu_1)^2/\sigma_1^2} y^2q(y)\,dy
        }{\sigma_1^2\left(
        \int_{-a}^\infty q(x)\,dx + \int_b^\infty q(x)\,dx
        \right)}\\
        & + \frac{2\mu_1}{\sigma_1^2}
        \frac{
        \int_{-a}^\infty e^{\alpha (y+\mu_1)^2/\sigma_1^2} yq(y)\,dy
        - \int_b^\infty e^{\alpha (y-\mu_1)^2/\sigma_1^2} yq(y)\,dy
        }{
        \int_{-a}^\infty q(x)\,dx + \int_b^\infty q(x)\,dx
        }\\
        &+ \frac{\mu_1^2}{\sigma_1^2}
        \frac{
        \int_{-a}^\infty e^{\alpha (y+\mu_1)^2/\sigma_1^2} q(y)\,dy
        +\int_b^\infty e^{\alpha (y-\mu_1)^2/\sigma_1^2} q(y)\,dy
        }{
         \int_{-a}^\infty q(x)\,dx + \int_b^\infty q(x)\,dx
        }
    \end{align*}
    The first term is close to $ \E\limits_{x\sim \hat{r}} e^{\alpha x^2}x^2$ while the second and third terms are close to zero. We first give the bound on the absolute values of last two terms.
    Since $\alpha<0$,
    \begin{align*}
        &\left\lvert\int_{-a}^\infty e^{\alpha (y+\mu_1)^2/\sigma_1^2} yq(y)\,dy
        - \int_b^\infty e^{\alpha (y-\mu_1)^2/\sigma_1^2} yq(y)\,dy
        \right\rvert\\
        \leq & 
        \int_{-a}^\infty e^{\alpha (y+\mu_1)^2/\sigma_1^2} yq(y)\,dy
        + 
        \int_{b}^\infty e^{\alpha (y-\mu_1)^2/\sigma_1^2} yq(y)\,dy \\
        \leq &\int_{-a}^\infty yq(y)\,dy + \int_b^\infty yq(y)\,dy
        <1
    \end{align*}
    Similarly,
    \[
    \int_{-a}^\infty e^{\alpha (y+\mu_1)^2/\sigma_1^2} q(y)\,dy
        + \int_b^\infty e^{\alpha (y-\mu_1)^2/\sigma_1^2} q(y)\,dy < 1
    \]
    By Lemma~\ref{lemma_almost_sym_var_bound}, Lemma~\ref{lemma_almost_sym_approx_reweighted_second_moment} and Lemma~\ref{lemma_almost_sym_approximation_reweighted_Zero_moment}, we have
    \begin{align*}
        &\frac{
        \int_{-a}^\infty e^{\alpha (y+\mu_1)^2/\sigma_1^2} y^2q(y)\,dy
        + \int_b^\infty e^{\alpha (y-\mu_1)^2/\sigma_1^2} y^2q(y)\,dy
        }{\sigma_1^2\left(
        \int_{-a}^\infty q(x)\,dx + \int_b^\infty q(x)\,dx
        \right)}\\
        \leq&  \frac{
        2\int_b^{\infty} e^{\alpha y^2/\sigma_2^2} y^2q(y)\,dy + C'\eps^5
        }{
        2\int_b^\infty q(y)\,dy
        }
        \cdot 
        \frac{1+2e\eps^3}{\sigma_2^2}\\
        =& \frac{\int_b^\infty e^{\alpha y^2/\sigma_2^2 } y^2 q(y)\,dy }{\sigma_2^2\int_b^\infty q(y)\,dy}
         + \frac{C'\eps^5}{\sigma_2^2 2\int_b^\infty q(y)\,dy}
         + \frac{2e\eps^3}{\sigma_2^2} \frac{
        2\int_b^{\infty} e^{\alpha y^2/\sigma_2^2} y^2q(y)\,dy + C'\eps^5
        }{
        2\int_b^\infty q(y)\,dy
        }\\
        <& \frac{\int_b^\infty e^{\alpha y^2/\sigma_2^2 } y^2 q(y)\,dy }{\sigma_2^2\int_b^\infty q(y)\,dy}
        +c_1\eps^{4} \quad\text{ for constant }c_1>0
    \end{align*}
    By combining with the second and third terms, we conclude that for constants $c_3,c_4>0$,
    \begin{align}
    \label{eq:almost_sym_exp_final_second}
        \E\limits_{x\sim \hat{r}} e^{\alpha x^2}x^2 - c_3\eps^4 
    < \E\limits_{x\sim \hat{q}} e^{\alpha x^2}x^2
    <\E\limits_{x\sim \hat{r}} e^{\alpha x^2}x^2 + c_4 \eps^4
    \end{align}
    Similarly, we have 
     \begin{align}
     \label{eq:almost_sym_exp_final_zero}
        \E\limits_{x\sim \hat{r}} e^{\alpha x^2} - c_3\eps^4 
    < \E\limits_{x\sim \hat{q}} e^{\alpha x^2}
    <\E\limits_{x\sim \hat{r}} e^{\alpha x^2} + c_4 \eps^4
    \end{align}
    Then, we would like to compute the gap between first and second eigenvalues of the re-weighted second moment of $\hat{q}$. We denote $Q$ as the product of $\hat{q}$ and $d-1$ fold of $q$. For $2\leq j\leq d$,
    \begin{align*}
        &\E\limits_{x\sim Q} e^{\alpha \|x\|^2}x_1^2 - \E\limits_{x\sim Q} e^{\alpha \|x\|^2}x_j^2\\
        =&\left(
        \E\limits_{x_1\sim \hat{q}}e^{\alpha x_1^2}x_1^2 
        \E\limits_{x_2\sim q}e^{\alpha x_2^2}
        -
        \E\limits_{x_2\sim q}e^{\alpha x_2^2}x_2^2
        \E\limits_{x_1\sim \hat{q}}e^{\alpha x_1^2}
        \right)
        \left(
        \E\limits_{x\sim q}e^{\alpha x^2}
        \right)^{d-2}
    \end{align*}
    We define $T(\alpha)$ as follows.
    \[
    T(\alpha)
    =\E\limits_{x_1\sim \hat{q}}e^{\alpha x_1^2}x_1^2 
        \E\limits_{x_2\sim q}e^{\alpha x_2^2}
        -
        \E\limits_{x_2\sim q}e^{\alpha x_2^2}x_2^2
        \E\limits_{x_1\sim \hat{q}}e^{\alpha x_1^2}
    \]
    Recall that $S(\alpha)$ is defined in Lemma~\ref{lemma_cov_sym_S_gap}.
    \[
    S(\alpha)
    =\E\limits_{x_1\sim \hat{r}}e^{\alpha x_1^2}x_1^2 
        \E\limits_{x_2\sim q}e^{\alpha x_2^2}
        -
        \E\limits_{x_2\sim q}e^{\alpha x_2^2}x_2^2
        \E\limits_{x_1\sim \hat{r}}e^{\alpha x_1^2}
    \]
    We calculate the difference between $T(\alpha)$ and $S(\alpha)$ using Equation~\eqref{eq:almost_sym_exp_final_second} and Equation~\eqref{eq:almost_sym_exp_final_zero}.
    \begin{align*}
        &S(\alpha)-T(\alpha)\\
        =& \left(\E\limits_{x_1\sim \hat{r}}e^{\alpha x_1^2}x_1^2 
        -\E\limits_{x_1\sim \hat{q}}e^{\alpha x_1^2}x_1^2 
        \right)\E\limits_{x_2\sim q}e^{\alpha x_2^2}
        -\left(
        \E\limits_{x_1\sim \hat{r}}e^{\alpha x_1^2}
        -\E\limits_{x_1\sim \hat{q}}e^{\alpha x_1^2}
        \right)
         \E\limits_{x_2\sim q}e^{\alpha x_2^2}x_2^2\\
         \leq & c_3\eps^4 + c_4 \eps^4
    \end{align*}
    By Lemma~\ref{lemma_cov_sym_S_gap}, $S(\alpha)>C'\eps^3$ for constant $C'>0$. So we know $T(\alpha )> C\eps^3$ for $C>0$. Then all proof follows as same as the case when $a+b=0$. We write out the proof for completeness.
    
    For $2\leq j\leq d$,
    \begin{align*}
        &\E\limits_{x\sim {P}}e^{\alpha \|x\|^2}x_1^2 -  \E\limits_{x\sim {P}}e^{\alpha \|x\|^2}x_j^2\\
        =& \left(
        \E\limits_{x_1\sim \hat{q}}e^{\alpha x_1^2}x_1^2 \E\limits_{x_2\sim q}e^{\alpha x_2^2}
        - \E\limits_{x_2\sim \hat{q}}e^{\alpha x_2^2}x_1^2 \E\limits_{x_1\sim q}e^{\alpha x_1^2}
        \right)
        \left(\E\limits_{x\sim q} e^{\alpha x^2}\right)^{d-2}\\
        =& T(\alpha ) \frac{\E\limits_{x\sim {P}}e^{\alpha \|x\|^2} x_1^2}
        {\E\limits_{x_1\sim \hat{q}}e^{\alpha x_1^2}x_1^2 \E\limits_{x_2\sim q}e^{\alpha x_2^2}}
    \end{align*}
    Since $\alpha_3<0$, we have
    \[
    \E\limits_{x_1\sim \hat{q}}e^{\alpha_3 x_1^2}x_1^2\leq \E\limits_{x_1\sim \hat{q}} x_1^2=1, \E\limits_{x_2\sim q}e^{\alpha_3 x_2^2}\leq 1.
    \]
    Also we have shown that $T(\alpha_3)>C\eps^3$. So we have
    \[
    \E\limits_{x\sim {P}}e^{\alpha \|x\|^2}x_1^2 -  \E\limits_{x\sim {P}}e^{\alpha \|x\|^2}x_j^2
    \geq C\eps^3 \E\limits_{x\sim {P}}e^{\alpha \|x\|^2}x_1^2
    \]
    Finally we will show that the first eigenvector corresponds to $e_1$.
    For any $v\in\R$, define $\phi(v)$ as
    \[
    \phi(v):=\mathop{\E}\limits_{x\sim {P}} \frac{e^{\alpha_3 \|x\|^2v^\top xx^\top v}}{v^\top v}
    \]
    Then we know for $2\leq j\leq d$, 
    \[
    \phi(e_1)-\phi(e_j) > C\eps^3\phi(e_1)
    \]
    For any vector $v=\sum_{i=1}^d \gamma_ie_i$, we have
    \begin{align*}
    \phi(v) =& \frac{1}{\sum_{i=1}^d \gamma_i^2}\E e^{\alpha_3 \|x\|^2}(\sum_{i=1}^d \gamma_i x_i)^2\\
    =& \frac{1}{\sum_{i=1}^d \gamma_i^2}
    \left(
    \sum_{i=1}^d \gamma_i^2 \phi(e_i) + 2\E e^{\alpha\|x\|^2}\sum_{i\neq j}\gamma_i\gamma_j x_i x_j
    \right)\\
    =& \frac{1}{\sum_{i=1}^d \gamma_i^2}
    \left(
    \sum_{i=1}^d \gamma_i^2 \phi(e_i) 
    + 2\sum_{i\neq j}\gamma_i\gamma_j \E e^{\alpha_3 \sum_{k\neq i,j} x_k^2}
    \E e^{\alpha_3 \left<x,e_i\right>^2}x_1
    \E e^{\alpha \left<x,e_j\right>^2}x_j
    \right)\\
    =& \frac{\sum_{i=1}^d \gamma_i^2 \phi(e_i)}{\sum_{i=1}^d \gamma_i^2}\\
    \leq & \phi(e_1)
    \end{align*}
        This shows that the top eigenvalue of $\tilde{\Sigma}$ is $\lambda_1=\max_v g(v)=g(e_1)$. In other word, the top eigenvector is $e_1$.
        Similarly the second eigenvalue of $\tilde{\Sigma}$ is $\lambda_2 =\max_{v:v\bot e_1}\phi(v) = \phi(e_j),2\leq j\leq d$. So we get
    \[
    \lambda_1-\lambda_2 > C\eps^3\lambda_1
    \]

\end{proof}

\section{Experiments}\label{section:experiments}

While our primary goal is to establish polynomial bounds on the sample and time complexity, our algorithms are natural and easy to implement. 
We study the efficiency and performance of Algorithm~\ref{algo_general} on data drawn from affine product distributions with margin. Here we consider three special cases of logconcave distribution: Gaussian, uniform in an interval and exponential. We include four experiments. In all results, we measure the performance of the algorithm using the $\sin$ of the angle between the true normal vector $u$ and the predicted vector $\hat{u}$, i.e.,
$\sin \theta(u,\hat{u})$, which bounds the $TV$ distance between the underlying distribution and the predicted one after isotropic transformation.
Experimental results strongly suggest that the sample complexity is a small polynomial, perhaps even just nearly linear in both the dimension and the separation parameter $\eps$.

\paragraph{Overall Performance.}
Here we conduct the experiments based on a grid search of $(a,b)$ pairs on three special cases of logconcave distribution: Gaussian, uniform in an interval and exponential. We measure the performance of $(a,b)$ pairs, where for each pair of $(a,b)$, we conduct five independent trials.  For Gaussian and Exponential distribution, we choose $-3\leq a< b\leq 3$ and for Uniform distribution, we choose $-1.5\leq a<b\leq 1.5$. Here we set the dimension $d=10$ and sample size $N=1000000$.
For the parameters, we choose $\alpha_1=\alpha_3=-0.1,\alpha_2=-0.2$. See Figure~\ref{fig:experiment4} as the heatmap of $\sin \theta(u,\hat{u})$ given different pairs of $(a,b)$.

Although in Algorithm~\ref{algo_general}, we use extremely small values of the weight parameter $\alpha$, our experiments show that larger constant values also work empirically, leading to much smaller sample complexity. This coincides with our qualitative lemmas (Lemma~\ref{lemma_contrastive_mean_qual}, Lemma~\ref{lemma_contrastive_covariance_qual}).

The algorithm performs well as seen in the results, except when $a$ and $b$ are both close to the edge, and thus there is almost no mass on one side of the band. Also, the uniform distribution is the easiest to learn, while the exponential is the hardest  among these three distributions.

As shown in all three plots, the algorithm performs the best when $a$ and $b$ are near symmetric with origin. In other words, contrastive covariance has better sample complexity than contrastive mean when we fix other hyperparameters. This coincides with our sample complexity bounds as in the proof of Theorem~\ref{thm:main}.

\begin{figure}[htp]
\centering
\captionsetup[subfigure]{justification=Centering}
\begin{subfigure}[t]{0.3\textwidth}\label{fig:exp4_gaussian}
    \includegraphics[width=\linewidth]{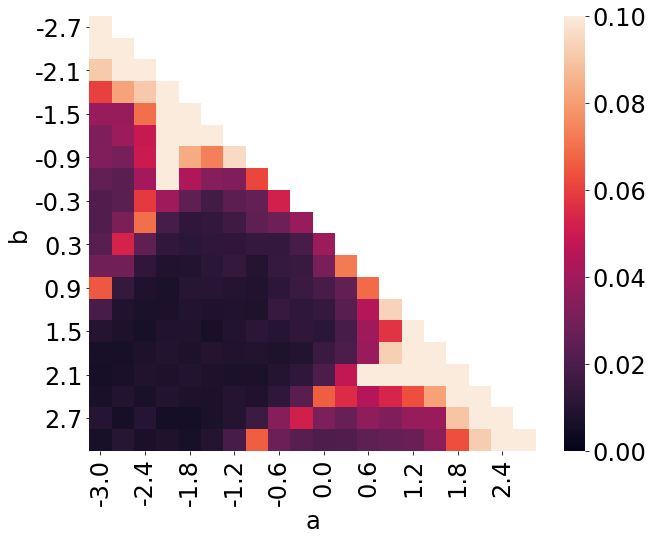}
    \caption{Gaussian}
\end{subfigure}
\begin{subfigure}[t]{0.3\textwidth}\label{fig:exp4_uniform}
    \includegraphics[width=\linewidth]{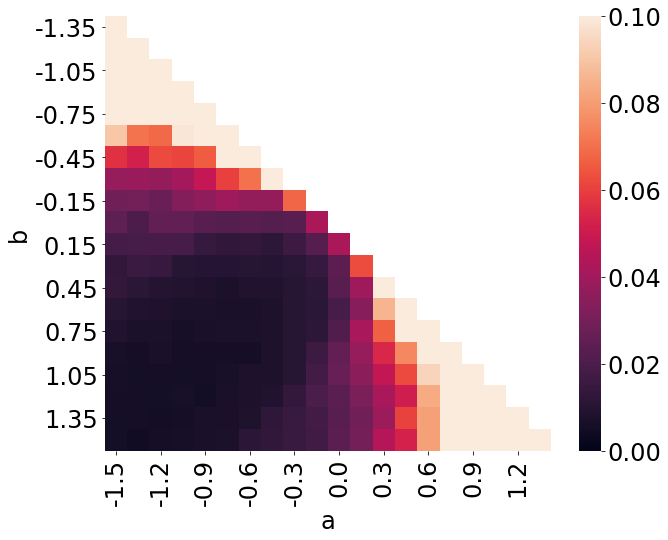}
    \caption{Uniform}
\end{subfigure}
\begin{subfigure}[t]{0.3\textwidth}\label{fig:exp4_expo}
    \includegraphics[width=\linewidth]{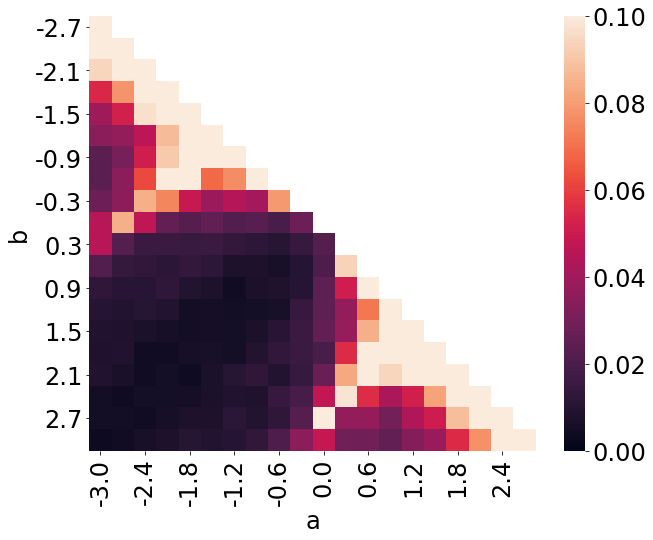}
    \caption{Exponential}
\end{subfigure}
\caption{We test the performance of
Algorithm~\ref{algo_general} based on a grid search of $(a,b)$.}
\label{fig:experiment4}
\end{figure}

\paragraph{Performance of Contrastive Mean and Covariance.}
In this experiment, we fix a negative $a$ as the left endpoint of the removed band, and measure the performance of both contrastive mean and contrastive covariance with respect to different margin right endpoint $b$. As shown in Figure~\ref{fig:experiment3}, contrastive mean performs well except when $a+b$ is close to zero, while contrastive covariance performs well only when $a+b$ is close to zero. This coincides with our algorithm and analysis for the two cases. In addition, our algorithm chooses the best normal vector among candidates from both contrastive mean and covariance. So our algorithm achieves good performance (minimum of contrastive mean and covariance curves).
        
Specifically, we choose $a=-2, b\in[-1.9, 4]$ for Gaussian and Exponential case, and $a=-0.5, b\in[-0.4, 0.9]$ for Uniform case. We choose the dimension $d=10$, the sample size $N=2000000$. We choose $\alpha_1=\alpha_3=-0.1,\alpha_2 = -0.2$. We average the result with $50$ independent trials.

\begin{figure}[htp]
\centering
\captionsetup[subfigure]{justification=Centering}

\begin{subfigure}[t]{0.32\textwidth}
    \includegraphics[width=\textwidth]{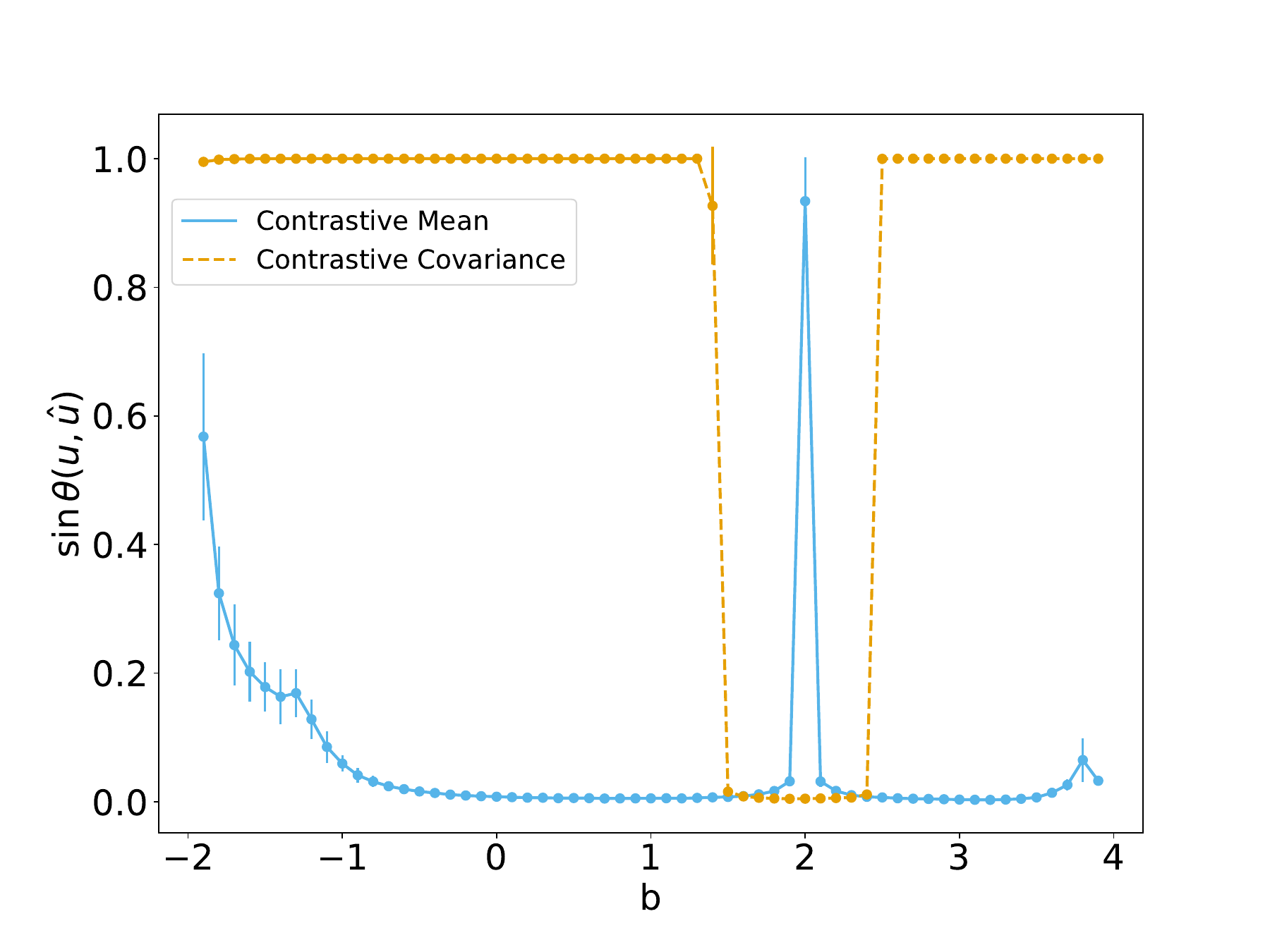}
    \caption{\small Gaussian, $a=-2$.}
    \label{fig:exp2_gaussian}
\end{subfigure} % maximize horizontal separation
\begin{subfigure}[t]{0.32\textwidth}
    \includegraphics[width=\linewidth]{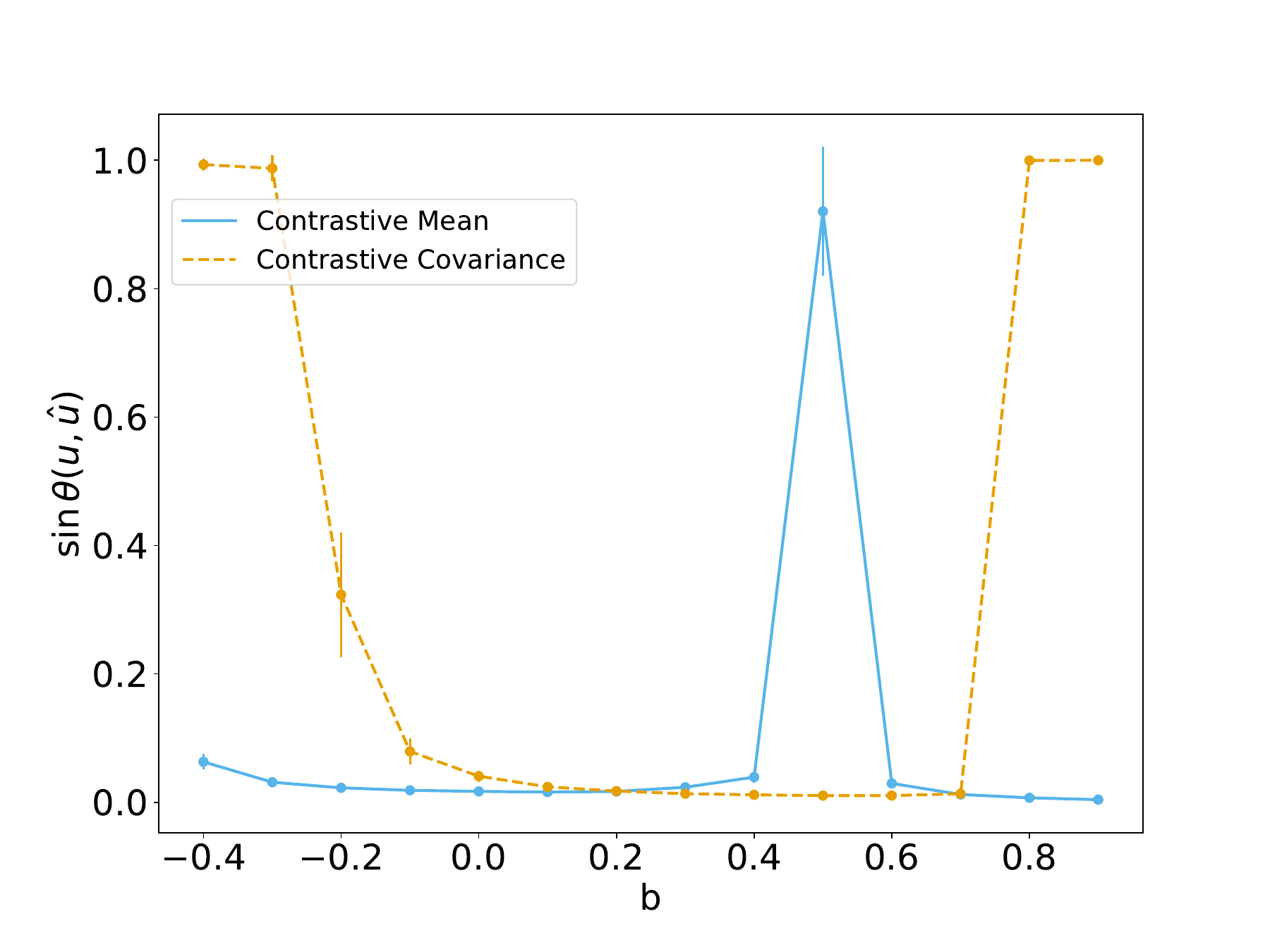}
    \caption{Uniform, $a=-0.5$.}
    \label{fig:exp2_uniform}
\end{subfigure}
\begin{subfigure}[t]{0.32\textwidth}
    \includegraphics[width=\linewidth]{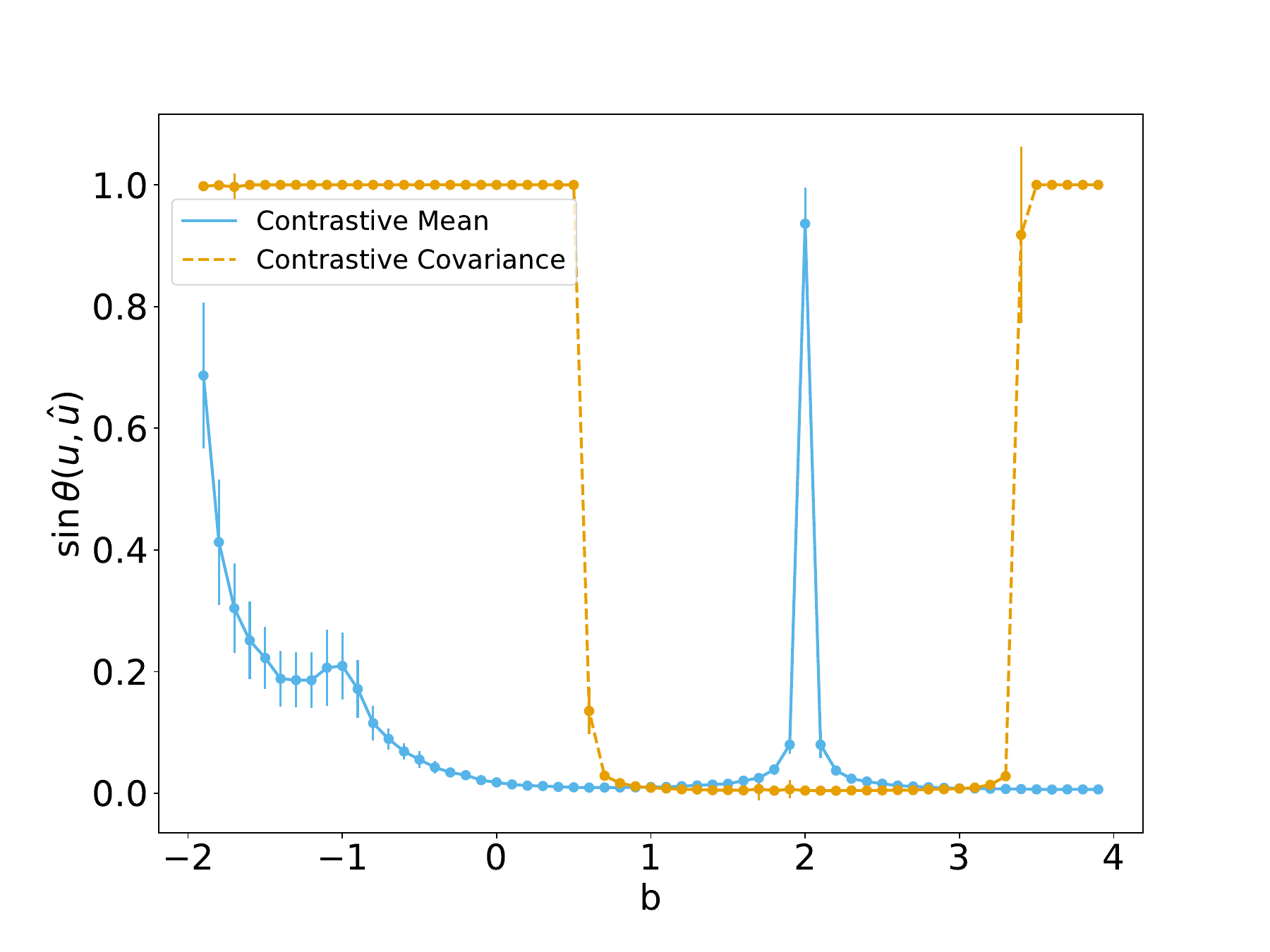}
    \caption{Exponential, $a=-2$.}
    \label{fig:exp2_exp}
\end{subfigure}
\caption{For a fixed $a$, we test the performance of
Algorithm~\ref{algo_general} while changing $b$. The yellow lines show the result computed using the top eigenvector of the contrastive covariance. The blue dotted lines show the better of the two contrastive means.}
\label{fig:experiment3}
\end{figure}

\paragraph{Dimension Dependence.}
In this experiment, we show the relationship between the input dimension $d$ and the sample complexity. For fixed number size $N=1000000$, we measure the performance of our algorithm with different $d$. The result is averaged based on a grid search of $(a,b)$ pairs, where for each pair of $(a,b)$, we conduct five independent trials. For Gaussian and Exponential distribution, we choose $-3\leq a< b\leq 3$ and for Uniform distribution, we choose $-0.8\leq a<b\leq 0.8$.

As shown Figure~\ref{fig:experiment2}, the performance scales linearly with growing dimension $d$, suggesting a linear relationship between the sample complexity and the input dimension.  
\begin{figure}[htp]
\centering
\captionsetup[subfigure]{justification=Centering}
\begin{subfigure}[t]{0.32\textwidth}\label{fig:exp3_gaussian}
    \includegraphics[width=\linewidth]{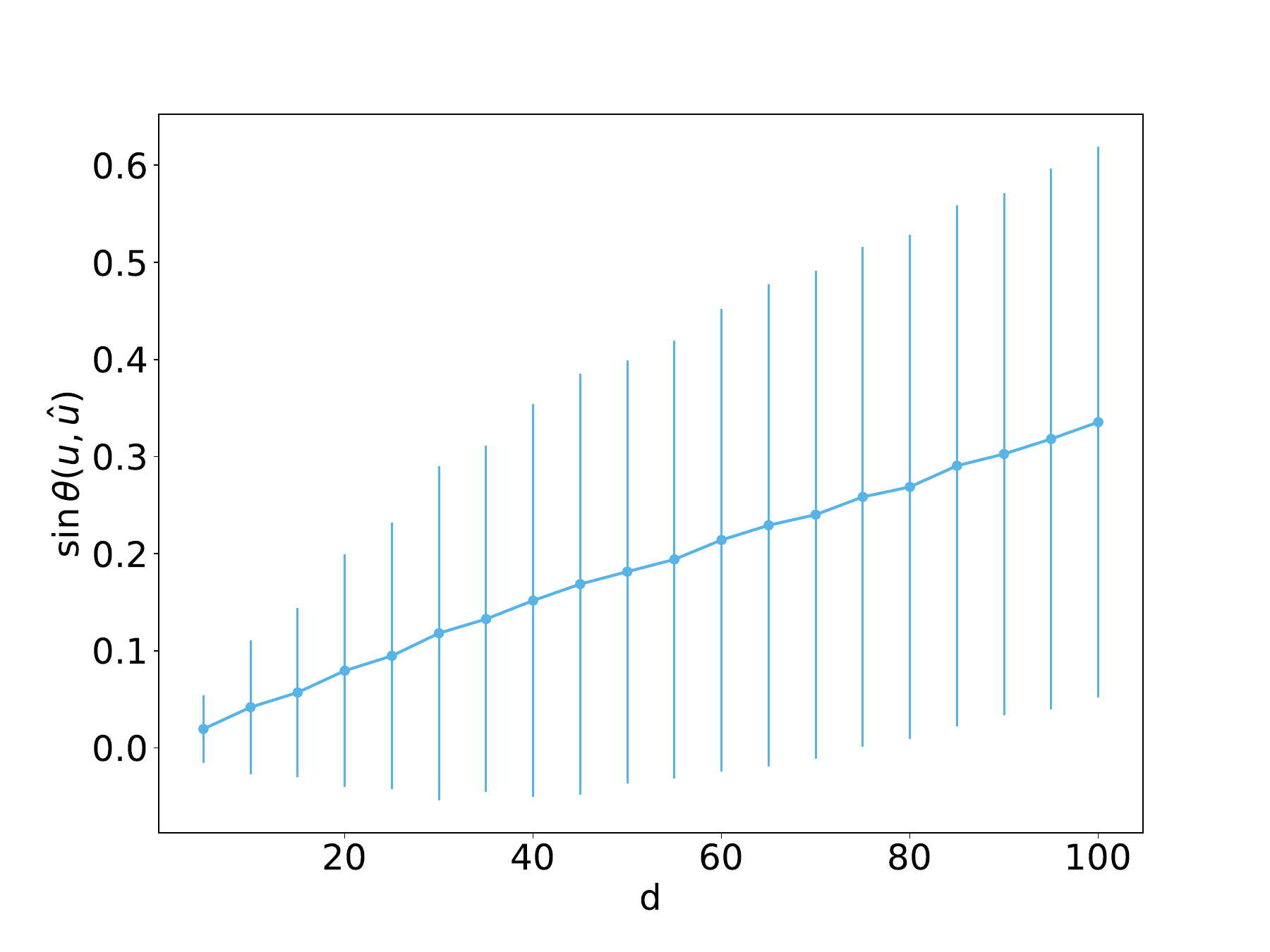}
    \caption{Gaussian}
\end{subfigure}
\begin{subfigure}[t]{0.32\textwidth}\label{fig:exp3_uniform}
    \includegraphics[width=\linewidth]{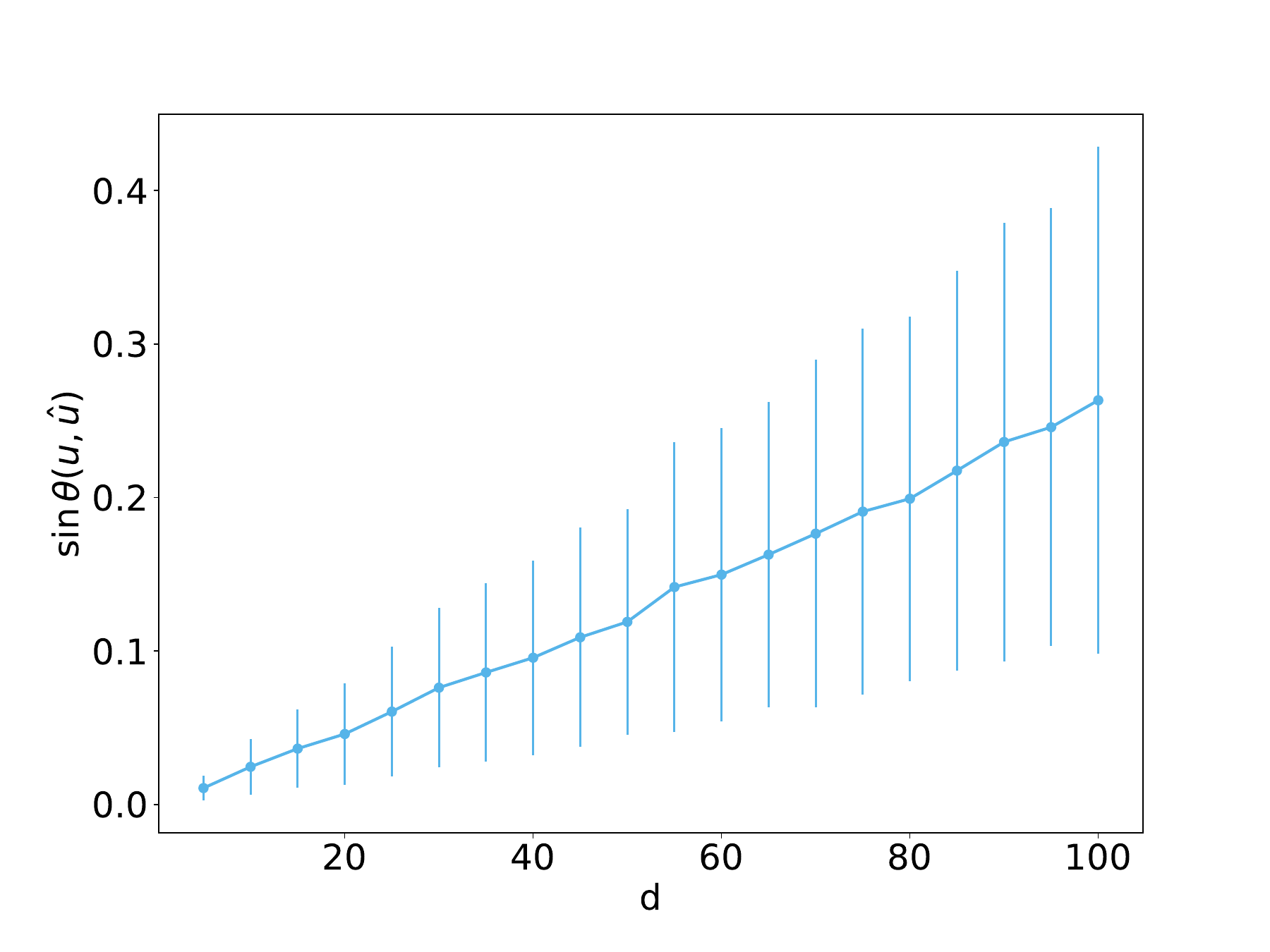}
    \caption{Uniform}
\end{subfigure}
\begin{subfigure}[t]{0.32\textwidth}\label{fig:exp3_expo}
    \includegraphics[width=\linewidth]{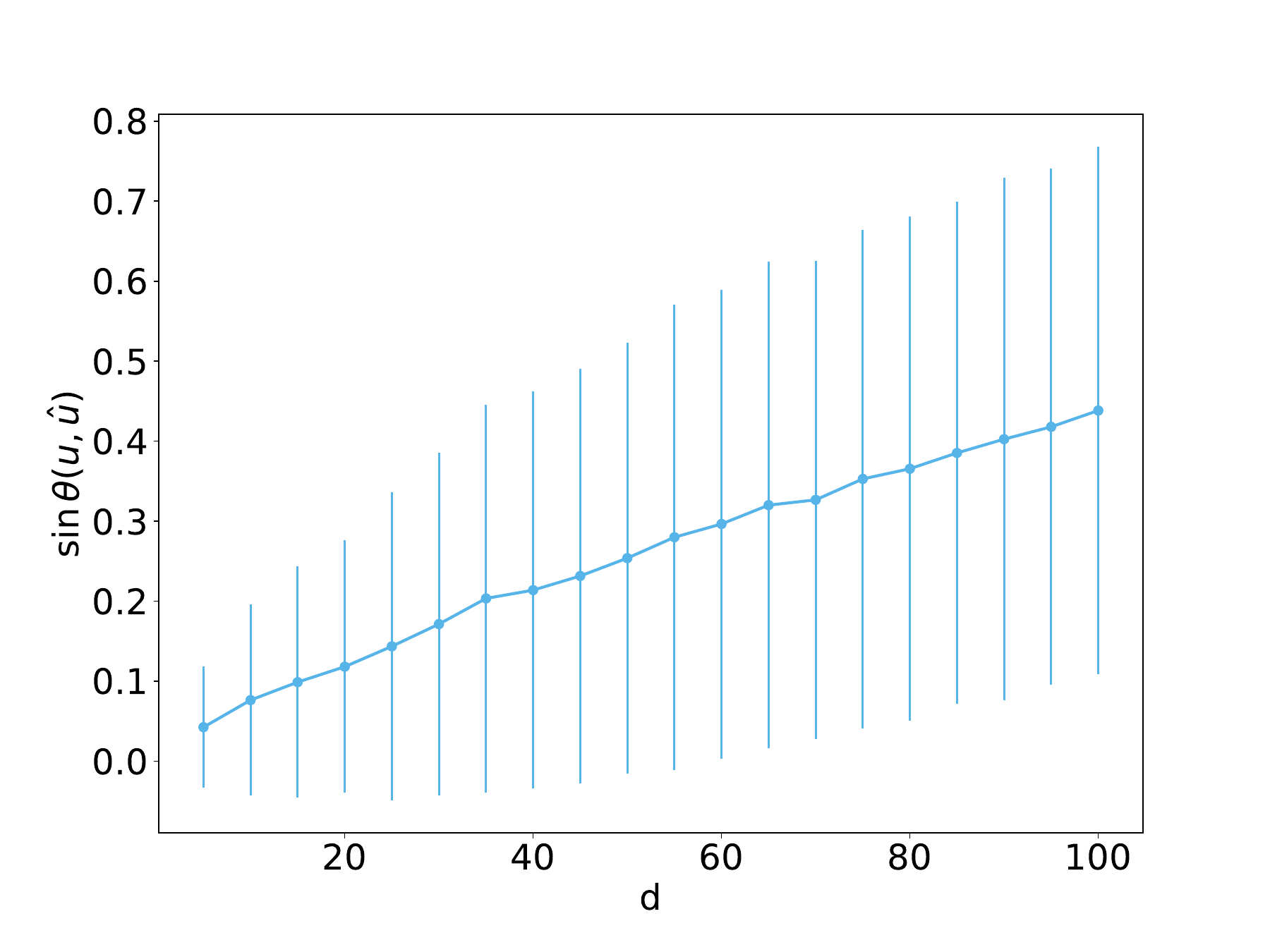}
    \caption{Exponential}
\end{subfigure}
\caption{For a fixed sample size $N$, we test the performance of
Algorithm~\ref{algo_general} by varying the dimension $d$.}
\label{fig:experiment2}
\end{figure}

\paragraph{$\epsilon$-Dependence.}
To further understand the dependence on the separation parameter $\eps$, we plot the performance versus $1/\eps$ in
Figure~\ref{fig:experiment1}. Here we calculate $1/\eps$ as
$1/q([a,b])$, and the performance as the median $\sin \theta(u,\hat{u})$ for specific mass $q([a,b])$. As we can see the performance drops near linearly with respect to $1/\eps$, which indicates that the sample complexity is possibly linear in $1/\eps$ as well.

\begin{figure}[htp]
\centering
\captionsetup[subfigure]{justification=Centering}
\begin{subfigure}[t]{0.3\textwidth}\label{fig:exp1_gaussian}
    \includegraphics[width=\linewidth]{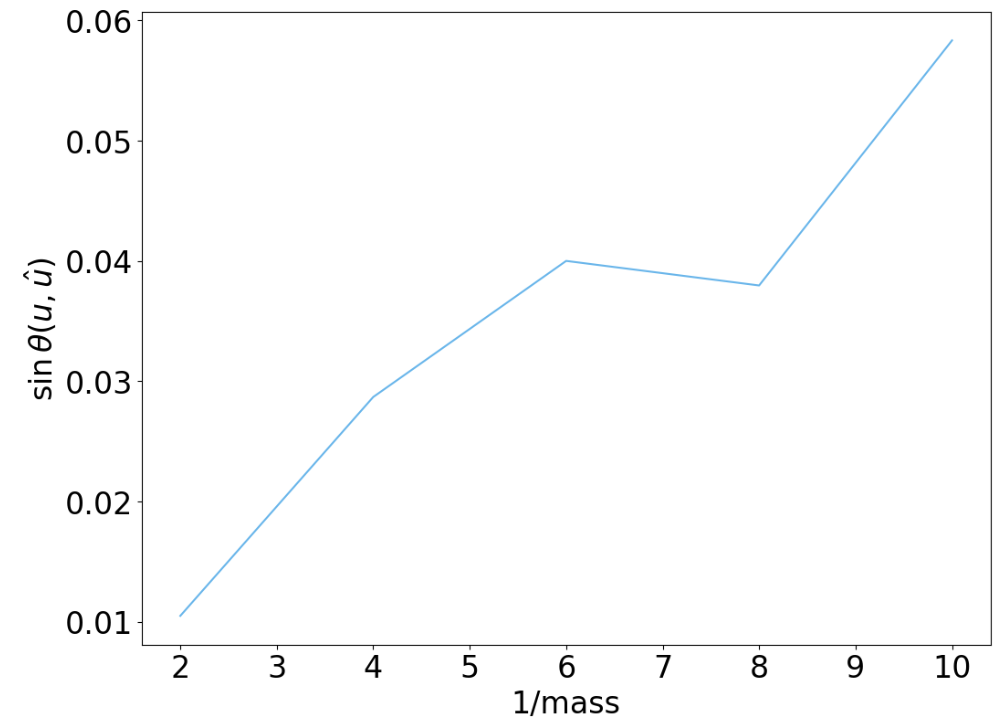}
    \caption{Gaussian}
\end{subfigure}
\begin{subfigure}[t]{0.3\textwidth}\label{fig:exp1_uniform}
    \includegraphics[width=\linewidth]{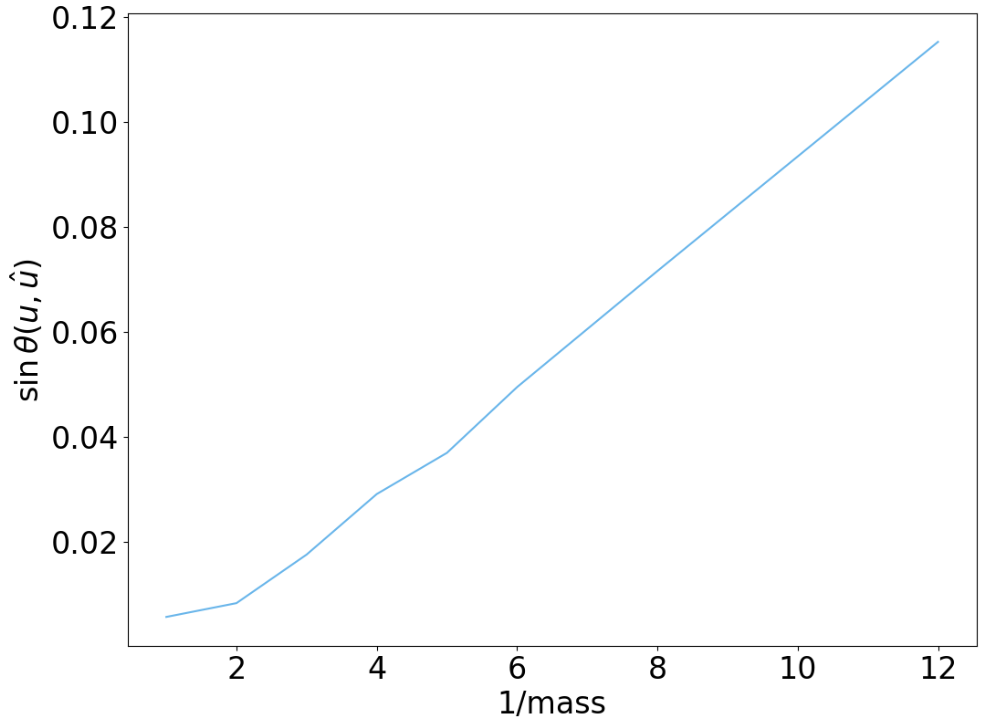}
    \caption{Uniform}
\end{subfigure}
\begin{subfigure}[t]{0.3\textwidth}\label{fig:exp1_expo}
    \includegraphics[width=\linewidth]{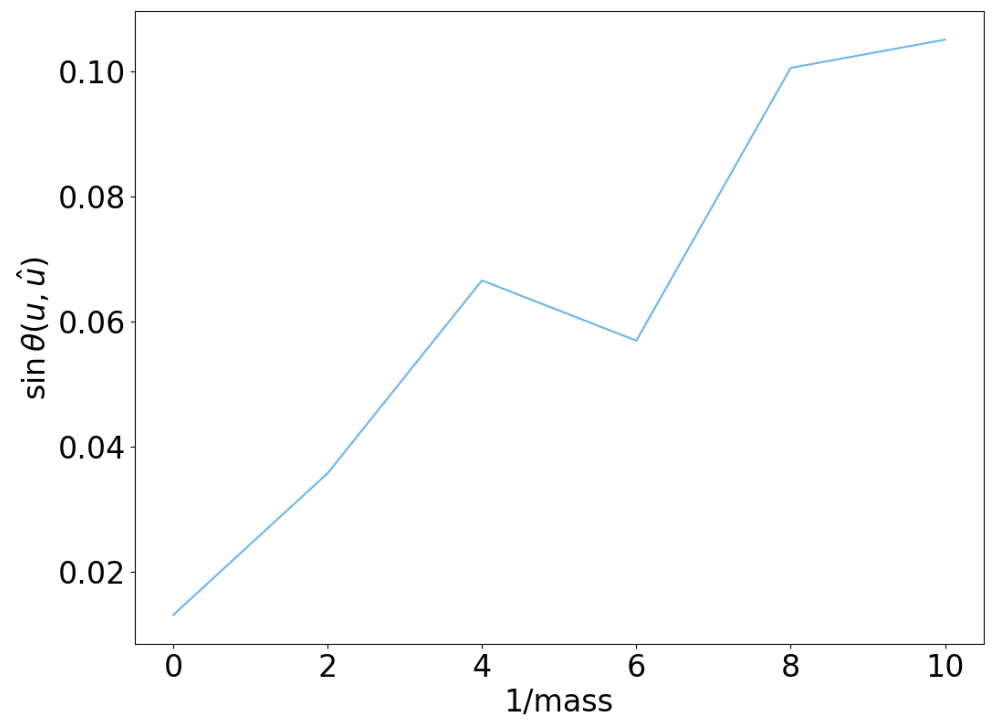}
    \caption{Exponential}
\end{subfigure}
\caption{The performance with respect to $1/\eps$.}
\label{fig:experiment1}
\end{figure}

\vspace{-0.1in}
\section{Discussion and Future Directions}
%We established well-defined distributional setting that enable the unsupervised learning of halfspaces in high dimensions. In this setting, 
We proposed and analyzed an efficient algorithm for unsupervised learning of symmetric product logconcave distributions with margin. Our algorithm only uses re-weighted first and second moments of samples and has the flavor of self-supervised learning. Specifically, contrastive covariance can be viewed as the simplest realization of contrastive learning without any data augmentation~\cite{tian2022deep}. 

We mention several open questions for future exploration:
\vspace{-0.02in}
\begin{itemize}
    \item \textbf{Analysis Refinement}. 
    While we prove a poly($d,1/\eps$) bound using specific values of the re-weighting parameter $\alpha$, as demonstrated by the qualitative lemmas (Lemma~\ref{lemma_contrastive_mean_qual} and Lemma~\ref{lemma_contrastive_covariance_qual}), any distinct pair of nonzero $\alpha$ values should work for the contrastive mean, and any bounded small $\alpha$ should work for the contrastive covariance. 
    %This flexibility ensures the applicability of the algorithm in various real-world scenarios. 

    Our experimental results align with this claim. In fact, they suggest a linear relationship between the sample complexity and the input dimension $d$ and inverse linear with the measure of the margin $\epsilon$, raising the possibility that the sample complexity is linear in $d$ and $1/\eps$. 

    \item \textbf{Distribution Generalization}.
    Can the algorithm's guarantees be extended to more general distributions? The effectiveness of the current algorithm relies on the symmetry of the one-dimensional distribution. Using higher but constant order re-weighting moments could be a way to handle asymmetric distributions.

    \item \textbf{Robust Learning Halfspaces}. 
    An important question to consider is whether the algorithm remains effective when a small fraction of the data falls within the margin (rather than zero). it is crucial that this data be sparser, with density significantly lower than that of the band being removed, to maintain the uniqueness of the halfspace.

    \item \textbf{Intersection of Halfspaces}.
    Another intriguing possibility is the generalization of the problem to include learning the intersection of multiple halfspaces. 
    %This expansion could significantly enhance the algorithm's utility in more complex scenarios.

    \item \textbf{Contrastive learning with Data Augmentation}.
    A nice, broader goal for learning theory might be to develop a model where data augmentation is provably useful to solve a classification task by using a suitable contrast function.
    
\end{itemize}

\paragraph{Acknowledgements.} This work was supported in part by NSF awards CCF-2007443 and CCF-2134105 and an ARC fellowship.

\newpage
\bibliography{bibfile}

% \newpage
% \appendix
% \section*{\LARGE Appendix}

% \input{arXiv_appendix}

\end{document}